%% file: main.tex
\documentclass[3p, final]{elsarticle}
\makeatletter
\def\ps@pprintTitle{%
  \let\@oddhead\@empty
  \let\@evenhead\@empty
  \def\@oddfoot{}%
  \let\@evenfoot\@oddfoot}
\makeatother

\input{templates/my-packages}
\input{templates/my-macros}

\input{templates/math-commands.tex}

\input{templates/ml-opt-packages.tex}

\begin{document}
\begin{frontmatter}
\input{sections/title}
\input{sections/abstract}
\end{frontmatter}

\input{sections/introduction}
\input{sections/methodology}
\input{sections/experiments}
\input{sections/results}
\input{sections/conclusions}
\input{sections/further-research}
\input{sections/acknowledgments}
\input{sections/references}

\begin{appendices}
\appendix
\makeatletter
\renewcommand{\thesection}{\Alph{section}}  
\makeatother
\input{sections/appendix}

\end{appendices}

\end{document}

%% file: templates/my-packages.tex
\usepackage{color, colortbl} 
\usepackage{xcolor}
\usepackage{environ}

\usepackage{enumitem}
\usepackage{bm}

\usepackage{amsmath,amsfonts,bm}

\usepackage{cancel}

\usepackage{tikz}
\usetikzlibrary{automata}
\usetikzlibrary{calc}
\usetikzlibrary{positioning}  
\usetikzlibrary{arrows,arrows.meta} 
\tikzset{
	main node/.style={circle,draw,minimum size=1cm,inner sep=0pt},
}

\usepackage{adjustbox}
\usepackage{graphicx}
\usepackage{wrapfig}

\usepackage{tabularx,booktabs}
\usepackage[flushleft]{threeparttable}
\usepackage{makecell}

\usepackage{calc,pgf}
\def\pgfmath@calc@minof#1#2{min(#1,#2)}

\usepackage{subcaption}
\usepackage{float}

\usepackage{xpatch}

\makeatletter
\xpatchcmd{\paragraph}{3.25ex \@plus1ex \@minus.2ex}{3pt plus 1pt minus 1pt}{\typeout{success!}}{\typeout{failure!}}
\makeatother

\usepackage{mathpazo}

\usepackage{pslatex}
\usepackage{lipsum}
\usepackage[nocomma]{optidef}

\usepackage{amsmath,amsfonts,amssymb,amsthm}

\usepackage{multicol}
\usepackage{multirow}
\usepackage{setspace}
\usepackage{ragged2e}

\usepackage{etoolbox}
\usepackage{filecontents}

\usepackage[title]{appendix}

\usepackage{blkarray, bigstrut}
\usepackage{xparse}

\usepackage{hyperref}
\hypersetup{colorlinks = true,urlcolor=blue, citecolor = black,linkcolor=black}

%% file: templates/my-macros.tex
\newcommand{\detailtexcount}{%
  \immediate\write18{texcount -merge -sum -incbib -dir \jobname.tex > \jobname.wcdetail }%
  \verbatiminput{\jobname.wcdetail}%
}

\newcommand{\quickwordcount}{%
  \immediate\write18{texcount -1 -sum -merge \jobname.tex > \jobname-words.sum }%
  \input{\jobname-words.sum}%
}

\newcommand{\quickcharcount}{%
  \immediate\write18{texcount -1 -sum -merge -char \jobname.tex > \jobname-chars.sum }%
  \input{\jobname-chars.sum} characters (not including spaces)%
}

\newcolumntype{Y}{>{\centering\arraybackslash}X}
\newcolumntype{b}{X}
\newcolumntype{s}{>{\hsize=.5\hsize}X}

\usepackage[most]{tcolorbox}

\newtcolorbox{note}[1][]{%
	enhanced jigsaw, 
	borderline west={2pt}{0pt}{red}, 
	sharp corners, 
	boxrule=0pt, 
	fonttitle={\normalsize\bfseries},
	coltitle={black},  
	title={Note:\ },  
	attach title to upper, 
	#1
}

%% file: templates/math-commands.tex
\usepackage{xspace}
\usepackage{econometrics}

\usepackage{eulervm}

\makeatletter
\newcommand\RedeclareMathOperator{%
	\@ifstar{\def\rmo@s{m}\rmo@redeclare}{\def\rmo@s{o}\rmo@redeclare}%
}
\newcommand\rmo@redeclare[2]{%
	\begingroup \escapechar\m@ne\xdef\@gtempa{{\string#1}}\endgroup
	\expandafter\@ifundefined\@gtempa
	{\@latex@error{\noexpand#1undefined}\@ehc}%
	\relax
	\expandafter\rmo@declmathop\rmo@s{#1}{#2}}
\newcommand\rmo@declmathop[3]{%
	\DeclareRobustCommand{#2}{\qopname\newmcodes@#1{#3}}%
}
\@onlypreamble\RedeclareMathOperator
\makeatother

\def\eqref#1{equation~\ref{#1}}

\def\1{\mathbf{1}}

\NewDocumentCommand{\grad}{e{_^}}{%
	\mathop{}\!
	\nabla
	\IfValueT{#1}{_{\!#1}}
	\IfValueT{#2}{^{#2}}
}

\newcommand{\ODE}{\texttt{ODE}\xspace}

\newcommand{\BPR}{\texttt{BPR}\xspace}

\newcommand{\TVODLULPE}{\texttt{TVODLULPE}\xspace}
\newcommand{\TVGODLULPE}{\texttt{MaTE}\xspace}
\newcommand{\MTP}{\texttt{MaTE}\xspace}
\newcommand{\RK}{\texttt{RK}\xspace}
\newcommand{\SUELOGIT}{\texttt{SUELOGIT}\xspace}
\newcommand{\STALOGIT}{\texttt{STALOGIT}\xspace}

\def\vzero{{\mathbf{0}}}

\def\vtheta{\pmb{\theta}}

\def\vf{{\mathbf{f}}}
\def\vg{{\mathbf{g}}}

\def\vp{{\mathbf{p}}}
\def\vq{{\mathbf{q}}}

\def\vt{{\mathbf{t}}}

\def\vv{{\mathbf{v}}}

\def\vx{{\mathbf{x}}}
\def\vy{{\mathbf{y}}}

\def\vpf{{\mathbf{p}}}

\def\vbeta{{\boldsymbol{\beta}}}
\def\vtheta{{\boldsymbol{\theta}}}
\def\vgamma{{\boldsymbol{\gamma}}}
\def\vomega{{\boldsymbol{\omega}}}
\def\vphi{{\boldsymbol{\phi}}}

\def\mD{{\mathbf{D}}}
\def\mE{{\mathbf{E}}}

\def\mL{{\mathbf{L}}}
\def\mM{\mathbf{M}}

\def\mO{{\mathbf{O}}}

\def\mW{{\mathbf{W}}}

\def\mZ{{\mathbf{Z}}}

\DeclareMathAlphabet{\mathsfit}{\encodingdefault}{\sfdefault}{m}{sl}
\SetMathAlphabet{\mathsfit}{bold}{\encodingdefault}{\sfdefault}{bx}{n}

\def\sA{{\mathcal{A}}}

\def\sH{{\mathcal{H}}}

\def\sK{{\mathcal{K}}}

\def\sR{{\mathbb{R}}}
\def\sS{{\mathcal{S}}}

\def\sV{{\mathcal{V}}}
\def\sW{{\mathcal{W}}}

\newtheoremstyle{boldremark}
{\dimexpr\topsep/2\relax} 
{\dimexpr\topsep/2\relax} 
{}          
{}          
{\bfseries} 
{.}         
{.5em}      
{}          

\newtheorem{definition}{Definition}

\newtheorem{prop}{Proposition}

\theoremstyle{plain}
\newtheorem{assumption}{Assumption}

\theoremstyle{boldremark}
\newtheorem{remark}{Remark}

\DeclareMathAlphabet{\mathcal}{OMS}{cmsy}{m}{n}

%% file: templates/ml-opt-packages.tex
\usepackage{caption}
\usepackage{algorithm,algpseudocode}

\floatname{algorithm}{Algorithm}

\usepackage{vwcol}  

\newcounter{algsubstate}

\newenvironment{algsubstates}
{\setcounter{algsubstate}{0}%
	\renewcommand{\State}{%
		\refstepcounter{algsubstate}%
		\Statex {\footnotesize\alph{algsubstate}:}\space}}
{}

\usepackage{optidef}

%% file: sections/title.tex
\title{Traffic estimation in unobserved network locations using data-driven macroscopic models}

\author[add1,add3]{Pablo Guarda}
\author[add1,add2]{Sean Qian\corref{mycorrespondingauthor}}
\cortext[mycorrespondingauthor]{Corresponding author}
\ead{seanqian@cmu.edu}
\address[add1]{Department of Civil and Environmental Engineering, Carnegie Mellon University, Pittsburgh, PA, U.S.A.}
\address[add2]{Heinz College, Carnegie Mellon University, Pittsburgh, PA, U.S.A.}
\address[add3]{Fujitsu Research of America, Pittsburgh, PA, U.S.A.}

%% file: sections/abstract.tex
\begin{abstract}

This paper leverages data-driven macroscopic transportation models and multi-source spatiotemporal data collected from automatic traffic counters and probe vehicles to accurately estimate traffic flow and travel time in links where these measurements are unavailable. This problem is critical in transportation planning applications where the sensor coverage is low and the planned interventions have network-wide impacts. The proposed model, named the Macroscopic Traffic Estimator (\MTP), can perform network-wide estimations of traffic flow and travel time only using the set of observed measurements of these quantities. Because \MTP is grounded in macroscopic flow theory, all parameters and variables of the model are interpretable. The estimated traffic flow satisfies fundamental flow conservation constraints, and the estimated travel time exhibits an increasing monotonic relationship with the traffic flows. Using logit-based stochastic traffic assignment as the principle for routing flow behavior makes the model fully differentiable with respect to the model parameters. This property facilitates the application of automatic differentiation tools and computational graphs to learn parameters from vast amounts of spatiotemporal data. We also successfully integrate neural networks and polynomial kernel functions to capture link flow interactions and enrich the mapping of traffic flows into travel times. \MTP also adds a destination choice model and a trip generation model capable of utilizing historical data on the number of trips generated by location, which is information that is generally more accessible than historical origin-destination (O-D) matrices. Experiments on synthetic data show that the model can accurately estimate travel time and traffic flow in out-of-sample links. Results obtained using real-world multi-source data from a large-scale transportation network in Fresno, CA suggest that \MTP outperforms data-driven benchmarks, especially in travel time estimation. The estimated parameters of \MTP are also informative about the hourly change in travel demand and supply characteristics of the transportation network.

\end{abstract}

\begin{keyword} 
computational graphs, traffic flow estimation, logit-based stochastic traffic assignment, origin-destination demand estimation, large-scale network modeling
\end{keyword}

%% file: sections/introduction.tex
\section{Introduction}

Estimating traffic flow and travel time is crucial for intelligent transportation systems and smart cities. It can help optimize the efficiency and reliability of transportation systems, such as public transit, freight delivery, and traffic management. It can enhance the safety and comfort of travelers by reducing congestion, accidents, and emissions. It can support better planning and decision-making for transportation policies, investments, and operations. Recent papers on this topic have applied various machine learning and deep learning techniques to improve the accuracy and efficiency of estimating traffic parameters such as speed \citep{zheng_dynamic_2022, guo_hierarchical_2021, chen_constructing_2022}, traffic volume and congestion \citep{liang_fine-grained_2022, chen_novel_2022, polson_deep_2017, wu_hybrid_2018}. 

In the realm of transportation network modeling, accurate estimation of traffic flow and travel time is crucial, especially in parts of the network where sensors or direct measurement tools are not installed or are temporarily non-operational. Our paper refers to this problem as out-of-sample estimation, namely, the estimation of traffic flow and travel time in links of the network where measurements of these quantities are unavailable. The term "estimation" pertains to the process of determining traffic flow and travel time in specific links of the network based on available data from other links. Unlike prediction, which often implies forecasting future values based on patterns or trends observed in the data, estimation in our context focuses on deducing current or historical measurements in the transportation network that are not directly observed.

Methods to estimate traffic flow and travel time can be categorized as model-based and data-driven \citep{bai_travel-time_2018}. Pure data-driven approaches are suitable for making estimations under recurrent traffic conditions and do not require incorporating domain knowledge into the modeling framework. Spatial kriging methods are a family of data-driven models that are being increasingly used in the transportation domain to perform network-wide estimation of traffic flow \citep{mathew_comparative_2021, song_traffic_2019, selby_spatial_2013} and speeds \citep{nie_correlating_2023}. They exploit spatial correlations in the data to provide estimations in unobserved locations. A fundamental limitation of kriging methods and other data-driven models is that they usually require large amounts of data to generalize well in traffic contexts or areas of the transportation networks that are not observed in the training set. Because of the lack of an analytical model to generate traffic flow and travel time estimations, it may be challenging or unfeasible to analyze the impact of some interventions in the transportation network, such as road closures or road capacity improvement. Furthermore, estimations of data-driven methods do not satisfy basic constraints arising from applying network flow theory to transportation systems, which may be critical for transportation planning applications. Lastly, most research in data-driven models studies traffic flow and travel time as two phenomena in isolation despite the known relationship between these two physical quantities. 

Traditional transportation planning models are model-based and are grounded on network flow theory. These models have a long-standing history in transportation and are used to estimate the impact of various interventions in the transportation network. One of the most well-known models is the \textit{four-step model}, which most transit agencies still use worldwide. These models help understand the demand and supply characteristics of the transportation network, and most of their parameters are interpretable. A rich body of literature has developed methods to estimate the parameters of some components of the 4-step model using emerging data sources, such as traffic counts, travel times, and GPS trajectory data \citep{Wu2018a, cascetta_calibrating_1997,Russo2011a, guarda_statistical_2024}. Some of the parameters that can be learned include the O-D matrix \citep{Yang2001, Ma2018}, the travelers' utility function coefficients \citep{guarda_statistical_2024, Yang2001, Caggiani2011, Wang2016,Garcia-Rodenas2009}, and the link performance functions \citep{wollenstein-betech_joint_2022, garcia-rodenas_adjustment_2013, suh_highway_1990, guarda_estimating_2024}. 

Recent literature has also leveraged advances in machine learning to learn the parameters of macroscopic network models in contexts of static traffic assignment. These methodological enhancements include the use of computational graphs \citep{guarda_statistical_2024, Wu2018a}, deep implicit layers \citep{liu_end--end_2023}, inverse optimization \citep{wollenstein-betech_joint_2022} and sensitivity analysis \citep{Ma2018}. One challenge that remains open is how to leverage the prediction capabilities of data-driven methods with the interpretability and theoretical consistency provided by model-based methods to perform network-wide estimation of traffic flow and travel time. By theoretical consistency, we refer to compliance with some fundamental properties and constraints associated with network flow. 

\section{Contributions of this research}

In this paper, we introduce the Macroscopic Traffic Estimator (\MTP), a model that leverages macroscopic traffic flow theory, computational graphs, and multi-source spatiotemporal data collected from automatic traffic counters and probe vehicles to accurately estimate traffic flow and travel time on network links lacking historical measurements. To assess the model's ability to solve this estimation problem, we computed the out-of-sample error, namely, the estimation error associated with links without observations of traffic flow and travel time. Addressing this problem is vital in transportation planning, especially in scenarios where sensor coverage is limited, and the interventions being evaluated have widespread network impacts. The \MTP can also estimate future traffic flows and travel times, enabling network-wide estimation based solely on the available sensor measurements within the transportation network.

With the motivation of developing a data-driven model with interpretable parameters, grounded on network flow theory and that has good performance in estimating traffic flow and travel time at a network-wide scale, \MTP enhances the computational graph framework introduced by \citet{guarda_estimating_2024} in multiple ways. First, \MTP includes layers that model trip generation and the travelers' destination choices to identify location-specific features from Census data that correlate with trip generation and are helpful to estimate changes in travel demand and demand-supply due to population and socio-demographic changes. Second, \MTP leverages the power of neural networks and polynomial kernels to enrich the mapping of traffic flow to travel times without increasing the risk of overfitting. Third, \MTP relaxes the assumption of network equilibrium made by \citet{guarda_estimating_2024} to extend the model's applicability to real-world settings where the equilibrium conditions may not perfectly hold in practice. Under this framework, our model leverages the traffic equilibrium principle as a regularizer of the link flow solution to prevent overfitting and maximize out-of-sample estimation accuracy. To measure out-of-sample performance, the model is trained with data from a subset of links that report traffic flow and travel time observations. Then, the model's estimations are evaluated in the set of links that have no measurements of travel time and traffic flow in the training set.  Figure  \ref{fig:computational-graph-general} shows an overview of the main components of \MTP and how different data sources are incorporated to estimate traffic flow and travel times. It also highlights in blue the extensions made to the computational graph developed by \citet{guarda_estimating_2024} to improve out-of-sample performance.

\begin{figure}[H]
	\centering
	\includegraphics[width=0.95\textwidth, trim= {3.5cm 10cm 12cm 3.5cm},clip]{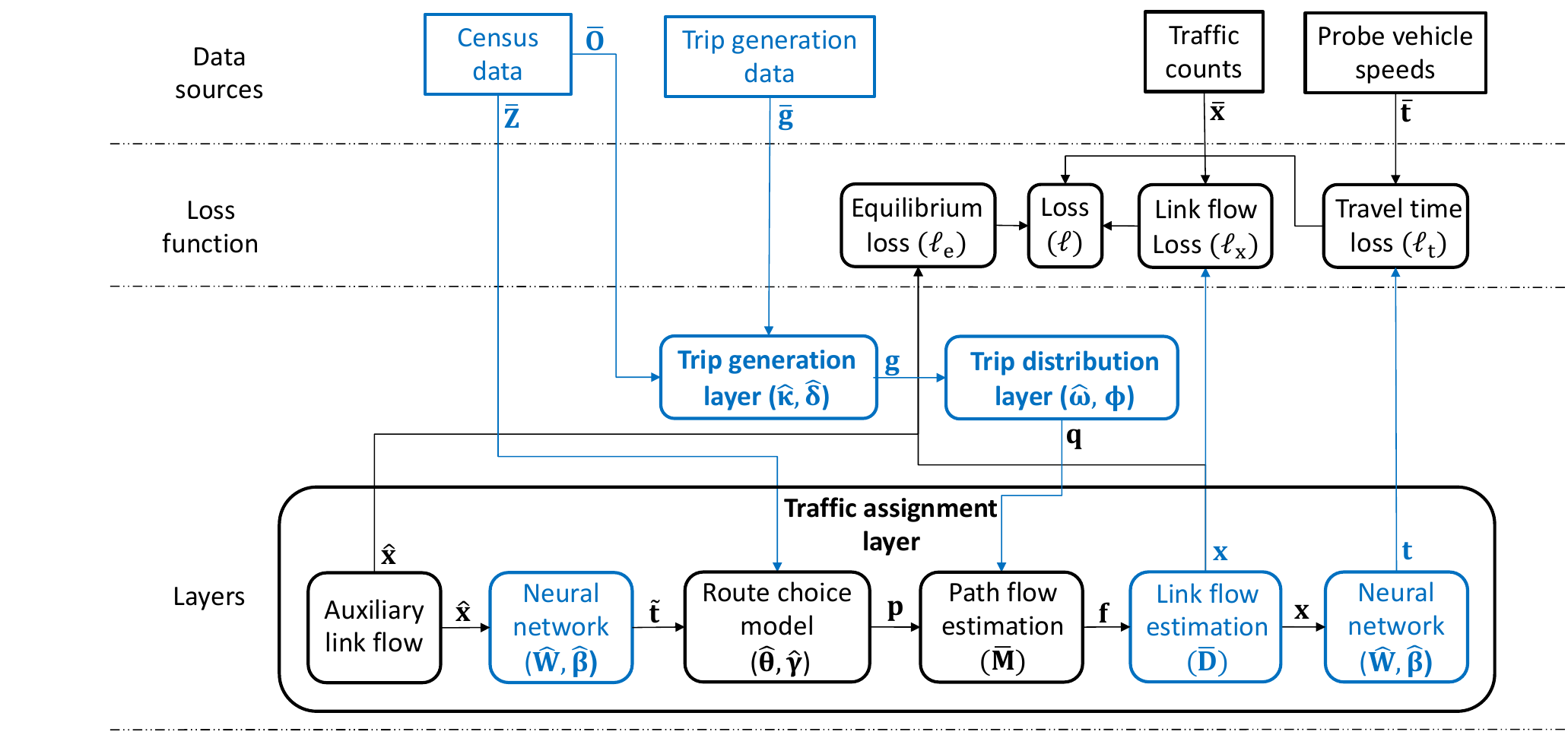}
	\caption{An illustration of \MTP. The blue elements represent the components that extend the model developed by \citet{guarda_estimating_2024}.
    }
    \label{fig:computational-graph-general}
\end{figure}

The methodological contributions to the existing literature can be summarized as follows. First, we conduct a rigorous cross-validation strategy to measure the in-sample and out-of-sample estimation errors associated with travel times and traffic flows (Figure \ref{fig:mate-validation-strategy}). To our knowledge, this study is among the first to examine out-of-sample estimation errors in large-scale network flow modeling rigorously. The performance of \MTP is also compared against standard data-driven models used in prior literature. Second, we incorporate neural networks in a computational graph framework to capture the effect of flow interactions and to enrich the mapping of traffic flow into travel time in macroscopic flow models. The parameters of the neural network can be jointly estimated with the rest of the parameters of the computational graph. This extends work from \citet{guarda_statistical_2024} and \citet{Wu2018a} where the link performance functions are assumed to have known functional form and their parameters are assumed given. It also extends work from  \citet{wollenstein-betech_joint_2022}, \citet{Wu2018a} and \citet{guarda_estimating_2024} where the travel time on a link does not depend on the flows in other links, but only on the flows in that link. Third, we model trip generation with location-specific attributes by adding a regression layer in the computational graph. This extends work from \citet{liu_end--end_2023} and \citet{guarda_statistical_2024} where no O-D matrix is estimated and from \citet{Ma2020}, \citet{wollenstein-betech_joint_2022} and \citet{guarda_estimating_2024} where the trip generation stage is not modeled. It also extends work from \citet{Wu2018a} where the trips generated by location are estimated but not modeled as a function of location-specific attributes, and the travel times are assumed exogenous.

\begin{figure}[H]
	\centering
	\includegraphics[width=0.7\textwidth, trim= {9cm 12.5cm 9cm 6.4cm},clip]{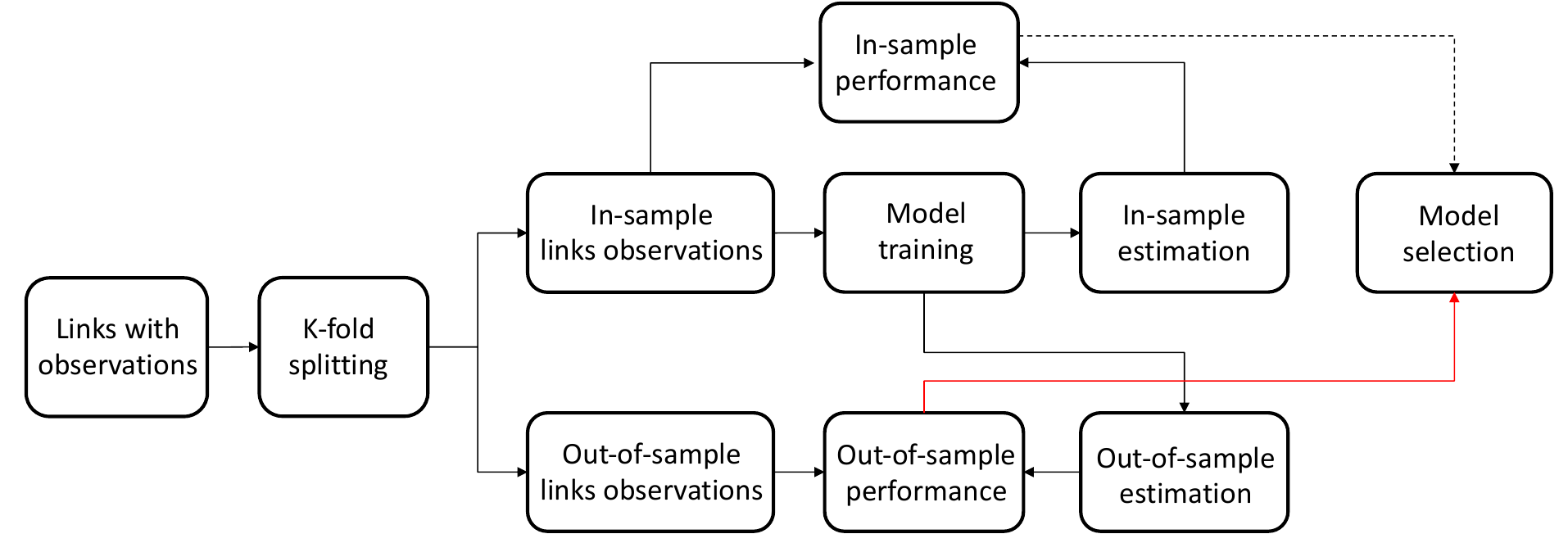}
	\caption{An illustration of the cross-validation strategy proposed to compute out-of-sample performance and perform model selection. The dashed line represents the model selection approach used in the current literature. The red line represents our approach.}
    \label{fig:mate-validation-strategy}
\end{figure}

The paper is organized as follows. Section \ref{sec:mathematical-formulation} introduces the mathematical formulation of our methodology. Section \ref{sec:solution-algorithm} describes the solution algorithm. Section \ref{sec:numerical-experiments} conducts experiments on a mid-size network to verify the computational efficiency of the proposed framework. In Section \ref{sec:large-scale-implementation}, our framework is applied to a large-scale network with real-world spatio-temporal data. Section \ref{sec:conclusions} presents our main conclusions and discusses avenues for further research. Note that all our analyses are replicable, open-sourced, and ready to be used for the transportation community (Section \ref{sec:model-implementation-data}). 

%% file: sections/methodology.tex
\section{Formulation}
\label{sec:mathematical-formulation}

This section presents the assumptions and the formulation of the Macroscopic Traffic Estimator (\MTP) model. Some of the notation used throughout this section and the paper is presented in Tables \ref{table:notation1} and \ref{table:notation2}, Appendix \ref{appendix:sec:notation}.

\subsection{Preliminaries}

Consider a transportation network characterized by a graph $\mathcal{G} = (\sV,\sA)$, where $\sA$ is the set of streets or links connecting the set $\sV$ of nodes or locations in the network. Individuals want to travel between a set $\sW$ of origin-destination pairs connected by a set $\sH$ of paths. Suppose there is access to a set $\sS$ of samples of travel time and traffic flows in a subset of links in the network, which may be collected during multiple periods of the day. In each period, the traffic flow and travel time are assumed to be constant. Each period may comprise samples collected on multiple days, e.g., two samples collected between 9:00 AM and 10:00 AM on two subsequent Wednesdays. Formally, each sample $i \in \sS$ comprises observations of traffic flow $\overline{\vx} \in \sA$ and travel time $\overline{\vt} \in \sR^{|\sA|}$ in a subset of links in the network and a specific period. Furthermore, a vector $\overline{\vg} \in \sR^{|\sV|}$ with a reference number of the trips generated in each node could be available for some periods. 

\subsection{Parameters}
\label{ssec:parameters}

The learnable parameters of the model are: 

\begin{itemize}
\item auxiliary link flow parameters: $\hat{\vx} \in \sR^{|\sS| \times |\sA|}$
\item feature-specific parameters in route choice utility function: $\hat{\vtheta} \in \sR^{|\sS| \times (1+|\sK_Z|)}$
\item link-specific parameters in route choice utility function: $\hat{\vgamma}  \in \sR^{|\sA|}$
\item performance function parameters: $\widehat{\mW} \in \sR^{|\sA| \times |\sA|}, \widehat{\vbeta} \in \sR^{b}$
\item feature-specific parameters of the trip generation function: $\hat{\vkappa} \in \sR^{|\sS| \times |\sK_O|}$
\item location-specific parameters of the trip generation function: $\hat{\vdelta}  \in \sR^{|\sS| \times |\sV|}$
\item origin-destination specific parameters in destination choice utility function: $\hat{\vomega}  \in \sR^{|\sS| \times |\sW|}$
\end{itemize}

where $\sK_Z$ is the set of exogenous features in the route choice utility function, $\sK_O$ is the set of features in the trip generation model and $b$ is the degree of the polynomial used to compute the performance functions. 

\subsection{Variables}
\label{ssec:variables}

Output variables that are functions of the learnable parameters are:

\begin{itemize}
\item link flows: $\vx\in \sR^{|\sS| \times |\sA|}$: 
\item auxiliary travel times: $\tilde{\vt}\in \sR^{|\sS| \times |\sA|}$
\item link travel times: $\vt\in \sR^{|\sS| \times |\sA|}$: 
\item path flows: $\vf\in \sR^{|\sS| \times |\sH|}$: 
\item path utilities: $\vv \in \sR^{|\sS| \times |\sH|}$
\item path choice probabilities: $\vp \in \sR^{|\sS| \times |\sH|}$
\item flows in origin-destination pairs: $\vq \in \sR^{|\sS| \times |\sW|}$
\item destination choice probabilities: $\vphi \in \sR^{|\sS| \times |\sW|}$
\item flows generated by node: $\vg \in \sR^{|\sS| \times |\sV|}$
\end{itemize}

\subsection{Inputs}
\label{ssec:inputs}

The inputs that are assumed to be given during model training are:

\begin{itemize}
\item Incidence matrix for mapping path flows to link flows: $\overline{\mD}\in \sR^{|\sA| \times |\sH|}$ 
\item Incidence matrix for mapping path flows to O-D flows: $\overline{\mM}\in \sR^{|\sW| \times |\sH|}$
\item Incidence matrix for mapping O-D flow into generated flow: $\overline{\mL}\in \sR^{|\sV| \times |\sW|}$
\item Tensor of exogenous features in route choice model: $\overline{\mZ} \in \sR^{|\sS| \times |\sA| \times |\sK_Z|}$
\item Tensor of exogenous features in trip generation model: $\overline{\mO} \in \sR^{|\sS| \times |\sV| \times |\sK_O|}$
\item Vectors of link capacities and free flow travel times: $\overline{\vx}^{\textmd{max}} \in \sR^{|\sA|}$, $\overline{\vt}^{\textmd{min}} \in \sR^{|\sA|}$
\item Matrices of observed link flow and travel time: $\overline{\vx} \in \sR^{|\sS| \times |\sA|}, \overline{\vt} \in \sR^{|\sS| \times |\sA|}$

\end{itemize}

\subsection{Assumptions}
\label{ssec:assumptions}

The main assumptions to formulate our model are: 

\begin{assumption}[Time-varying travel behavior]
\label{assumption:time-varying-travel-behavior}
The travelers' preferences and the number of trips generated in each location vary by period. 
\end{assumption}

As a consequence of Assumption \ref{assumption:time-varying-travel-behavior}, the learnable parameters $\hat{\vg}$, $\hat{\vx}$, and $\hat{\vtheta}$ associated with the trip generation, the link flow parameters and the route choice utility, respectively, can change between observations collected in different periods. 

\begin{assumption}[\SUELOGIT]
\label{assumption:suelogit}
Network traffic flow approximately follows stochastic user equilibrium with logit assignment in each period
\end{assumption}

If the travel demand and the travelers' utility function are allowed to change by period (Assumption \ref{assumption:time-varying-travel-behavior}), the equilibrium state should also vary by period. Consequently, the link flow ($\vx$), path flow $\vf$, and travel time variables $\vt$ can also change by period.

\begin{remark}
Our formulation allows for a relaxation of equilibrium conditions in the transportation network. Thus, the \SUELOGIT assumption does not need to hold perfectly in practice but can guide the network flow toward a stochastic equilibrium. 
\end{remark}

\begin{assumption}[Performance functions]
\label{assumption:performance-functions}
The performance function in a link monotonically increases with respect to the traffic flow in that link, and it has parameters independent of the period. 

\end{assumption}

The standard BPR function is an example of a performance function that satisfies Assumption \ref{assumption:performance-functions}. This paper incorporates a more flexible performance function that can capture the impact of the traffic flow in adjacent links on the travel time in a link.  

\begin{assumption}[Observable and time-invariant path sets]
\label{assumption:invariant-path-sets}
The composition of the path sets is given and independent of the period
\end{assumption}
As a consequence of Assumption \ref{assumption:invariant-path-sets}, the incidence matrices $\overline{\mD} \in \sR^{|\sA| \times |\sH|}, \ \overline{\mM}\in \sR^{|\sW| \times |\sH|}$ are given and equal for all samples. This assumption can be relaxed by incorporating column generation methods for path set generation \citep{guarda_statistical_2024}, but this is out of the scope of this paper. 

\begin{assumption}[Trip generation and trip distribution]
\label{assumption:destination-choices}
Travelers make destination choices from every location in the network following a logit model
\end{assumption}

This assumption is key to incorporating the trip generation step in the model and then distributing the trips generated by location among origin-destination pairs. 

\begin{assumption}[Observable and time-invariant destination sets]
\label{assumption:invariant-destination-sets}
The set of destinations that are reachable from every origin location is known and is independent of the period
\end{assumption}
As a consequence of Assumption \ref{assumption:invariant-destination-sets}, the incidence matrix $\overline{\mL} \in \sR^{|\sV| \times |\sW|}$ that maps origin-destination flows into generated trips by location is given and equal for all samples. This assumption can be relaxed by incorporating column-generation methods that can dynamically update the set of reachable destinations, but this is out of the scope of this paper. 

\subsection{Constraints}
\label{ssec:constraints}

The parameters and variables of the models must satisfy a set of constraints arising from the assumptions (Section \ref{ssec:assumptions}) and network flow theory. These constraints enforce that the model estimations comply with some fundamental properties of network flow and help prevent overfitting by reducing the feasible space of solutions. The constraints presented in this section must be satisfied for every sample $i \in \sS$ of observations collected from the network at a given period. The parameters, variables, and inputs are indexed by $i$ when they cannot be ensured to be equal between samples. 

\subsubsection{Conservation of network flows}
\label{sssec:flow-conservation-constraints}
As in any path-based static traffic assignment model, two sets of flow conservation constraints must be satisfied:
\begin{align}
\label{eq:path-flows-link-flows-conservation-constraint}
\vx_i &= \overline{\mD} \vf_i \\
\label{eq:path-flows-od-flows-conservation-constraint}
\vq_i &= \overline{\mM}\vf_i
\end{align}

where $\vx_i, \vf_i, \vq_i$ represents the flow in links, paths, and O-D pairs associated with sample $i$. With the incorporation of the trip generation step, the model must also satisfy the following constraint: 
\begin{align}
\label{eq:od-flows-generated-trips-conservation-constraint}
\vg_i &= \overline{\mL} \vq_i
\end{align}

\subsubsection{Route and destination choice models}

Path choice probabilities can be written in a vectorized form as follows \citep{guarda_estimating_2024}:
\begin{align}
    \label{eq:path-choice-probabilities-constraint}
    \vp_i 
    &= \frac{\exp\left(\displaystyle \vv_i \right)}{\displaystyle \overline{\mD}^{\top}\overline{\mD} \exp(\vv_i  )} 
\end{align}
where 
\begin{align}
    \label{eq:path-utility-constraint}
    \vv_i 
    &= \overline{\mM}^{\top} \left(\begin{bmatrix} \tilde{\vt}_i & \overline{\mZ}_i \end{bmatrix}
\hat{\vtheta}_i + \hat{\vgamma}\right)
\end{align}
is the vector of path choice utilities associated with a sample $i$. The left multiplication by the incidence matrix $\mD$ allows mapping the link utilities into path utilities. Similarly, the destination choice probabilities can be written as: 
\begin{align}
    \label{eq:destination-choice-probabilities-constraint}
    \vphi_i 
    = \frac{\exp\left(\displaystyle \hat{\vomega}_i \right)}{\displaystyle \overline{\mL}^{\top}\overline{\mL} \exp(\hat{\vomega}_i  )}
\end{align}
where $\hat{\vomega}_i$ is a vector of origin-destination specific parameters, which could vary between samples according to the period. In contrast to the route choice utility (Eq. \ref{eq:path-utility-constraint}), the specification of the destination choice utility does not include origin or destination-specific features. Incorporating these features is possible, but this is out of the scope of the paper. 

\subsubsection{Aggregation of individual choices to network flow}
\label{sssec:choices-to-flows-constraint}

Route choices can be aggregated to path flows as follows \citep{guarda_estimating_2024}: 
\begin{equation}
\vf_i=\left(\overline{\mD}^{\top}\hat{\vq}_i\right) \circ \vp_i
\label{eq:route-choice-constraint}
\end{equation}

\noindent Similarly, destination choices can be aggregated to O-D flows as:
\begin{equation}
\vq_i=\left(\overline{\mL}^{\top}\hat{\vg}_i\right) \circ \vphi_i
\label{eq:destination-choice-constraint}
\end{equation}
where it becomes clear that the O-D flows are a function of some learnable parameters. Note that the conservation constraints (\ref{eq:path-flows-od-flows-conservation-constraint}) and (\ref{eq:od-flows-generated-trips-conservation-constraint}) become redundant in the presence of constraints (\ref{eq:route-choice-constraint}) and (\ref{eq:destination-choice-constraint}). 

\subsubsection{Specification of trip generation function}
\label{sssec:specification-trip-generation-function-constraint}

The number of trips generated in a location is modeled with a linear function dependent on both location-specific and feature-specific parameters:
\begin{align}
    \label{eq:generation-model-constraint}
    \vg_i
    &= \overline{\mO}_i \hat{\vkappa}_i + \hat{\vdelta}_i
\end{align}
where each column of the matrix $\overline{\mO}_i$ contains the values of an exogenous feature at each location, e.g., the median household income in the Census block associated with the location. To ensure that the number of generated trips is non-negative, the vector $\vg$ is projected to the non-negative orthant. 

\begin{remark}
To increase representational capacity, location-specific parameters $\hat{\vdelta}_i$ could be period-specific. However, this could compromise the identification of the features-specific parameters $\hat{\vkappa}_i$. A solution to prevent identifiability issues is to obtain pre-trained weights for $\hat{\vkappa}_i$ and fix their values during model training.
\end{remark}

\subsubsection{Mapping of traffic flows into travel times}
\label{sssec:performance-function-constraint}

The \MTP model introduces a more flexible representation of the performance function that captures the impact of traffic flow interactions on travel times (Eq. \ref{eq:neural-performance-function}). At a high level, the performance function is a Multilayer Perceptron (\texttt{MLP}) that receives as input features generated from a polynomial kernel applied to the traffic flows at every link in the network, it uses a dense layer to generate the traffic flow interactions, and it outputs the links' travel times. Due to these features, we refer to this function as \textit{neural performance function}. Similar to classical performance functions used in static traffic assignment, it can capture the non-linear relationship between traffic flow and travel time through a polynomial kernel function. The performance function also satisfies the BPR property that the travel time in the links is equal to the free flow travel time $\overline{\vt}^{\min}$ when the link flow is zero. Formally, the vector of the link travel time parameters $\hat{\vt}$ and variables $\vt$ are defined as:
\begin{align}
\label{eq:neural-performance-function}
\tilde{\vt}_i&=\overline{\vt}^{\min} \circ \left(1 + \left(\overline{\mE}\circ\widehat{\mW} \right) \Phi(\hat{\vx}_i) \right)\\
\vt_i&=\overline{\vt}^{\min} \circ \left(1 + \left(\overline{\mE}\circ\widehat{\mW} \right) \Phi(\vx_i) \right)
\end{align}
with 
\begin{align}
\label{eq:polynomial-kernel-function}
\Phi(\hat{\vx}_i) &= \sum^{b}_{k=1} \beta_k\left(\frac{\hat{\vx}_i}{\overline{\vx}^{\max}}\right)^k\\
\Phi(\vx_i) &= \sum^{b}_{k=1} \beta_k\left(\frac{\vx_i}{\overline{\vx}^{\max}}\right)^k
\end{align}

where $\Phi(\cdot): \sR^{|\sA|} \to \sR^{|\sA|}$ is the polynomial kernel function that captures the non-linear relationship between the traffic flows and travel times, $b$ is a hyperparameter that determines the degree of the polynomial function and $\hat{\beta}_k$ is a learnable parameter associated with the term of degree $k\leq b$ in the polynomial function. Similar to the standard BPR function, link flows in the polynomial kernel function are normalized by the vector of link capacities $\overline{\vx}^{\max}$. The elements of the vector of parameters $\hat{\beta}$ are projected to the non-negative orthant to enforce that the travel time in a link monotonically increases with respect to the traffic flow in that link. An advantage of using a polynomial kernel instead of a BPR function that utilizes an exponential function with a learnable parameter is that the learnable parameter becomes linear with respect to the travel times, which helps to stabilize the gradient computation. Another advantage of using a linear combination of polynomial functions is increased representational capacity to map traffic flow into travel times. 

The treatment of traffic flow interactions is through the matrix $\widehat{\mW} \in \sR^{|\sA| \times |\sA|}$. Each element $W_{k,l}$ in $\widehat{\mW}$ captures the impact of the traffic flow of a link associated with row $k$ into the travel times of the link associated with column $l$. Since increases in the traffic flow on a link should increase the travel time, it is advisable to constrain the elements in the diagonal of $\widehat{\mW}$ to be positive. To constrain the representational capacity of the performance function, $\widehat{\mW}$ is masked by the matrix $\overline{\mE} \in \sR^{|\sA| \times |\sA|}$ through the element-wise multiplication shown in Eq. \ref{eq:neural-performance-function}. Each element $E_{k,l}$ in $\overline{\mE}$ takes the value one if $k = l$ or if the link associated with row $k$ is adjacent to the link associated with column $l$. Note that $\overline{\mE}$ is created programmatically from the network adjacency matrix, and if $\overline{\mE}$ is enforced to be diagonal, the effect of traffic flow interactions is ignored.

\subsection{Network equilibrium} 

The stochastic user equilibrium with logit assignment (\SUELOGIT) problem can be cast as a fixed point formulation in link flow space \citep{gentile_new_2018}. As shown by \citet{guarda_estimating_2024}, the fixed-point problem equates to finding a solution for $\hat{\vx}_i$ in: 
\begin{equation}   
\label{eq:fixed-point-sue-constraint}
\hat{\vx}_i
= {\overline{\mD}} \ \vf (\vp(\vv(\vt(\hat{\vx}_i))))
\end{equation}

where the link flow variables $\vx_i$ and the link flow parameters $\hat{\vx}_i$ are enforced to be the same, i.e., $\vx_i = \hat{\vx}_i$, and all learnable parameters of the model, except for $\hat{\vx}_i$, are fixed. A natural criterion to assess if the solution is at equilibrium is to compare the relative difference between the link flow parameters $\hat{\vx}_i$ and the link flow variables $\vx_i$. Formally, the relative gap in epoch $j$ associated with sample $i$ is defined as: 
\begin{equation}
\label{eq:relative-gap}
\displaystyle
\rho_{i,j} = \frac{\left\|\vx_i^{(j)}-\hat{\vx}_i^{(j)}\right\|_1}{\left\|\vx_i^{(j)}\right\|_1}
\end{equation}

In contrast to \citet{guarda_estimating_2024}, the \MTP model is not enforced to satisfy the equilibrium conditions strictly. Instead, we let the learning algorithm find the relative gap that maximizes the estimation accuracy associated with traffic flow and travel time. In the proposed framework, the relative gap can be interpreted as a proxy of the degree of equilibrium in the network.

\subsection{Loss function}
\label{ssec:loss-function}

The loss function $\ell$ is a weighted sum of at least three components; the link flow loss $\ell_{x}= d_{x}\left(\vx_i,\overline{\vx}_i\right)$, the travel time loss $\ell_{t}=d_{t}\left(\vt_i,\overline{\vt}_i\right)$ and the equilibrium loss $\ell_e=d_{e}\left(\vx_i, \hat{\vx}_i\right)$, where $d_x, d_{t}, d_{e} $ represent some distance metric between the inputs of the functions. Incorporating the equilibrium component in the loss function aims to satisfy the fixed point constraint in (\ref{eq:fixed-point-sue-constraint}). We choose the Euclidean norm as the distance metric and divide the loss components $\ell_{x}$ and $\ell_{t}$ by the number of non-missing observations associated with link flows ($\overline{n}_{x_i}$) and travel times ($\overline{n}_{t_i}$) in sample $i$. The equilibrium component is divided by the number of elements $|\sA|$ in the vector of link flow parameters $\hat{\vx}_i$.

To account for the difference in scale and dispersion in multi-source data, we also divide the loss components $\ell_{x}$ and $\ell_{t}$ by the standard deviation of the observed travel times ($\overline{\sigma}_{t_i}$) and traffic flows ($\overline{\sigma}_{x_i}$), respectively. The equilibrium component $\ell_e$ is divided by the standard deviation of the observed link flow in the sample, which encourages that the distribution of link flow parameters $\hat{\vx}_i$ and observed link flow are similar. Dividing $\ell_e$ by the standard deviation of the link flow variables $\vx_i$ is not convenient for the gradient computation because the denominator in $\ell_e$ becomes a function of the learnable parameters $\hat{\vx}_i$.

The set of hyperparameters $\lambda_{x},\lambda_{t},\lambda_{e}$ weight the loss components $\ell_{x},\ell_{t}, \ell_{e}$ to obtain the total loss. A higher value of $\lambda_{e}$ prioritizes obtaining a solution that satisfies the equilibrium condition, but it may compromise estimation performance. Higher values of $\lambda_{t}$ or $\lambda_{x}$ may induce lower reconstruction errors for travel time or link flow, respectively. The impact of the choice of hyperparameters is explored in the following sections. Note that the loss function could also include other components, such as a deviation between estimated and reference O-D matrices or between estimated and reference vectors of the generated trips by location. 

\subsection{Optimization problem}
\label{ssec:optimization-problem}

The optimization problem $\mathcal{P}$ seeks to minimize the loss function under the constraints in Section \ref{ssec:constraints}. The loss function adds the contribution of each sample $i \in S$ to the total loss. Some non-negativity constraints, such as $\vt_i \geq 0, \forall i \in S$, are not included because they are directly satisfied with the proposed constraints in $\mathcal{P}$. For convenience, we define the auxiliary variables $\vy$ and $\tilde{\vy}$ to denote the output of the polynomial kernel function $\Phi: \sR^{|\sA|} \to \sR^{|\sA|}$. We want to remark on the importance of normalizing the loss components to account for the difference in scale of multi-source data and the difference in coverage between samples and data sources. An appropriate normalization helps to find solutions that balance the loss reduction among components and reduce the need to fine-tune the hyperparameters of the loss function. 

Note that as a consequence of Assumptions \ref{assumption:invariant-path-sets} and \ref{assumption:invariant-destination-sets}, Section \ref{ssec:assumptions}, the incidence matrices are not indexed by the sample $i$. In line with Assumption \ref{assumption:performance-functions}, Section \ref{ssec:assumptions}, the performance function parameters are not indexed by $i$. For other parameters, we assume that observations collected in the same period have the same parameter values. Parameter sharing helps to reduce the model dimensionality and prevents overfitting the training data. Given that the standard deviation and the coverage of observations may vary between samples, the number of non-missing observations and the standard deviation associated with each data source is indexed by the sample $i$. 

{
\begin{mini*}<b>
    {\substack{\hat{\vx},\hat{\vkappa_i},\hat{\vdelta_i},\hat{\vomega_i}, \hat{\vtheta}, \hat{\vgamma}, \widehat{\mW},\hat{\vbeta}}}
    { \sum_{i \in \sS} \left(\lambda_{x}\frac{\| \vx_i- \overline{\vx}_i\|^2_2}{\overline{n}_{x_i}\overline{\sigma}_{x_i}} +\lambda_{t}\frac{ \|\vt_i - \overline{\vt}_i\|^2_2}{\overline{n}_{t_i}\overline{\sigma}_{t_i}} 
    + \lambda_{e}\frac{\|\vx_i - \hat{\vx}_i\|^2_2}{|\sA|\overline{\sigma}_{x_i}}\right)}{}{}
  	\addConstraint{\vy_i}{=\sum^{b}_{k=1} \beta_k\left(\frac{\hat{\vx}_i}{\overline{\vx}^{\max}}\right)^k}{\quad \forall i \in \sS}
	 \addConstraint{\tilde{\vt}_i}{=\overline{\vt}^{\min} \circ \left(1 +\left(\overline{\mE}\circ\widehat{\mW} \right) \vy_i\right)}{\quad \forall i \in \sS}
    \addConstraint{\vv_i}{=\overline{\mD}^{\top} \left(\begin{bmatrix} \tilde{\vt}_i & \overline{\mZ}_i \end{bmatrix} \hat{\vtheta}_i + \vgamma\right)}{\quad \forall i \in \sS}
    \addConstraint{\vp_i}{=\frac{\exp\left(\displaystyle \vv_i  \right) }{\displaystyle \overline{\mM}^{\top}\overline{\mM}\exp(\vv_i)}}{\quad \forall i \in \sS}
    \addConstraint{	\vg_i}{= \overline{\mO}_i \hat{\vkappa}_i + \hat{\vdelta}_i}{\quad \forall i \in \sS}
    \addConstraint{	\vphi_i}{= \frac{\exp\left(\displaystyle \hat{\vomega_i} \right)}{\displaystyle \overline{\mD}^{\top}\overline{\mD} \exp(\hat{\vomega_i}  )}}{\quad \forall i \in \sS}
    \addConstraint{\vq_i}{=\left(\overline{\mL}^{\top}\hat{\vg}_i\right) \circ \vphi_i}{\quad \forall i \in \sS}
    \addConstraint{\vf_i}{=\left(\overline{\mM}^{\top}\vq_i\right) \circ \vp_i}{\quad \forall i \in \sS}
    \addConstraint{\vx_i}{={\overline{\mD}}\vf_i }{\quad \forall i \in \sS}
    \addConstraint{\tilde{\vy}_i}{=\sum^{b}_{k=1} \beta_k\left(\frac{\vx_i}{\overline{\vx}^{\max}}\right)^k}{\quad \forall i \in \sS}
	 \addConstraint{\vt_i}{=\overline{\vt}^{\min} \circ \left(1 + \left(\overline{\mE}\circ\widehat{\mW} \right) \tilde{\vy}_i\right)}{\quad \forall i \in \sS}
    \addConstraint{\hat{\vx}_i,\vg_i}{\geq \vzero}{\quad \forall i \in \sS}
    \addConstraint{\textmd{diag}(\widehat{\mW}),\hat{\vbeta}}{> \vzero}{}
\end{mini*}
}

\section{Solution algorithm}
\label{sec:solution-algorithm}

This section first describes the main steps of our solution algorithm, including the forward and backward operations performed in the computational graph and the parameter updating scheme. Then, it derives some properties of the solution obtained with the algorithm. Appendix \ref{appendix:sec:illustrative-example} presents an example of the application of the algorithm and the computational graph associated with a toy transportation network. 

\newlength\myindent
\setlength\myindent{2em}
\newcommand\bindent{%
  \begingroup
  \setlength{\itemindent}{\myindent}
  \addtolength{\algorithmicindent}{\myindent}
}
\newcommand\eindent{\endgroup}

\begin{algorithm}[!htbp] 
    \small %
	\captionsetup{font=small}
        \caption{Macroscopic Traffic Estimator (\MTP)}
	\label{alg:tvgodlulpe}
	\scalebox{1}{%
		\begin{minipage}{\linewidth}
			\begin{algorithmic} 

                \State \State \textit{Step 0: Initialization.} 
                \begin{algsubstates}
                \State Initialize the model parameters $\hat{\vx}_i, \hat{\vt}_i, \hat{\vtheta}_i, \hat{\vgamma}, \widehat{\mW}, \widehat{\vbeta}, \hat{\vkappa}_i, \hat{\vdelta}_i, \hat{\vomega}_i, \forall i \in \sS$
                \State Process input $\overline{\mD}, \overline{\mM}, \overline{\mL}, \overline{\mZ}, \overline{\mO}_i, \overline{\vx}^{\textmd{max}}, \overline{\vt}^{\textmd{min}}$, $\overline{\vx}_i$, $\overline{\vt}_i, \forall i \in \sS$
                \end{algsubstates}
                
                \State
				
				\State \textit{Step 1: Forward Iteration.} 
    			
				\begin{algsubstates}
                    \State Compute $\vx_i, \vt_i, \forall i \in \sS$ with Eqs. (\ref{eq:forward-pass-polynomial-features})-(\ref{eq:forward-pass-link-traveltime-output}). 
                    \State Compute the loss function $\ell$ with Eq. \ref{eq:forward-pass-loss-function}. 
					
				\end{algsubstates}

                \State
				
				\State \textit{Step 2: Backward iteration}. Compute the gradient of the loss function $\ell$ with respect to the model parameters. 

                \State 
 
                \State \textit{Step 3: Parameters updating}. Update model parameters using a gradient-based method. 
                
                \State 
                
                \State \textit{Step 4: Termination criterion}. If the number of epochs is lower than the maximum epochs, go back to Step 1. If not, terminate the algorithm.
              
			\end{algorithmic}
		\end{minipage}%
	}
\end{algorithm}

\subsection{Parameter initialization} 
\label{ssec:initialization}

The initialization of the model's parameters can significantly impact the convergence to a solution due to the non-convexity of the optimization problem. A convenient strategy for parameter initialization is to leverage prior domain knowledge. For example, it could be convenient to initialize the parameters of the utility function with estimates obtained from travel behavior studies. The parameters $\vdelta$ of the trip generation model can be initialized with reference values of the generated trips when available. Similar to \citet{guarda_estimating_2024}, the link flow parameters $\hat{\vx}$ are initialized with the link flow obtained from a single pass of logit-based stochastic traffic assignment, assuming that travel times are equal to the free flow travel times. The subset of parameters that are a function of the set of learnable parameters does not require explicit initialization (Section \ref{ssec:parameters}). 

\subsection{Forward pass}
\label{ssec:forward}

A forward pass consists of a chain of functions or layers applied to an input to generate output values. Suppose that all learnable parameters of the model (Section \ref{ssec:parameters}) are initialized to a set of feasible values. Without loss of generality, let's pick a single sample $i \in \sS$ from the dataset. A forward pass in our computational graph performs the following operations to compute the loss function and the output of each layer:

\begin{align}
\label{eq:forward-pass-polynomial-features}
\vy_i&=\sum^{b}_{k=1}\beta_k\left(\frac{\hat{\vx}_i}{\overline{\vx}^{\max}}\right)^k \\
\label{eq:forward-pass-link-traveltime}
\tilde{\vt}_i&=\overline{\vt}^{\min} \circ \left(1 + \left(\overline{\mE}\circ\widehat{\mW} \right) \vy_i\right)\\
\label{eq:forward-pass-path-utility}
\vv_i&=\overline{\mD}^{\top} \left(\begin{bmatrix} \tilde{\vt}_i & \overline{\mZ}_i \end{bmatrix} \hat{\vtheta}_i + \vgamma\right)\\
\label{eq:forward-pass-path-probability}
\vp_i&=\frac{\exp\left(\displaystyle \vv_i  \right) }{\displaystyle \overline{\mM}^{\top}\overline{\mM}\exp(\vv_i)}\\
\label{eq:forward-pass-trip-generation}
\vg_i&= \overline{\mO}_i \hat{\vkappa}_i + \hat{\vdelta}_i\\
\label{eq:forward-pass-destination-probability}
\vphi_i&= \frac{\exp\left(\displaystyle \hat{\vomega_i} \right)}{\displaystyle \overline{\mD}^{\top}\overline{\mD} \exp(\hat{\vomega_i}  )}\\
\label{eq:forward-pass-od-flow}
\vq_i&=\left(\overline{\mL}^{\top}\hat{\vg}_i\right) \circ \vphi_i\\
\label{eq:forward-pass-path-flow}
\vf_i&=\left(\overline{\mM}^{\top}\vq_i\right) \circ \vp_i\\
\label{eq:forward-pass-link-flow}
\vx_i&={\overline{\mD}}\vf_i \\
\label{eq:forward-pass-polynomial-features-output}
\tilde{\vy}_i&=\sum^{b}_{k=1}\beta_k\left(\frac{\vx_i}{\overline{\vx}^{\max}}\right)^k \\
\label{eq:forward-pass-link-traveltime-output}
\vt_i&=\overline{\vt}^{\min} \circ \left(1 + \left(\overline{\mE}\circ\widehat{\mW} \right)\tilde{\vy}_i\right)\\
\label{eq:forward-pass-loss-function}
\ell_i &= \lambda_{\overline{x}}\| \vx_i- \overline{\vx}_i\|^2_2 +\lambda_{\overline{t}} \|\vt_i - \overline{\vt}_i\|^2_2  
+ \lambda_{e}\|\vx_i - \hat{\vx}_i\|^2_2
\end{align}

The parameters are then projected according to the inequality constraints in $\mathcal{P}$ (Section \ref{ssec:optimization-problem}).

\subsection{Backward pass} 
\label{ssec:backward}

The gradients of the parameters with respect to the loss function can be backpropagated through the layers of the computational graph using the derivative chain rule. The analytical form of the chain-rule derivatives for most parameters is provided by \citet{guarda_estimating_2024}. Implementation-wise, the chain-rule derivatives across the model layers are obtained through the automatic differentiation tools in TensorFlow.

\subsection{Parameter updating}
\label{ssec:parameter-updating}

The updating scheme follows the standard formula for gradient-based methods applied to minimization problems. At every step, the negative weight between the learning rate and the gradient of the loss function with respect to the parameters is added to the current parameter estimates. 

\subsection{Training and inference}
\label{ssec:other-design-aspects}

During model training, all parameters of the model are learned. During inference, the link flow parameters are adjusted such that the relative gap of the models in the two stages is the same. It is possible to enforce a different value for the relative gap during inference if the modeler believes this can maximize estimation accuracy. Because there is no access to observed link flow and travel time during inference,  all hyperparameters of the loss function except for $\lambda_e$ are set to 1. 

\subsection{Properties}
\label{sec:mathematical-properties}

The computational graph and solution algorithm presented in this paper enhances the \TVODLULPE model and \texttt{PESUELOGIT} algorithm developed by \citet{guarda_estimating_2024}. Figure \ref{fig:computational-graph}, Appendix \ref{appendix:sec:illustrative-example} highlights some of the main differences in the computational graph. Propositions \ref{prop:stalogit-solution} and \ref{prop:equilibrium-solution} summarize key properties of our model and solution algorithm. 


\begin{definition}[\STALOGIT]
\label{def:sta-logit}
At Logit-based Stochastic Traffic Assignment (\STALOGIT), path flows are distributed among the paths connecting each O-D pair according to a logit distribution.
\end{definition}

\begin{prop}[\STALOGIT solution]
    \label{prop:stalogit-solution}
    A forward pass of the computational graph $\mathcal{C}$ gives a link flow solution that is a valid Logit-based Stochastic Traffic Assignment (\STALOGIT).
\end{prop}

\begin{proof}
Consider the forward pass for a sample $j \in \sS$. Suppose the current value of the vector of O-D flows is $\vq^{\prime}$. The current path flow solution $\vf^{\prime}$ of the computational graph must satisfy Eq. \ref{eq:route-choice-constraint} for a given value $\vp^{\prime}$ of the path choice probabilities. Because $\vp^{\prime}$ is Logit-distributed (Eq. \ref{eq:path-choice-probabilities-constraint}), $\vf^{\prime}$ is a valid Logit-based Stochastic Traffic Assignment (\STALOGIT). Then, the current link flow solution $\vx^{\prime}$ is consistent with \STALOGIT, which completes the proof. 
\end{proof}

\begin{remark}
In \citet{guarda_estimating_2024}, the link flow solution is equated to the link flow parameters. Thus, the link flow parameters are a valid \STALOGIT only when the solution of the computational graph is at \SUELOGIT. As shown in this proposition, the computational graph for the \MTP model produces link flow solutions that satisfy \STALOGIT at any iteration regardless of the \SUELOGIT condition. Because the travel times in the last layer of the computational graph are a function of the link flows, they are also consistent with \STALOGIT. 
\end{remark}

\begin{definition}
\label{def:sue-logit}
At Stochastic User Equilibrium with logit assignment (\SUELOGIT), path flows are distributed among the paths connecting each O-D pair according to a logit distribution, where path costs are determined as functions of the assigned path flows \citep{Fisk1980}.
\end{definition}

\begin{prop}[Equilibrium solution]
    \label{prop:equilibrium-solution} 
    If the loss function of \MTP only includes the equilibrium component,  Algorithm \ref{alg:tvgodlulpe} solves the \SUELOGIT problem and provides solutions that are a valid \STALOGIT
\end{prop}

\begin{proof}
This problem equates to finding a fixed point in link flow space in the computational graph. As shown in Eq. \ref{eq:fixed-point-sue-constraint},  this problem is equivalent to the \SUELOGIT problem. By Proposition \ref{prop:stalogit-solution}, at any iteration, the solution of \MTP provides a valid \STALOGIT, which completes the proof. 
\end{proof}

\begin{remark}
Proposition \ref{prop:equilibrium-solution} together with Proposition \ref{prop:stalogit-solution} show that the \MTP algorithm provides consistent solutions in settings where the equilibrium condition may or may not hold. Note that if the only learnable parameters are the link flow parameters $\hat{\vx}$ and the neural performance functions are monotonically increasing (Assumption \ref{assumption:performance-functions}), the \SUELOGIT problem is strictly convex in the path and link flow space (Proposition 2, \citet{guarda_statistical_2024}). 
\end{remark}

%% file: sections/experiments.tex
\section{Numerical experiments}
\label{sec:numerical-experiments}

To study the performance of our algorithm to estimate network flows and travel times, we conduct experiments with synthetic data generated from the Sioux Falls SD network. The network comprises 24 nodes and 76 links (Figure \ref{subfig:sioux-falls-network}, Appendix \ref{appendix:sec:networks}). The O-D matrix is obtained from \citet{TNTP}, and we generate a set of 1,584 paths corresponding to the three shortest paths among 528 O-D pairs. We employ a validation framework where the ground truth values of the model parameters are assumed to be known and used to generate synthetic measurements of traffic flow and travel time consistent with \SUELOGIT. All the experiments are conducted on an Apple M2 Pro with a 10‑core CPU, 16 GB of unified memory, and 1 TB SSD. The runtime of the experiments presented in this section is approximately 4 hours. All models are trained with TensorFlow \citep{abadi_tensorflow_2016} and using the Adam optimizer \citep{Kingma2015}.

\subsection{Models specifications}

Table \ref{table:models-specifications-siouxfalls} describes the three model specifications tested in our experiments. The \SUELOGIT model performs logit-based stochastic user equilibrium for each period of the day, and it has only the link flow parameters $\hat{\vx}$ as learnable parameters. The \TVODLULPE model formulated by \citet{guarda_estimating_2024} is used as a benchmark against \MTP developed in this paper. The \TVODLULPE estimates performance functions of BPR class with link-specific parameters, O-D matrices that are period-specific, and utility function parameters that are both feature-specific and period-specific. The utility function also includes link-specific parameters that capture any effect not captured through the feature-specific attributes. The \MTP estimates the number of generated trips per location using location-specific fixed effects, and it uses origin-destination-specific fixed effects and softmax functions to compute the number of trips per O-D pair, namely, the O-D matrix. Both sets of parameters are also assumed to be period-specific. The link performance function of \MTP is modeled with three polynomial features ($b = 3$) and with an MLP layer that captures the traffic flow interactions among links. Both the \TVODLULPE and \MTP define the link flow parameters as learnable. 

\begin{table}[H]
\begin{adjustbox}{width=\textwidth}  
\renewcommand{\arraystretch}{1.05}
\centering
\begin{threeparttable}
\caption{Models specifications and parameters learned with synthetic data from Sioux Falls network}
\label{table:models-specifications-siouxfalls}
\begin{tabular}{ccccccc}
\hline
\multirow{2}{*}[-0.5cm]{Model} & \multicolumn{6}{c}{Parameters} \\ \cline{2-7} 
 & \begin{tabular}[c]{@{}c@{}}Link\\ flows\end{tabular} & \begin{tabular}[c]{@{}c@{}}Utility \\ function \end{tabular} & \begin{tabular}[c]{@{}c@{}} Generation \end{tabular} & \begin{tabular}[c]{@{}c@{}}O-D\\ estimation\end{tabular}  & \begin{tabular}[c]{@{}c@{}}Link \\ performance \\ function \end{tabular} & Total \\ \hline
\SUELOGIT & 76$T$ & $-$& $-$ & $-$ & $-$ & $152$ \\
\TVODLULPE & 76$T$ & 76 + 3$T$ & $-$ & 528$T$ & $2\times 76$ &
$228 + 607T = 1,442$
\\
\MTP & 76$T$ & 76 + 3$T$ & 76T & 528$T$ & $P + \|\mW\|_0  = 3+864$ & 
$943 + 683T = 2,309$
\\ \hline
\end{tabular}
\begin{tablenotes}
      \footnotesize
        \item Note: $0 \leq \|\mW\|_0 \leq |\sA|^2$ corresponds to the number of cells in the flow interaction matrix that are different than zero. 
        \item $T$: number of distinct periods
        \item $P$: hyperparameter that defines the degree of the polynomial used to represent the link performance function. 
    \end{tablenotes}
\end{threeparttable}
\end{adjustbox}
\end{table}

\subsection{Data generation}
\label{ssec:data-generation}

To generate synthetic data, we compute logit-based stochastic user equilibrium (\SUELOGIT). In the \SUELOGIT model, only the link flow parameters $\hat{\vx}$ are learnable. An O-D matrix for each period is generated using a Gaussian distribution with a mean equal to the true O-D matrix obtained from \citep{TNTP} and a standard deviation equal 10\% of the average number of trips per O-D pair. The link cost performance functions are assumed of \BPR class with parameters $\alpha = 0.15$ and $\beta=4$. The free flow travel times and link capacities are obtained from \citet{TNTP}. The utility function is assumed to be linear in parameters and dependent on the travel time $t$ and two exogenous attributes: the standard deviation of travel time $v$ and the number of street intersections per mile $s$. The exogenous attributes of each link are generated using uniform random variables bounded between 0 and 1. The parameters weighting the attributes of the utility function are assumed common among travelers and equal to $\theta_t = -1, \ \theta_v = -1.3, \ \theta_r = -3$, such that the reliability ratio $\texttt{RR} =  \theta_t/\theta_v$ is equal to $1.3$.

With the \SUELOGIT model, we generate 300 samples comprising 22,800 measurements of traffic flow and travel time. Each sample is associated with a set of measurements collected on a given day and at a period of that day in each link of the network. This resembles a real-world use of this model trained by data collected over 300 time periods. A third of the samples are generated with the ground truth O-D matrix, which is assumed to be fixed and equal for all samples. To represent a period of lower traffic congestion, we generate the remaining 200 samples with the same O-D matrix but scaled by a factor of 0.8. Thus, we create a dataset with observations collected from two periods or hours of the day, each comprising 200 and 100 samples representing traffic conditions during peak and off-peak hours, respectively. To introduce randomness in the travel time and traffic flow measurements collected between days, we add a Gaussian error term with zero mean and standard deviation equal to 5\% of the true mean of the observed measurements. 

Figure \ref{fig:siouxfalls-flow-traveltime} shows a scatter plot with the observed travel times and traffic flow measurements obtained with the data generation process during the off-peak (orange) and peak hours (blue). The clouds of points around the true values of the observed measurements result from adding Isotropic Gaussian Noise in the data-generating process. Because the O-D matrix for the off-peak period is scaled by a factor of 0.8, the traffic flows and travel times are consistently lower than during the peak period. Due to the heterogeneity in the link's capacities and free flow speeds, the correlation between observed travel time and traffic flow is positive but small.

\begin{figure}[H]
	\centering
	\includegraphics[width=0.3\textwidth]{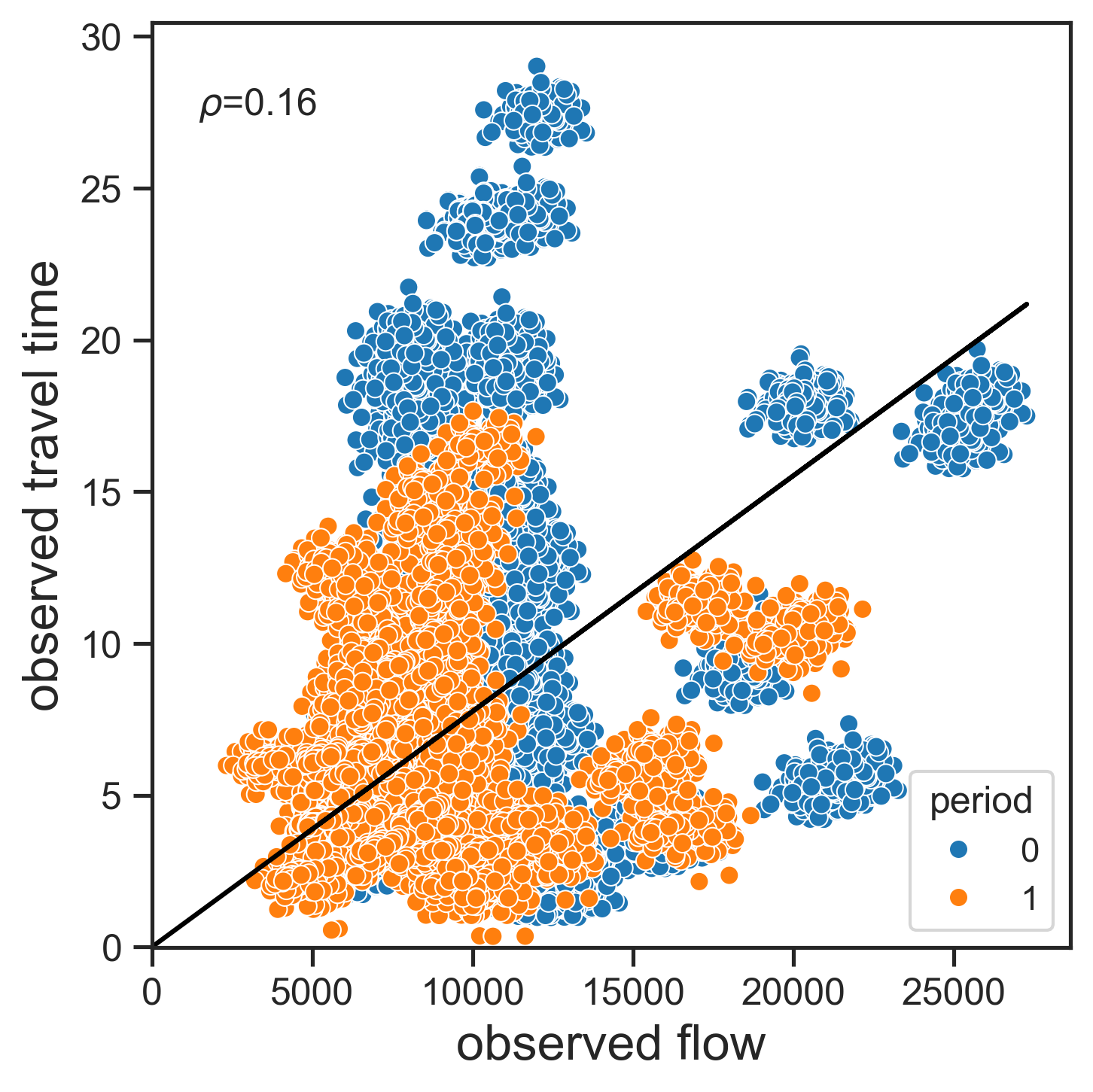}
	\caption{Synthetic traffic flow and travel time measurements generated in the Sioux Falls network during off-peak (orange) and peak periods (blue)}
	\label{fig:siouxfalls-flow-traveltime-synthetic-data}
\end{figure}

To train the \TVODLULPE and \MTP models, it is assumed that there is only access to a noisy measurement of the ground truth number of generated trips. To train \MTP, the reference trip generation $(\bar{\vg})$ is obtained by aggregating by node the O-D matrix used to compute \SUELOGIT. The reference O-D demand $(\bar{\vq})$ used to train the \TVODLULPE model is generated by uniformly distributing the generated trips in $\bar{\vg}$ among the set of reachable destinations from each node. 

\subsection{Hyperparameters and loss function}
\label{ssec:hyperparameters-experiments}

The number of epochs is 60 and the learning rate is $0.05$. This learning rate ensures that the \SUELOGIT model satisfies a relative gap threshold of $10^{-4}$ and that the final solution of the other models minimizes the estimation error of traffic flows and travel times. To speed up convergence, the batch size is set to 1, which means that gradient updates are conducted sequentially along each minibatch. Thus, each gradient update uses the 76 observations of traffic flow and travel time available for each sample.

The choice of hyperparameters weighting the components of the loss function is straightforward. The hyperparameters are set to 1 or 0 according to the parameters learned in the models. For example, the estimation of \SUELOGIT only seeks to find link flows that are fixed points, namely, where the equilibrium component becomes zero. Thus, the weights of all components in the loss function except for the equilibrium component are set to zero at training time. In the \TVODLULPE and \MTP models, the weights of the loss components associated with the estimation error of traffic flow and travel time are set to one such that the observed measurements of traffic flow and travel time can be reproduced with the model. The deviation between (i) the estimated and reference O-D matrices or (ii) the estimated and reference generation vectors to the loss functions are excluded in the loss function since these terms are not required for achieving good in-sample performance. The mean squared error (MSE) is chosen as the distance metric to compute each component of the loss function, but since the optimization algorithm minimizes the sum of a normalization of the loss components (Section \ref{ssec:loss-function}), the MSE of each component could not monotonically decrease over epochs. 

\subsection{Parameters initialization and constraints}
\label{ssec:parameters-initialization-experiments}

In the \TVODLULPE model, the parameters $\alpha$ and $\beta$ of the BPR function are initialized to 1 and constrained to take values in the range $]0, 10]$. The parameters of the O-D matrices are constrained to be non-negative and initialized with the values of the reference O-D matrices created during the data generation process (Section \ref{ssec:data-generation}). In \MTP, the parameters weighting the polynomial features are initialized to 1 and constrained to take non-negative values. This approach guarantees that the link performance functions are not monotonically decreasing\footnote{Further research can look at establishing milder constraints on the parameters' space to ensure that the link performance functions are monotonically non-decreasing.}. In addition, the weight matrix of the MLP layer is masked such that only the elements in the diagonal and those associated with the downstream and upstream flow of a link are equal to 1. In line with the assumption of the increasing monotonic relationship between link flow and travel time, we constrain the diagonal terms of the kernel matrix $\mW$ of the MLP layer to be positive. Under the assumption that an increase in the traffic flow in a link should not reduce the travel time in another link, we constrain the non-diagonal terms of $\mW$ to be non-negative. For consistency with the upper bounds defined for the BPR parameters, we constrain the parameters of $\mW$ to be no greater than 10. 

The fixed effects of the generation layer are initialized with the reference generation vectors. The utilities associated with the O-D pairs are all initialized to zero, which induces an even distribution of the generated trips over the set of destinations reachable from each origin location. To represent a setting where the correct signs but not the magnitudes of the feature-specific parameters of the utility function are known, the \TVODLULPE and \MTP models initialize all these parameters to -$1$. In the three models, the initial values for the link flow parameters are equated to the output of a single pass of logit-based stochastic traffic assignment in the transportation network. 

\subsection{In-sample performance}

This section studies if the \SUELOGIT, \TVODLULPE, \MTP models can find solutions that satisfy the network equilibrium conditions and reproduce the synthetic measurements of traffic flow and travel time of the training set. In line with the formulation of \MTP, the computational graph for the \TVODLULPE formulated by \citet{guarda_estimating_2024} is modified such that the estimated travel time and traffic flow are obtained from the layers located on the right-hand side of Figure \ref{fig:computational-graph}, Appendix \ref{appendix:sec:illustrative-example}. We also analyze the estimated parameters associated with the trip generation, the O-D matrix, the utility function, and the link performance functions obtained with the \TVODLULPE and \MTP models. We refer to these analyses as \textit{in-sample} because estimations are made in links that report traffic flow and travel time observations in the training set. To study the convergence of the learning algorithm, we compute the change of the relative value of the loss components over epochs, namely, the percentage difference in the value of a loss component at a given epoch with respect to the initial epoch. We compute the Mean Absolute Percentage Error (MAPE) to analyze model estimation performance. 

\subsubsection{Logit-based stochastic user equilibrium (\SUELOGIT)}

Figure \ref{fig:siouxfalls-convergence-suelogit} shows the changes in the relative gap and the MSE over epochs when training the \SUELOGIT model. A relative gap of $4.9\cdot10^{-4}$ is achieved in 18 epochs with a runtime of 63.4 seconds. Figure \ref{fig:siouxfalls-flow-traveltime-suelogit} shows a scatter plot comparing the observed and estimated values of travel time and traffic flow. Although the set of observed travel times and traffic flows are not used for model training, the estimated values obtained at equilibrium closely match the observed measurements. This result is due to (i) the synthetic measurements of traffic flow and travel times are generated to satisfy the \SUELOGIT equilibrium condition (Section \ref{ssec:data-generation}) and (ii) the strict convexity of the \SUELOGIT that holds under our experimental setting and guarantees the existence of a unique global optimum in link flow space \citep{guarda_estimating_2024}.

\begin{figure}[H]
\centering
\begin{subfigure}[t]{0.3\columnwidth}
	\centering
	\includegraphics[width=\textwidth, trim= {0cm 0cm 0cm 0cm},clip]{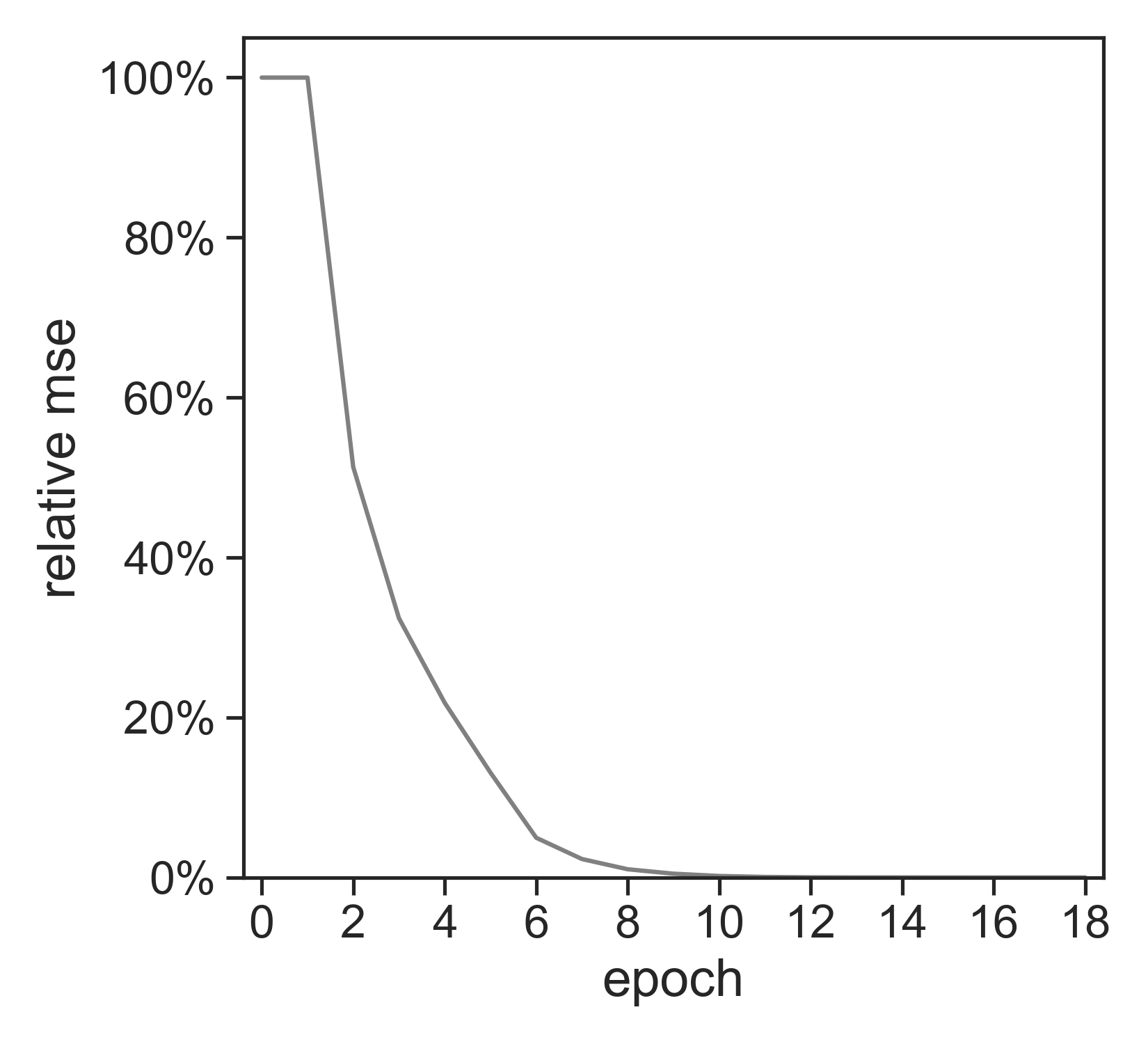}
	
    \caption{Change in relative MSE}
	\label{subfig:siouxfalls-convergence-suelogit}
\end{subfigure}
\begin{subfigure}[t]{0.3\columnwidth}
    \centering
	\includegraphics[width=\textwidth]{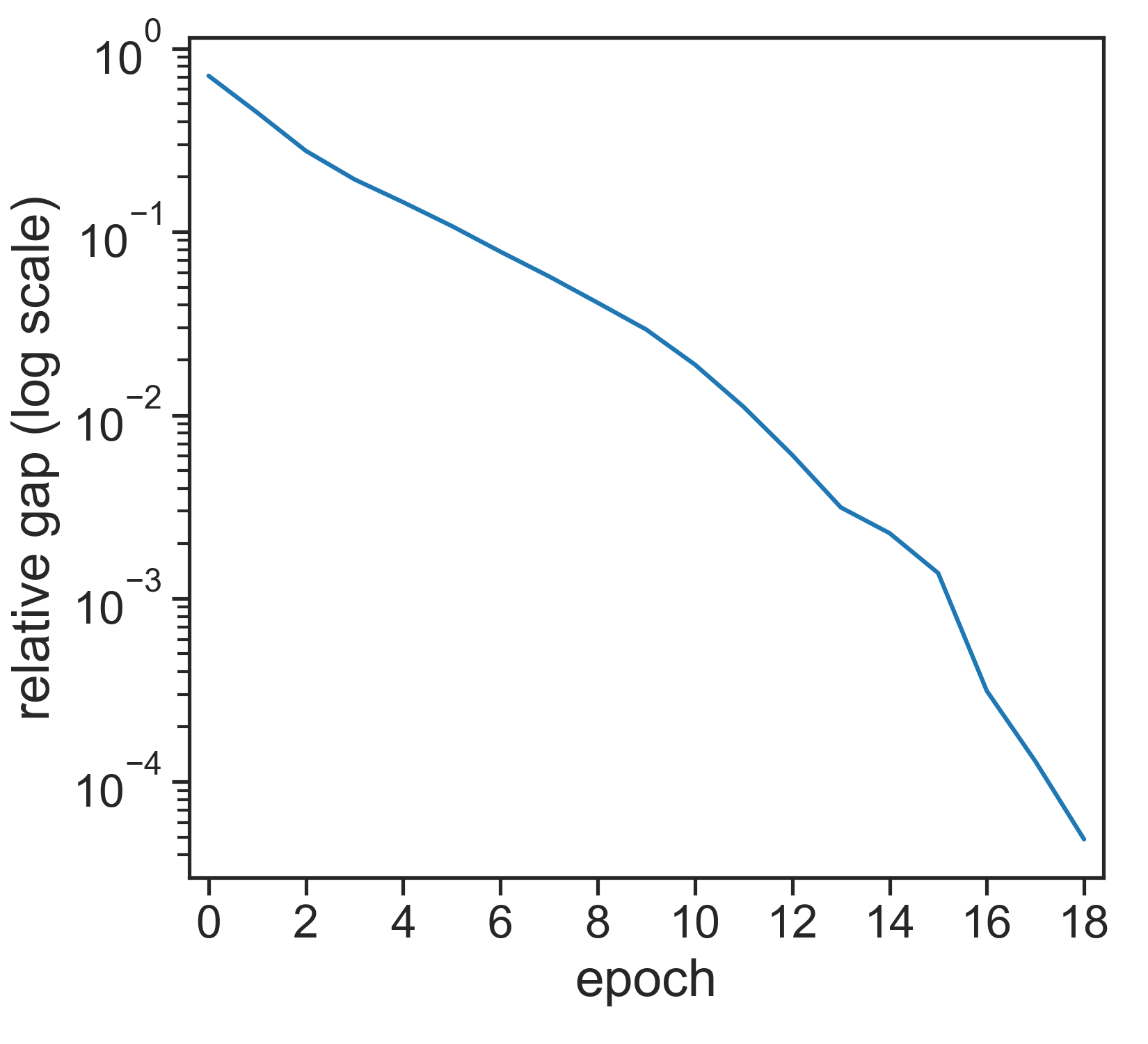}
	\caption{Change in relative gap}
	\label{subfig:siouxfalls-relative-gap-suelogit}
\end{subfigure}

\caption{Convergence of logit-based stochastic user equilibrium (\SUELOGIT) model using synthetic data from the Sioux Falls network}
\label{fig:siouxfalls-convergence-suelogit}

\end{figure}

\subsubsection{Benchmark model (\TVODLULPE)}
\label{sec:siouxfalls-tvodlulpe}

Figure \ref{fig:siouxfalls-convergence-tvodlulpe} shows the change in relative MSE, MAPE, and relative gap over epochs while training the \TVODLULPE model. A relative gap of $8.8\cdot10^{-3}$ is achieved in the final epoch (Figure \ref{subfig:siouxfalls-relative-gap-tvodlulpe}) and with a total runtime of 331 seconds. The relative MSE decreases steadily over epochs (Figure \ref{subfig:siouxfalls-relative-mse-tvodlulpe}) and the MAPEs of traffic flow and travel time at the final epoch are equal to 4.8\% and 7.1\%, respectively (Figure \ref{subfig:siouxfalls-mape-tvodlulpe}). The bottom plots in Figure \ref{fig:siouxfalls-flow-traveltime-tvodlulpe}, Appendix \ref{appendix:ssec:siouxfalls-model-training} suggest that estimations are accurate in the two hourly periods.

\begin{figure}[H]
\centering
\begin{subfigure}[t]{0.3\columnwidth}
    \centering
	\includegraphics[width=\textwidth]{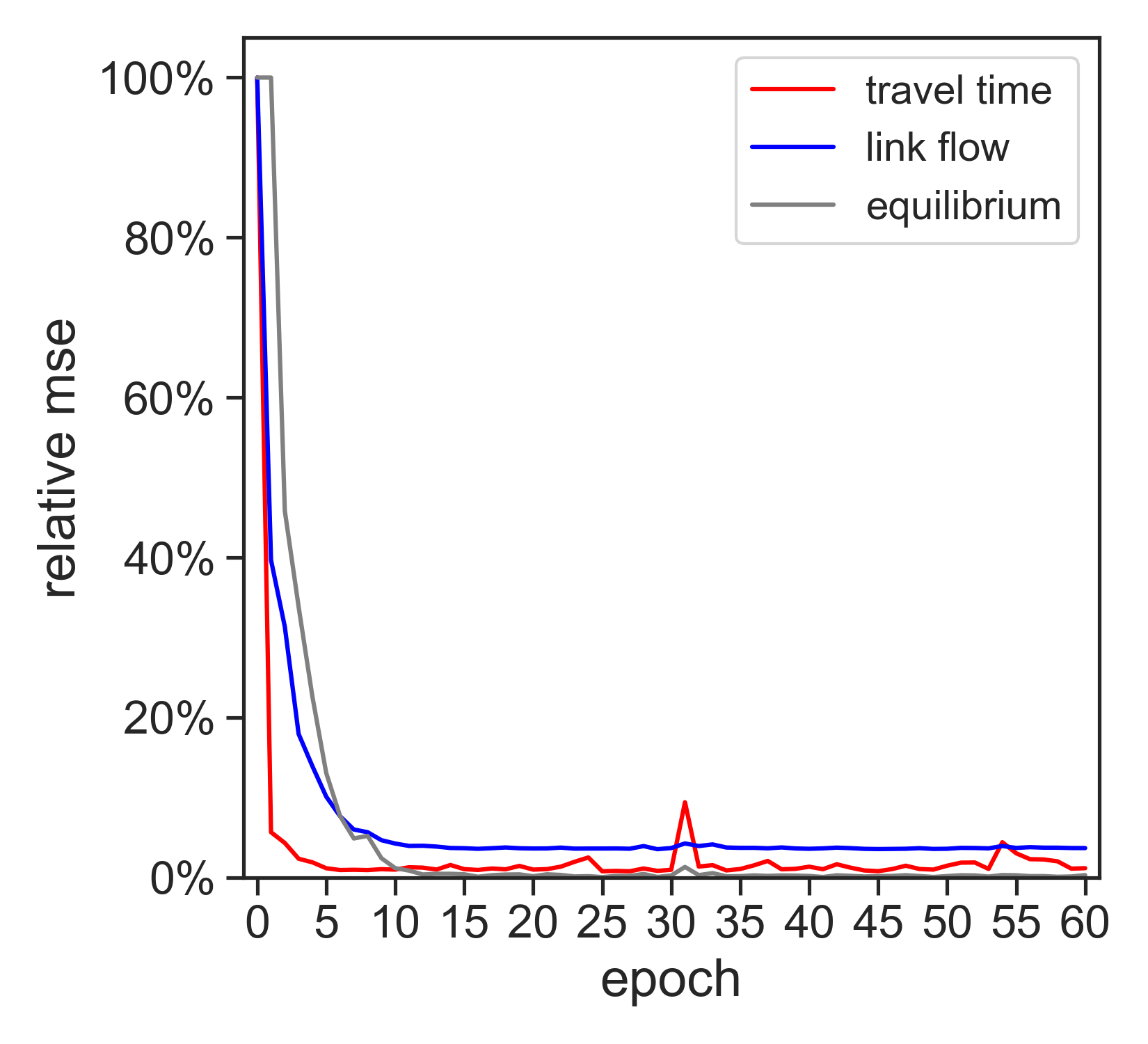}
    \caption{Relative MSE}
	\label{subfig:siouxfalls-relative-mse-tvodlulpe}
\end{subfigure}
\begin{subfigure}[t]{0.3\columnwidth}
	\centering
    \includegraphics[width=\textwidth]{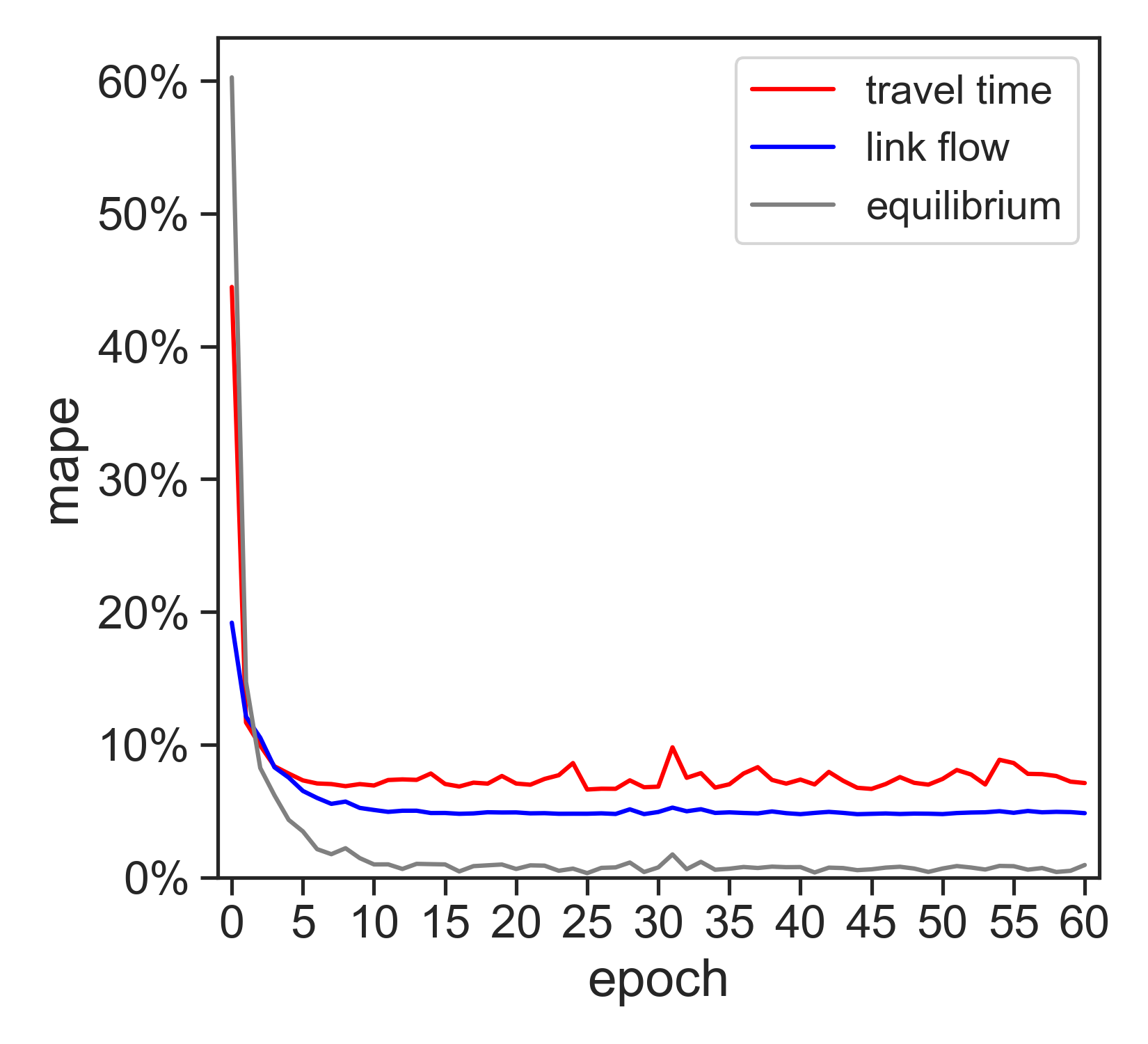}
    \caption{MAPE}
	\label{subfig:siouxfalls-mape-tvodlulpe}
\end{subfigure}
\begin{subfigure}[t]{0.3\columnwidth}
	\centering
    \includegraphics[width=\textwidth]{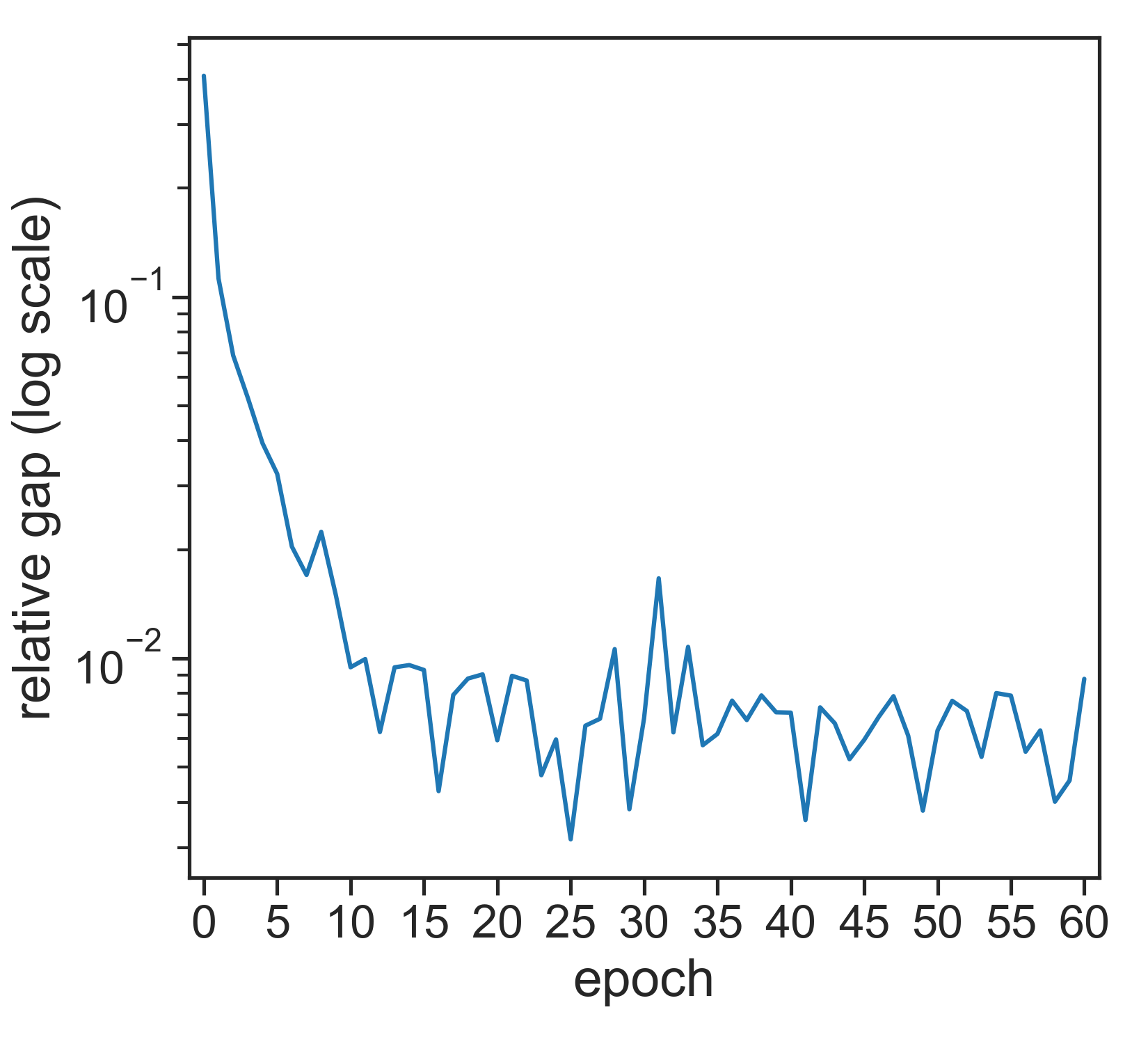}
    \caption{Relative gap}
	\label{subfig:siouxfalls-relative-gap-tvodlulpe}
\end{subfigure}
\caption{Convergence of \TVODLULPE model using synthetic data from Sioux Falls network}
\label{fig:siouxfalls-convergence-tvodlulpe}

\end{figure}

\subsubsection{Our model (\MTP)}
\label{sec:siouxfalls-mate}
 
Figure \ref{fig:siouxfalls-convergence-mate} shows the change in relative MSE, MAPE, and relative gap over epochs while training \MTP. The MAPE is reported from the second epoch to properly assess the reduction of MAPE over epochs (Figure \ref{subfig:siouxfalls-mape-mate}). \MTP uses a more flexible class of link performance functions and estimates a trip generation model that should make finding a solution more challenging than \TVODLULPE. Regardless, the final MAPEs of traffic flow and travel time are 5.6\% and 8.3\%, respectively, and slightly higher than the \TVODLULPE model. As shown in Figure \ref{subfig:siouxfalls-relative-gap-mate}, the relative gap is 0.018 in the final epoch. The relative MSE decreases steadily over epochs, and the large drop observed between epochs 0 and 1 is consistent with the corresponding drop in MAPE. The scatter plots at the bottom of Figure \ref{fig:siouxfalls-flow-traveltime-mate}, Appendix \ref{appendix:ssec:siouxfalls-model-training} suggest that estimations are accurate in the two hourly periods. 

\begin{figure}[H]
\centering
\begin{subfigure}[t]{0.3\columnwidth}
    \centering
	\includegraphics[width=\textwidth]{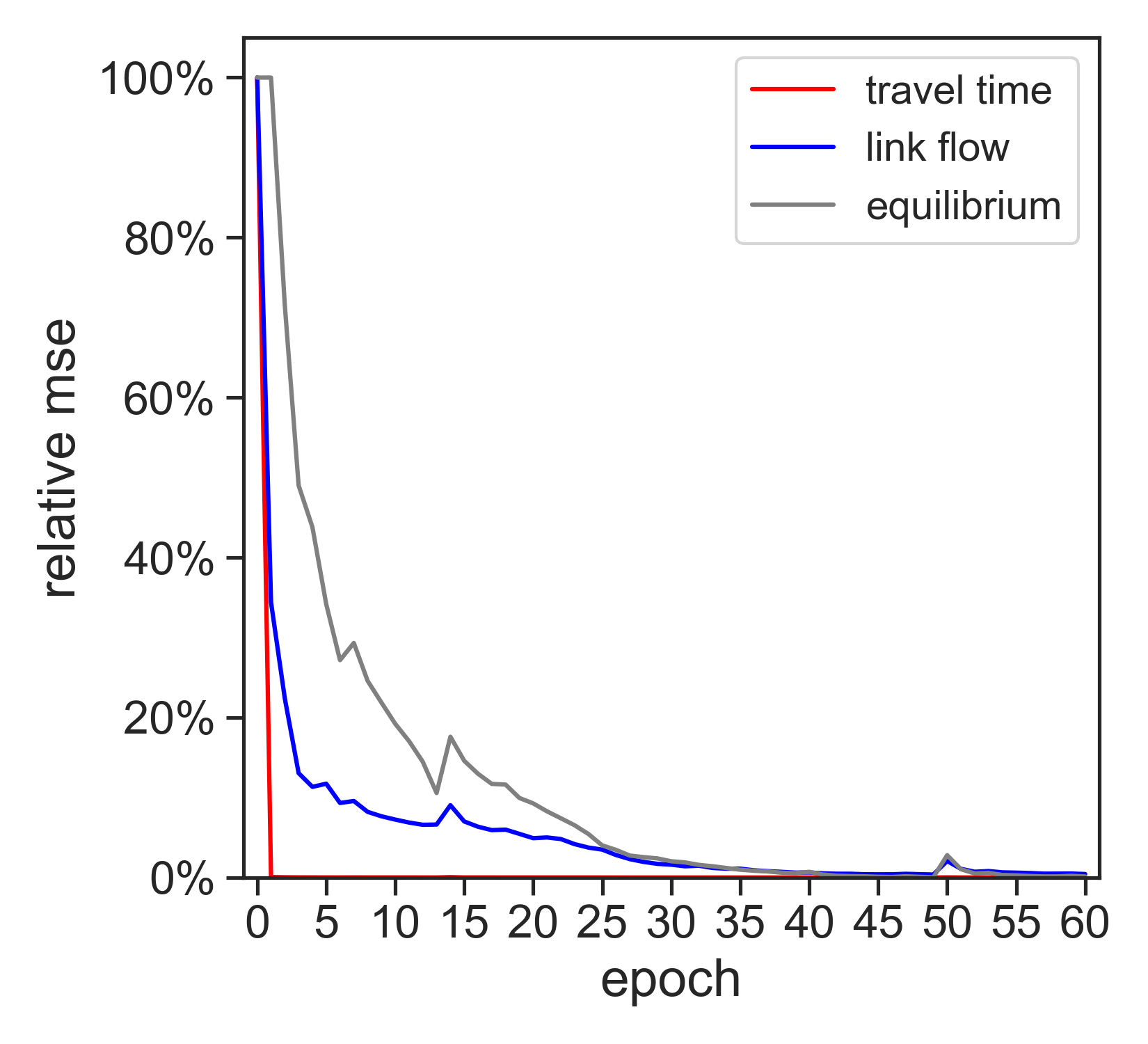}
    \caption{Relative MSE}
	\label{subfig:siouxfalls-relative-mse-mate}
\end{subfigure}
\begin{subfigure}[t]{0.3\columnwidth}
	\centering
    \includegraphics[width=\textwidth]{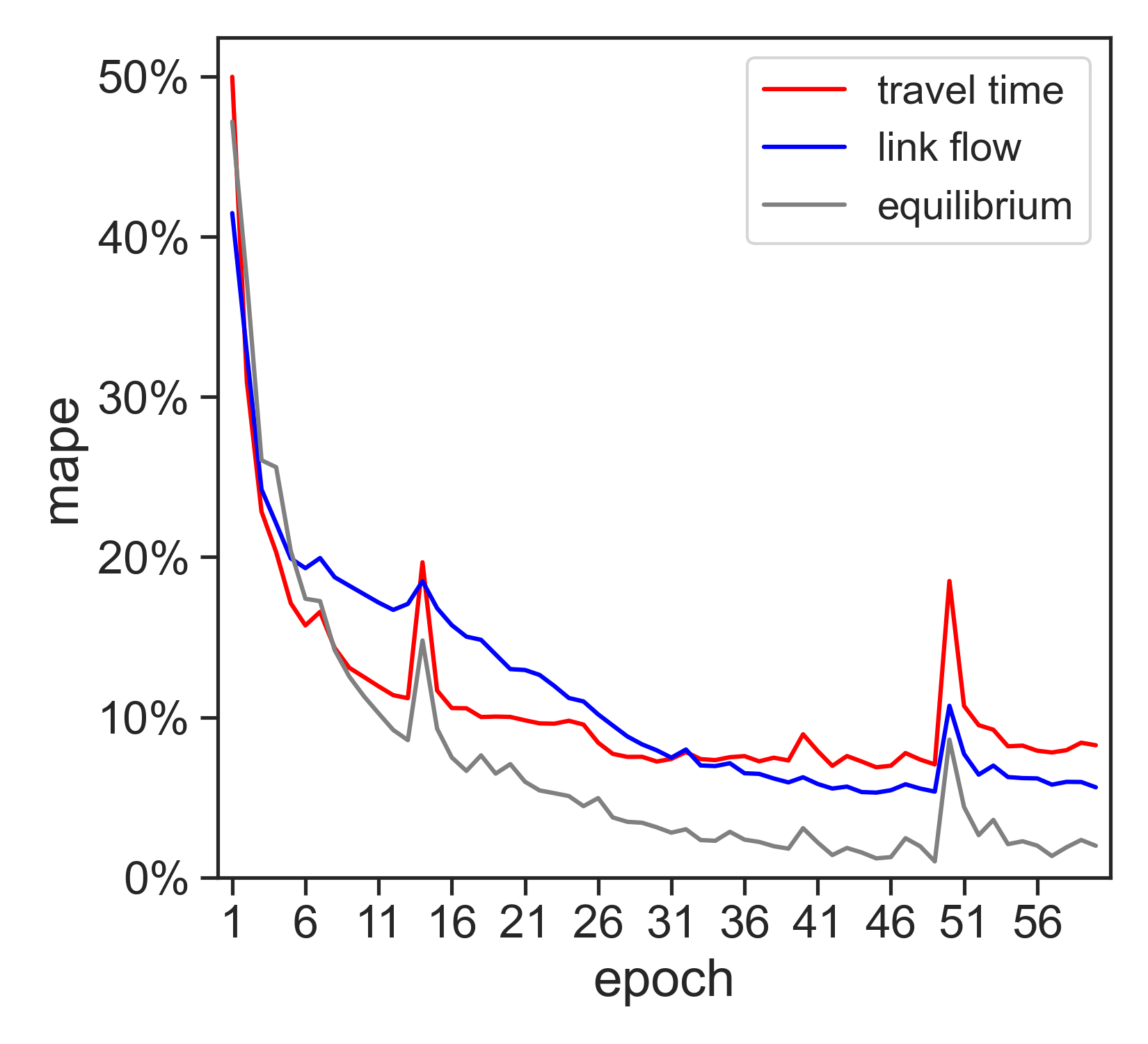}
    \caption{MAPE}
	\label{subfig:siouxfalls-mape-mate}
\end{subfigure}
\begin{subfigure}[t]{0.3\columnwidth}
	\centering
    \includegraphics[width=\textwidth]{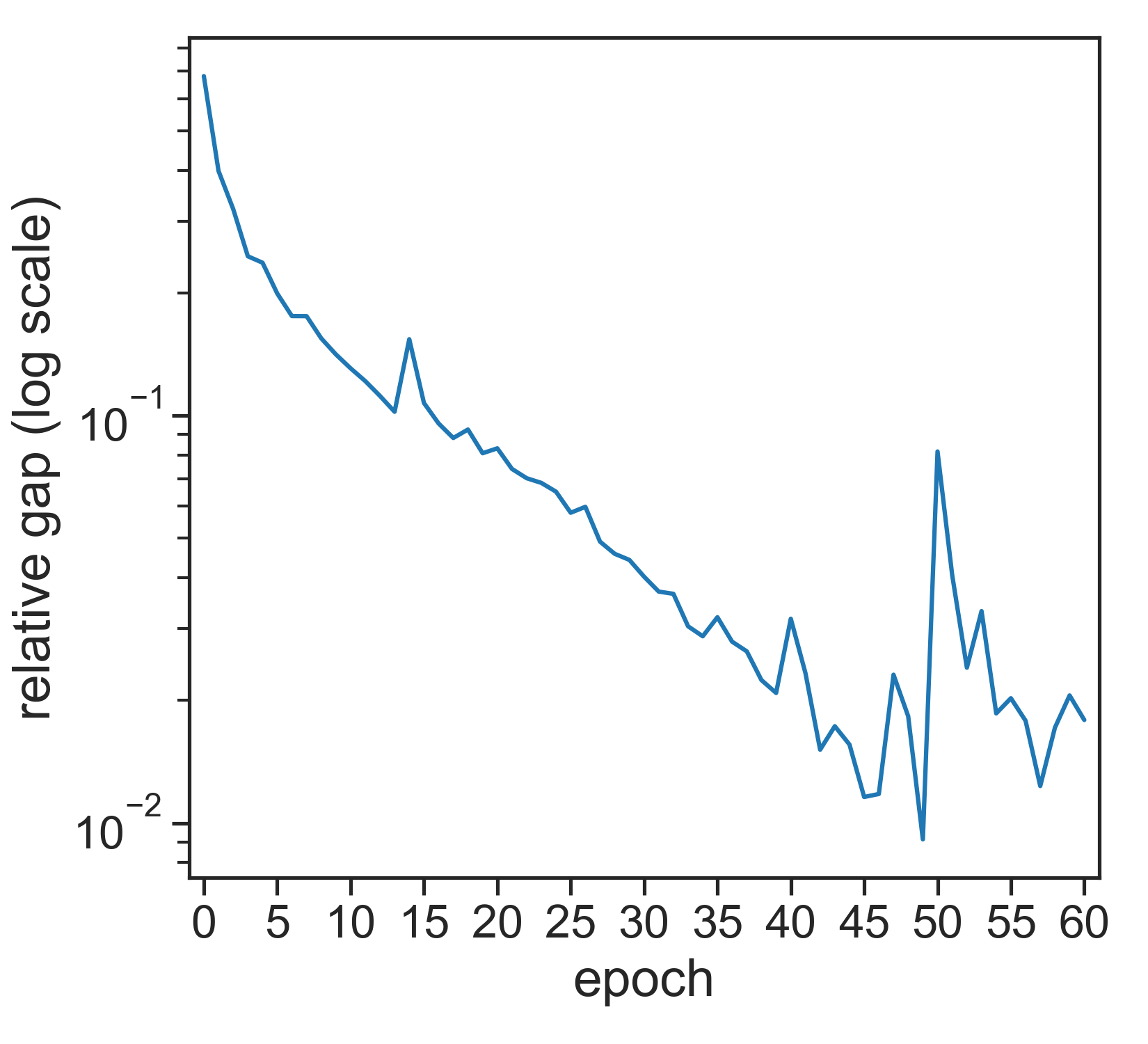}
    \caption{Relative gap}
	\label{subfig:siouxfalls-relative-gap-mate}
\end{subfigure}
\caption{Convergence of \MTP model using synthetic data from Sioux Falls network}
\label{fig:siouxfalls-convergence-mate}

\end{figure}

\subsubsection{Utility function}
\label{sssec:siouxfalls-utility-function}

Due to the high non-convexity of the learning problem, the initialization of parameters of the utility function (Section \ref{ssec:parameters-initialization-experiments}) is expected to have a high impact on the recovery of the ground truth parameters. However, for the chosen parameter initialization, we find no significant differences in the magnitudes of the utility function parameters estimated with the \TVODLULPE and \MTP models (Figure \ref{fig:siouxfalls-utility-function-estimation}). While the estimate values are not equal to their ground truth values, they are negative, and the reliability ratios are higher than 1 in both hourly periods. 

\begin{figure}[!htbp]
\centering

\begin{subfigure}[t]{0.3\columnwidth}
    \centering
 \includegraphics[width=\textwidth]{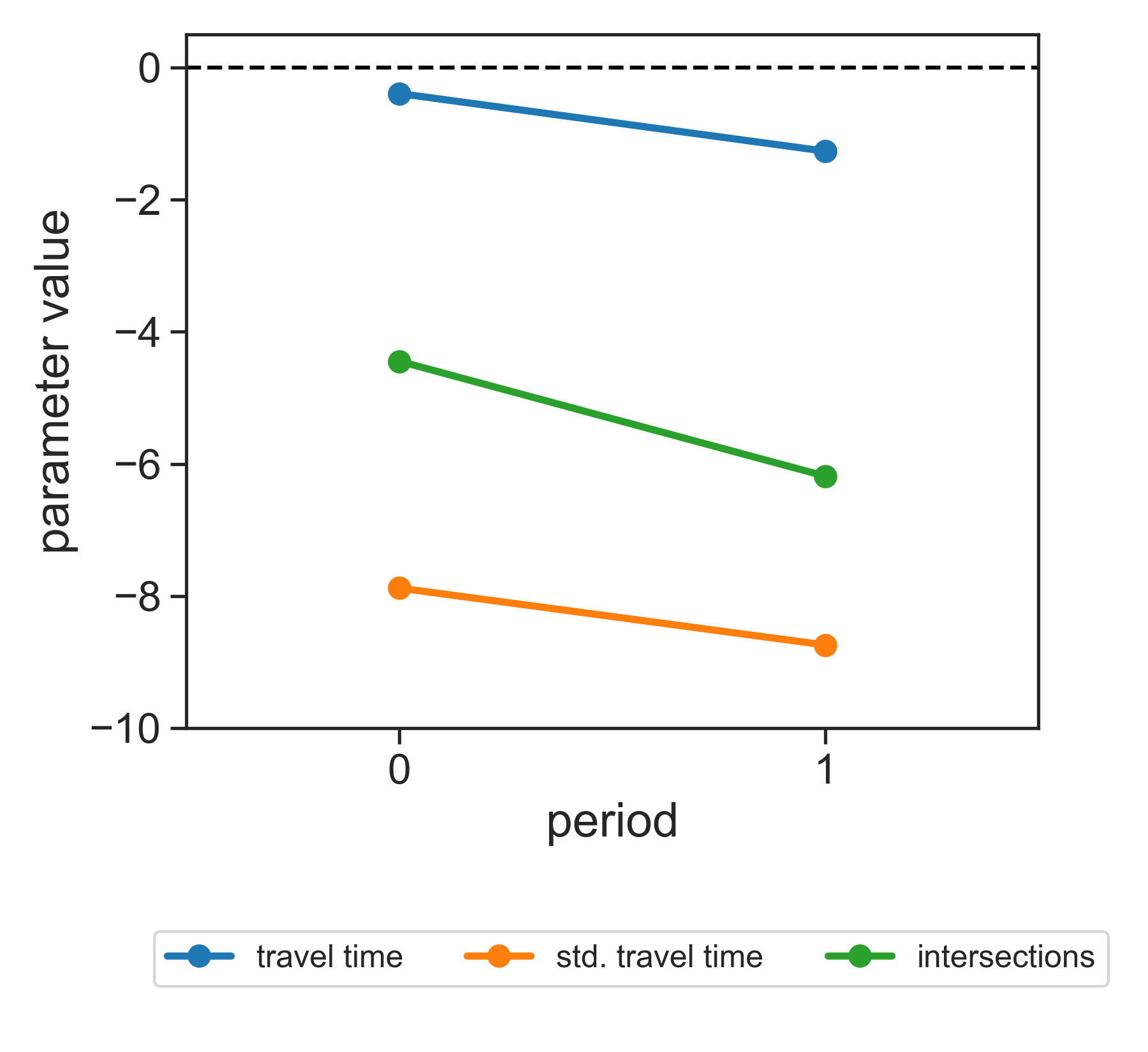}
	\caption{\TVODLULPE}
	\label{subfig:siouxfalls-utility-periods-tvodlulpe}
\end{subfigure}
\begin{subfigure}[t]{0.3\columnwidth}
	\centering
	 \includegraphics[width=\textwidth]{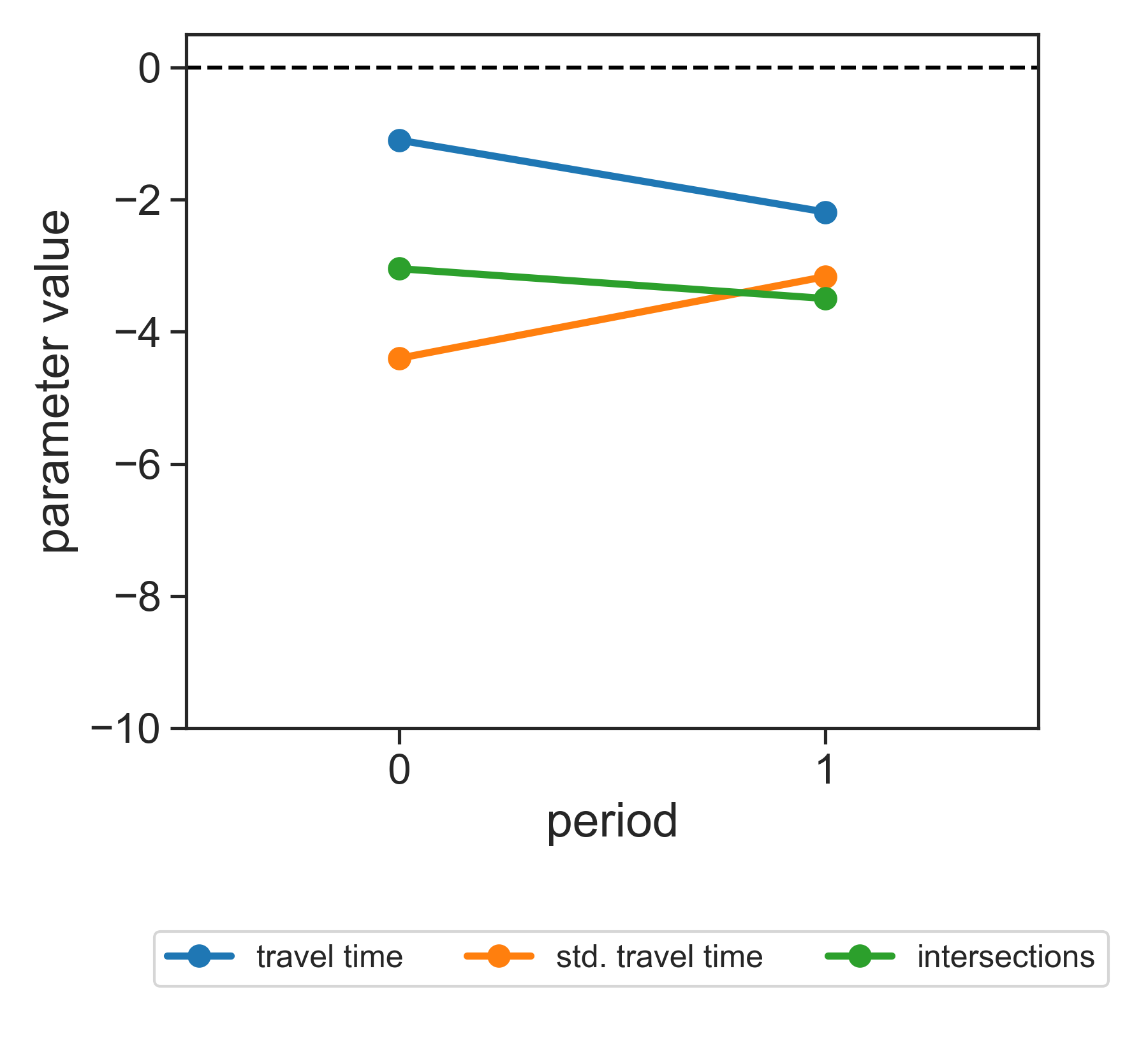}
	\caption{\MTP}
	\label{subfig:siouxfalls-utility-periods-tvhodlulpe}
\end{subfigure}

\caption{Estimates of utility function parameters by period of the day and model using synthetic data from the Sioux Falls network}
\label{fig:siouxfalls-utility-function-estimation}
\end{figure}

\subsubsection{Origin-destination matrices}
\label{sssec:siouxfalls-ode}

The \TVODLULPE model directly learns the O-D matrix parameters. In contrast, \MTP indirectly learns the O-D matrix by estimating the generated trips per location and then distributing these trips over the set of reachable destinations from each location through a logit-based destination choice model. Figure \ref{fig:siouxfalls-comparison-heatmaps-ode} shows heatmaps with the O-D matrices estimated by the \TVODLULPE and \MTP models for each hourly period. Since the O-D parameters of the \TVODLULPE model are initialized with the reference O-D matrix, the estimated and true O-D matrices become similar. This finding is also supported by the difference of the R$^2$ for the relationship between the true and estimated number of trips of both models (Figure \ref{subfig:siouxfalls-scatter-ode}). Despite that the estimation of \MTP cannot incorporate any priors on the O-D matrix, \MTP finds an O-D matrix that induces a solution at equilibrium (Figure \ref{subfig:siouxfalls-relative-gap-mate}) and that reproduces the traffic flow and travel time in training set with high accuracy (Figure \ref{fig:siouxfalls-flow-traveltime-mate}). To drive the O-D solution toward the reference O-D in \MTP, it suffices to incorporate the difference between the estimated and the reference O-Ds in the loss function. 
historic OD matrix.  

\begin{figure}[!htbp]
\centering
\begin{subfigure}[t]{0.65\columnwidth}
    \centering
    \includegraphics[width=\textwidth]{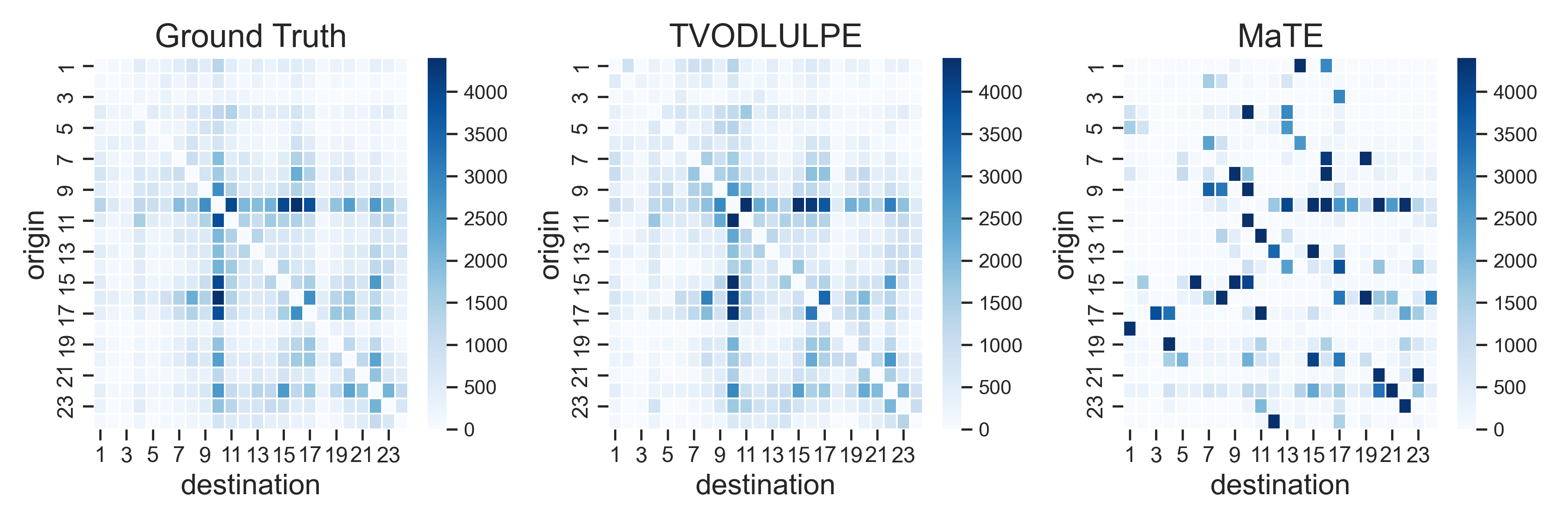}
    \caption{Period 0}
	\label{subfig:siouxfalls-comparison-heatmaps-ode-period-0}
\end{subfigure}
\begin{subfigure}[t]{0.65\columnwidth}
	\centering
    \includegraphics[width=\textwidth]{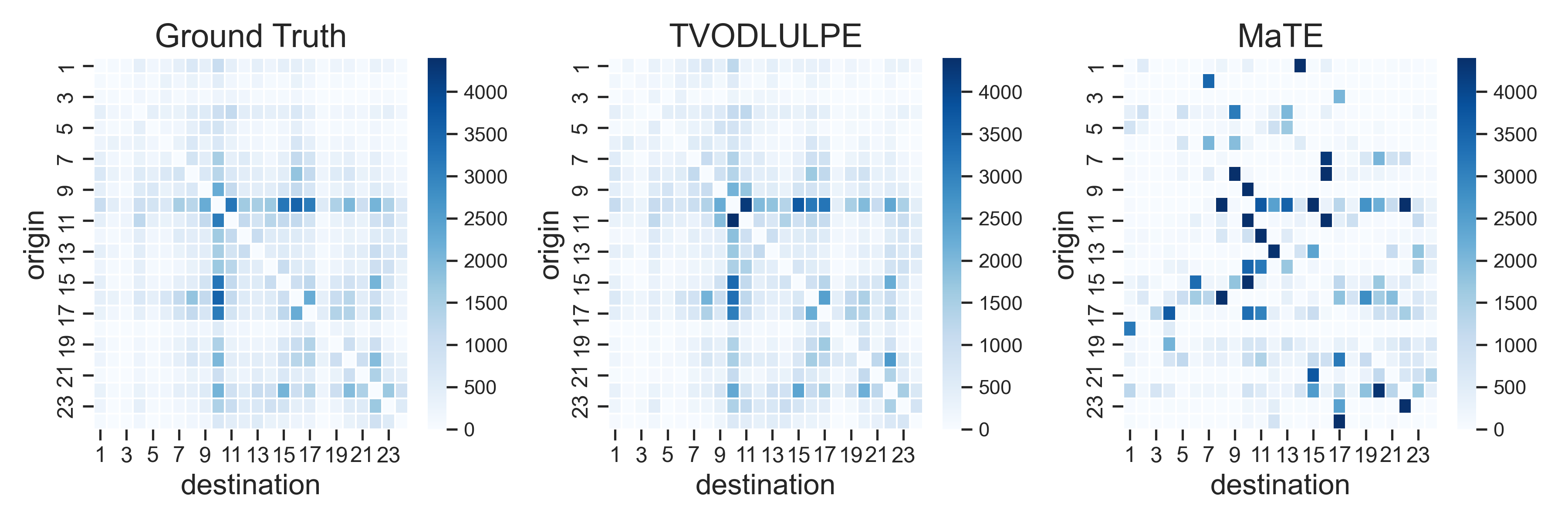}
    \caption{Period 1}
	\label{subfig:siouxfalls-comparison-heatmaps-ode-period-1}
\end{subfigure}
\caption{Comparison of ground truth and estimated O-D matrices by period using synthetic data from the Sioux Falls network}
\label{fig:siouxfalls-comparison-heatmaps-ode}
\end{figure}

\subsubsection{Performance functions}
\label{sssec:siouxfalls-performance-functions}

Figure \ref{fig:siouxfalls-comparison-link-performance-functions} shows the performance functions learned with the \TVODLULPE and \MTP models. On the left side, it shows ground truth values of the performance functions used to generate the synthetic data in the experiments. \MTP learns a polynomial of order 3 with a function form equal to $0.0002x_a + 0.0012x_a^2 + 0.0448x_a^3$, where $x_a$ is the value of the traffic flow in link $a \in \sA$. The average values of the parameters $\alpha$ and $\beta$ of the BPR function learned by the \TVODLULPE model are $0.44$ and $5.1$. Results suggest that the \TVODLULPE and \MTP models tend to underestimate the maximum value of the range of the performance functions. However, both models can reproduce the exponential shape of the ground truth performance functions. Notably, the \MTP and \TVODLULPE models achieve similar performance in reproducing travel times (Figures \ref{fig:siouxfalls-flow-traveltime-tvodlulpe} and \ref{fig:siouxfalls-flow-traveltime-mate}, Appendix \ref{appendix:ssec:siouxfalls-model-training}). 

\begin{figure}[H]
	\centering
    \begin{subfigure}[t]{0.3\columnwidth}
    \centering
    \includegraphics[width=\textwidth, trim= {11.5cm 0cm 0cm 1.1cm}, clip]{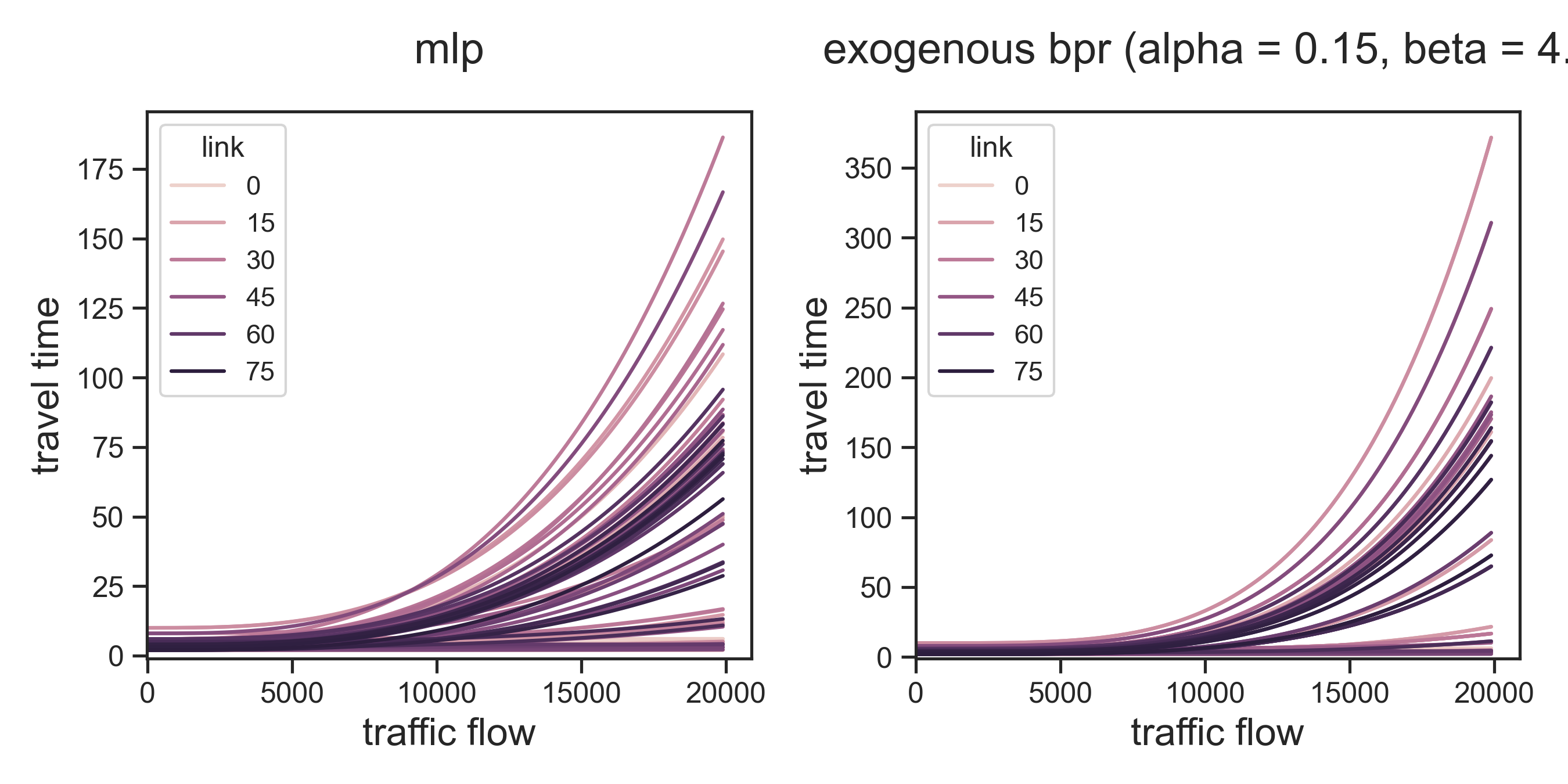}
    \caption{Ground truth}
	\label{subfig:siouxfalls-performance-functions-groundtruth}
\end{subfigure}
\begin{subfigure}[t]{0.3\columnwidth}
	\centering
 \includegraphics[width=\textwidth, trim= {0cm 0cm 11.5cm 1.1cm}, clip]{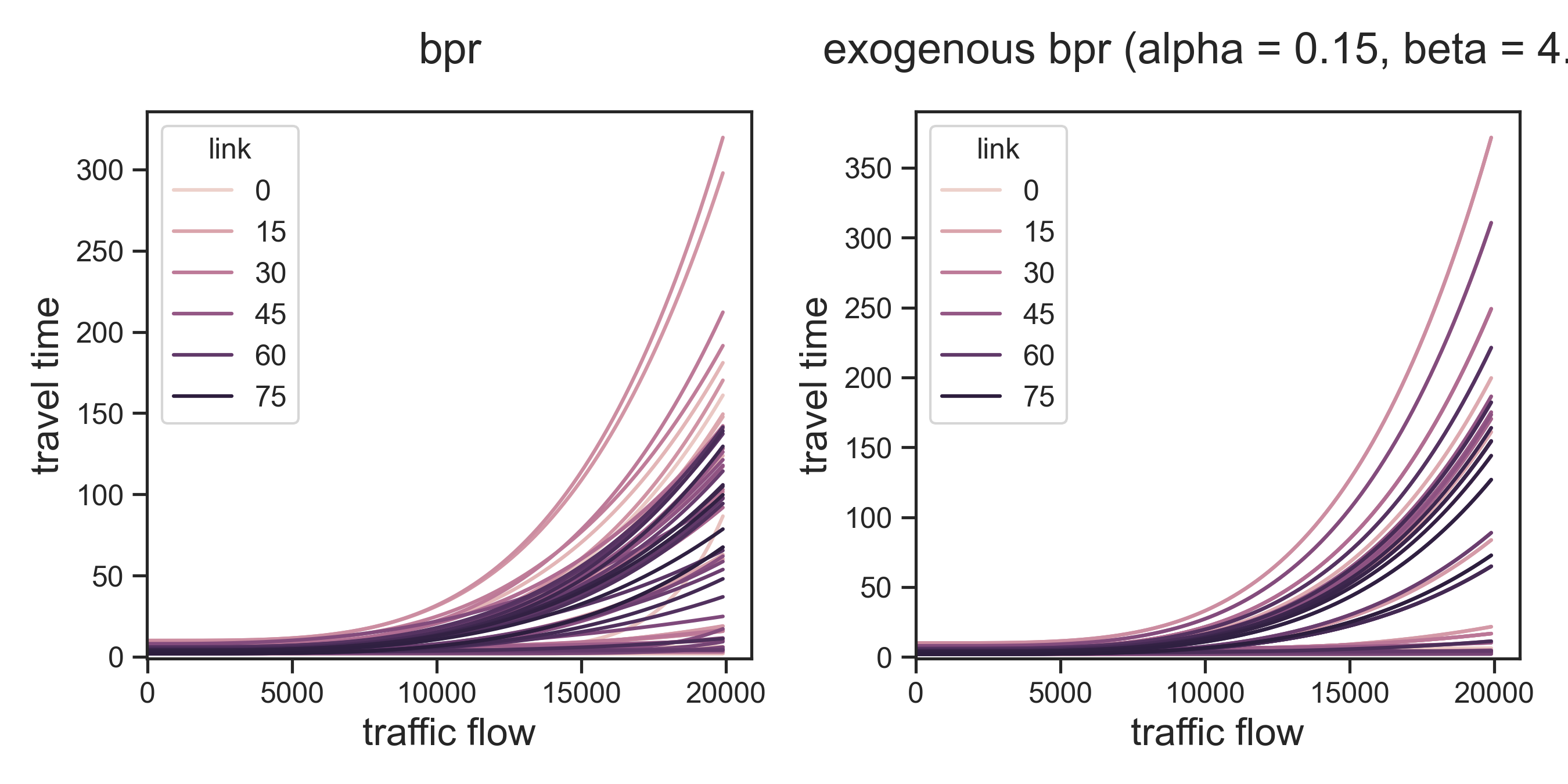}
    \caption{\TVODLULPE}
	\label{subfig:siouxfalls-performance-functions-tvodlulpe}
    \end{subfigure}
\begin{subfigure}[t]{0.3\columnwidth}
	\centering
	\includegraphics[width=\textwidth, trim= {0cm 0cm 11.5cm 1.1cm}, clip]{figures/experiments/siouxfalls-comparison-all-link-performance-functions-mate.png}
    \caption{\MTP}
	\label{subfig:siouxfalls-performance-functions-mate}
\end{subfigure}
    \caption{Comparison of ground truth and learned performance functions using synthetic data from the Sioux Falls network}
	\label{fig:siouxfalls-comparison-link-performance-functions}
\end{figure}

Figure \ref{fig:siouxfalls-kernel-link-performance-functions-mate} shows on the right side the kernel matrix learned by \MTP, which, consistent with the experimental setting, is diagonal primarily. The average values of the diagonal and non-diagonal parameters of the kernel matrix $\mW$ learned by \MTP are $4.686$ and $0.005$, respectively. 

\begin{figure}[H]
	\centering
	\includegraphics[width=0.7\textwidth]{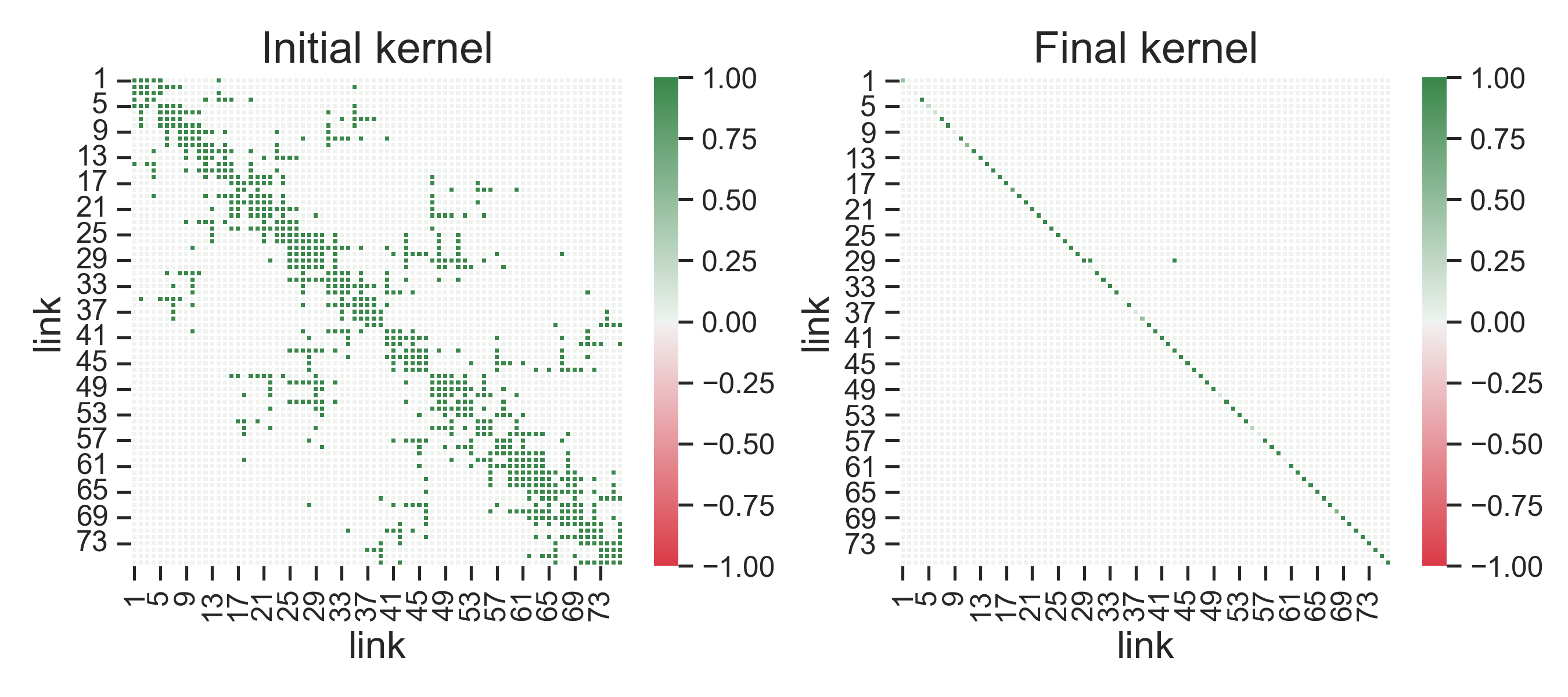}
	\caption{Kernel of traffic flow interactions learned by \MTP and using synthetic data from the Sioux Falls network}
	\label{fig:siouxfalls-kernel-link-performance-functions-mate}
\end{figure}

\subsection{Out-of-sample performance}
\label{siouxfalls-outofsample-performance}

Suppose there are measurements of traffic flow and travel times available in a subset of links in the network to train a model. The goal is to evaluate the \textit{out-of-sample performance} of this model, that is, the model's accuracy in estimating traffic flow and travel times in the remaining links in the network. To our knowledge, this analysis is rarely conducted in transportation network modeling studies, and we consider it to be key to assessing if a model does not overfit the training data. Typically, studies analyze the \textit{in-sample performance} of the model, that is, the estimation accuracy in links that report observations of traffic flow and travel time in the training set. The design of our validation strategy follows the standard cross-validation techniques used in machine learning studies. First, the links are randomly split into $K$ folds, among which $K-1$ are used to train the model, and one is used as a validation set. Then, the model is trained and validated with the $K$ combinations of folds. The estimation errors obtained in the training and validation folds are used to study the in-sample and out-of-sample performance of the model, respectively. We choose $K=5$ folds and the Mean Absolute Percentage Error (MAPE) as the metric to measure estimation performance in the validation folds. Furthermore, to analyze the impact of the different components of \MTP in estimation performance, we perform an ablation study with respect to the class of performance functions and the incorporation of the generation step. 

\subsubsection{Impact of choice of performance function}
Figure \ref{fig:siouxfalls-kfold-link-performance-function} shows the impact of the choice of performance function on the out-of-sample performance. The only difference between the model specifications is the choice of performance function used to map traffic flow into travel time. To isolate the effect of the class of performance function on estimation performance, we set the value of all the model parameters to their ground truth values, except for the link flow parameters. The top plots show the MAPE at the start and end of model training in both the training and the validation sets. The bottom plots compare the MAPE of the models against a benchmark model that uses the historical mean of travel time and traffic flow in the training set to make estimations in the validation set. Note that both models outperform the benchmark estimations by a wide margin in terms of in-sample and out-of-sample performance. However, using BPR functions with link-specific parameters significantly increases the MAPE for travel times in the out-of-sample links, suggesting overfitting. One way to mitigate this problem is not using link-specific parameters for the BPR function. However, the drawback is that this may increase the risk of underfitting when these parameters are not common across links. Overall, our results suggest that the use of \textit{neural performance functions} with (i) an MLP layer to capture traffic flow interactions and (ii) polynomial features to model the relationship between traffic flow and travel times improves performance in out-of-sample travel time estimation tasks. 

\begin{figure}[H]
\centering
\begin{subfigure}[t]{0.365\columnwidth}
    \centering
 \includegraphics[width=\textwidth, trim= {0cm 0cm 5.5cm 0cm}, clip]{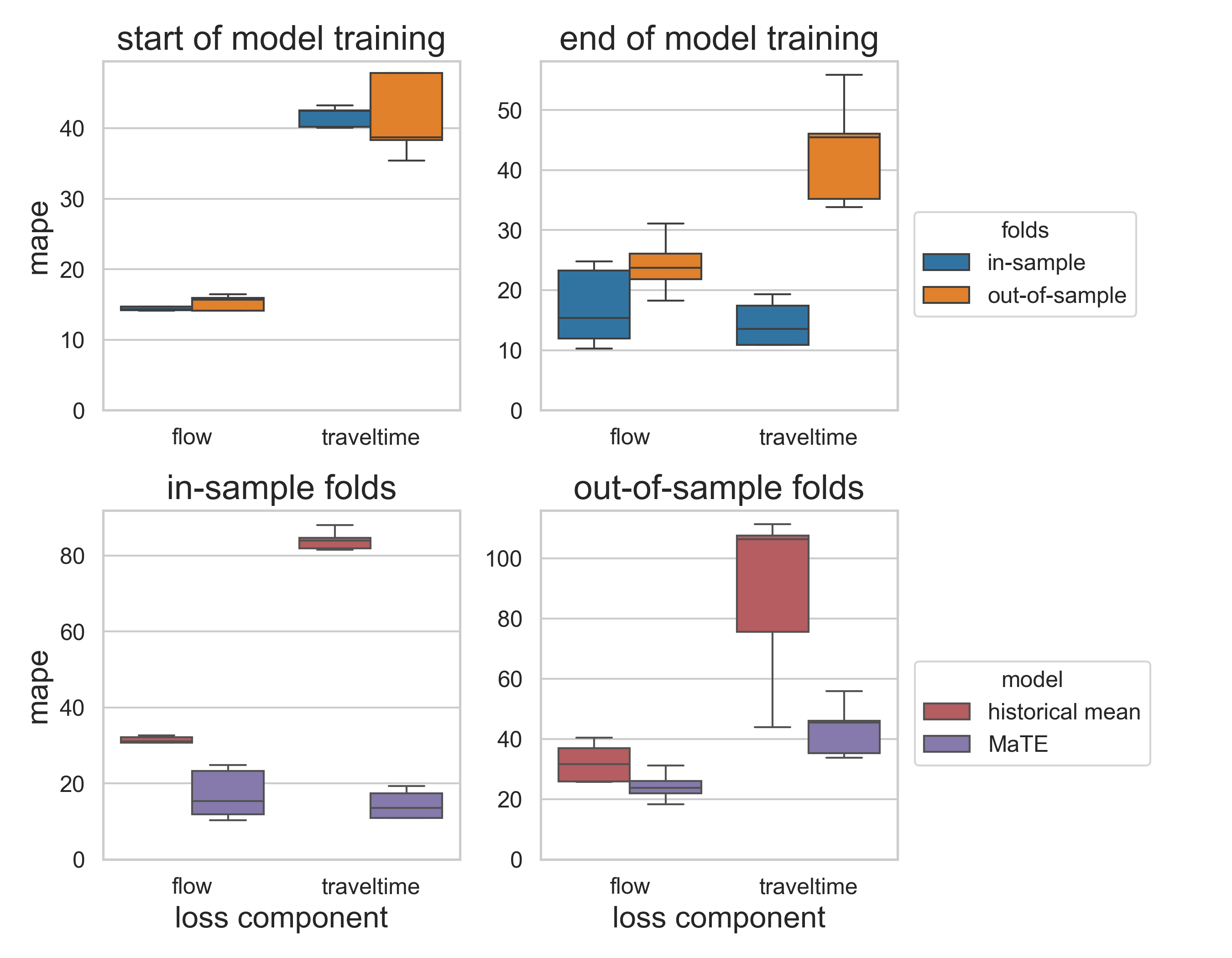}
	\caption{BPR with link-specific parameters}
	\label{subfig:siouxfalls-kfold-mape-tvlpe-bpr}
\end{subfigure}
\begin{subfigure}[t]{0.48\columnwidth}
	\centering
\includegraphics[width=\textwidth]{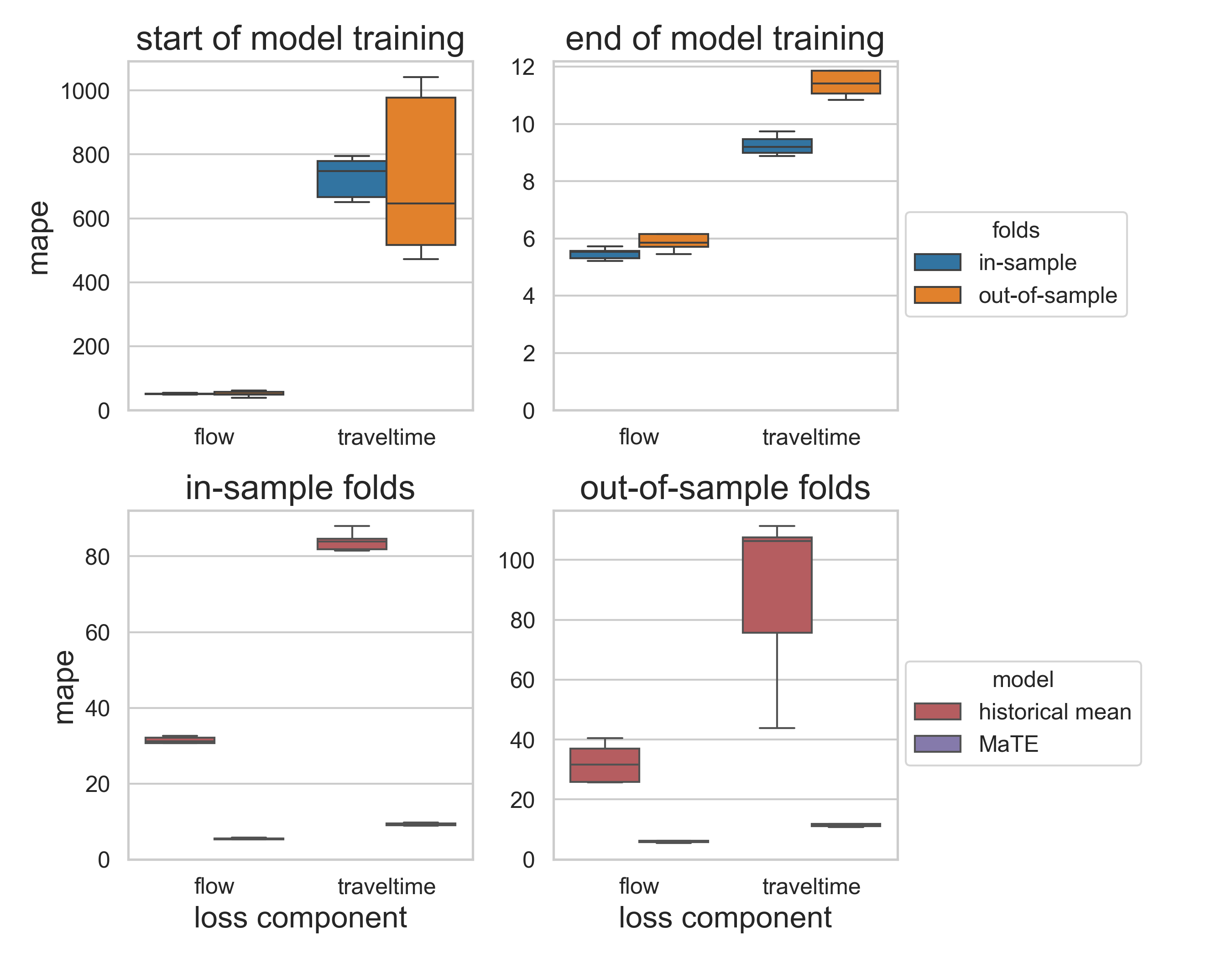}
	\caption{MLP + Polynomial of Degree 3}
	\label{subfig:siouxfalls-kfold-mape-tvlpe-mlp}
\end{subfigure}
\caption{Impact of choice of performance function in out-of-sample performance using synthetic data from the Sioux Falls network}
\label{fig:siouxfalls-kfold-link-performance-function}
\end{figure}

\subsubsection{Impact of incorporating the generation step}

Figure \ref{subfig:siouxfalls-kfold-mape-mate} shows the impact of incorporating the generation step on estimation performance. Results shown in Figures \ref{subfig:siouxfalls-kfold-mape-tvodlulpe-refod-mlp} and \ref{subfig:siouxfalls-kfold-mape-tvodlulpe-mlp} are obtained with the traditional \ODE method implemented in the \TVODLULPE model. To make a fair comparison between the results of the \MTP and \TVODLULPE models shown in Figures \ref{subfig:siouxfalls-kfold-mape-tvodlulpe-refod-mlp} and \ref{subfig:siouxfalls-kfold-mape-mate}, respectively, it is assumed that there is only access to a reference of the trip generation. In addition, to isolate the effect of the generation step in estimation performance, the two models use \textit{neural performance functions}. Compared to the \TVODLULPE model, results show that incorporating the generation step does not affect the median MAPE obtained in the out-of-sample links. Furthermore, access to a reliable O-D matrix reduces the estimation errors of traffic flow and travel time by around 10\% and 20\%, respectively. Therefore, without access to a reliable reference O-D matrix, the models have a higher risk of overfitting the set of traffic flow and travel time measurements in the in-sample links. While this does not significantly affect in-sample performance (Figure \ref{fig:siouxfalls-flow-traveltime-mate}, Appendix \ref{appendix:ssec:siouxfalls-model-training}), it negatively impacts estimation performance in out-of-sample links. 

\begin{figure}[H]
\centering
\begin{subfigure}[t]{0.295\columnwidth}
    \centering
 \includegraphics[width=\textwidth, trim= {0cm 0cm 5.5cm 0cm}, clip]{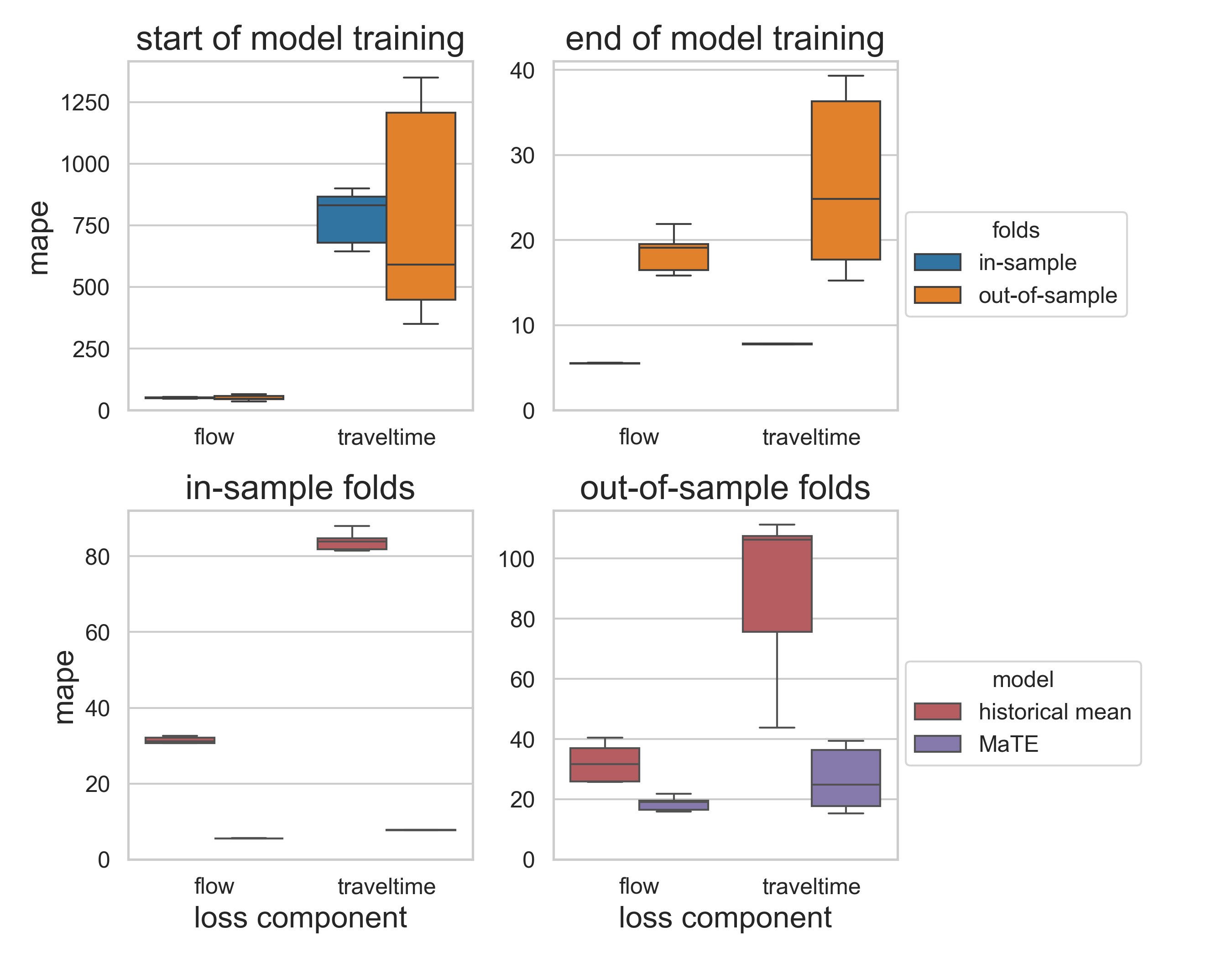}
    \caption{With reference O-D} 
	\label{subfig:siouxfalls-kfold-mape-tvodlulpe-refod-mlp}
\end{subfigure}
\begin{subfigure}[t]{0.295\columnwidth}
    \centering
 \includegraphics[width=\textwidth, trim= {0cm 0cm 5.5cm 0cm}, clip]{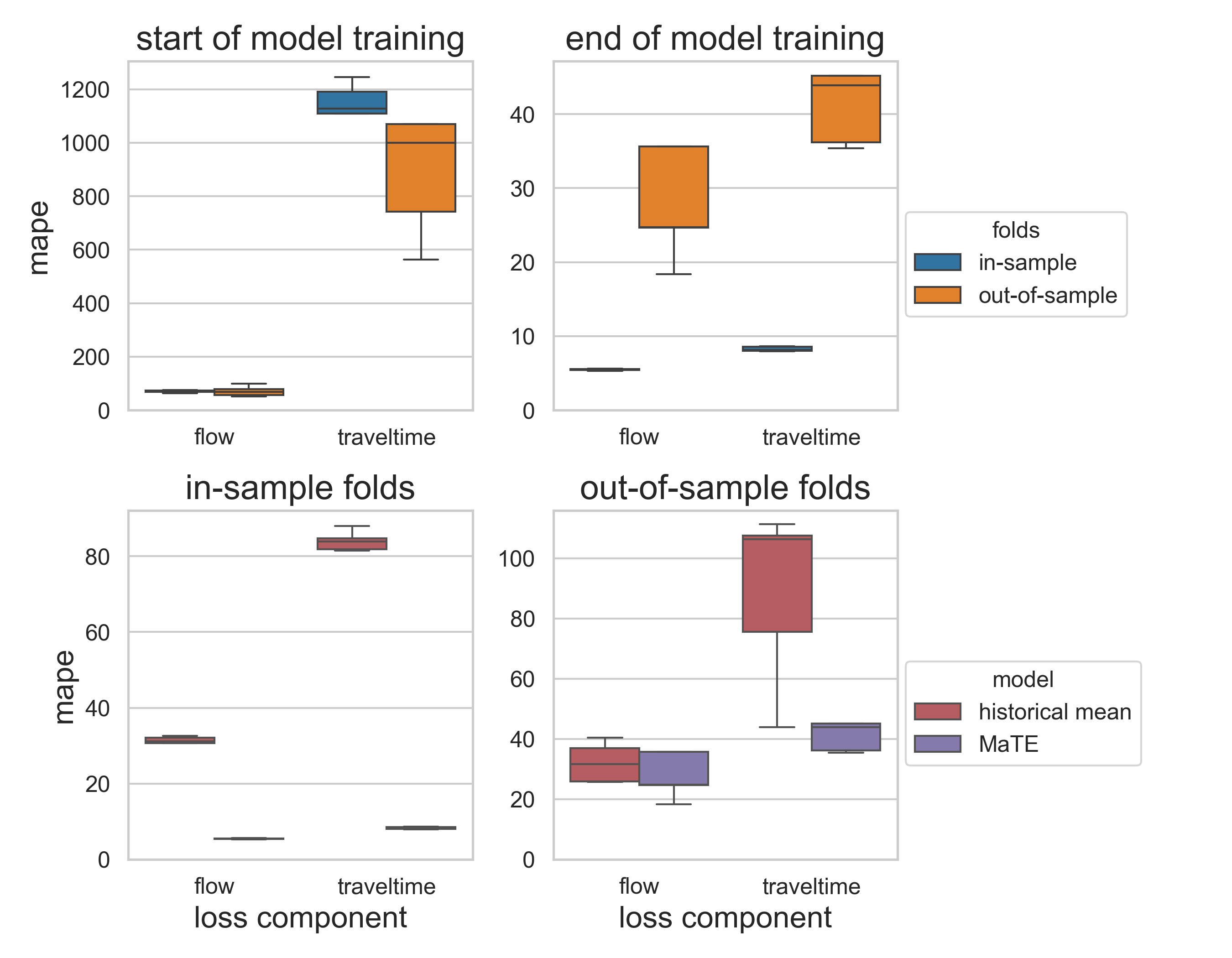}
    \caption{Without reference O-D}
	\label{subfig:siouxfalls-kfold-mape-tvodlulpe-mlp}
\end{subfigure}
\begin{subfigure}[t]{0.38\columnwidth}
	\centering
\includegraphics[width=\textwidth]{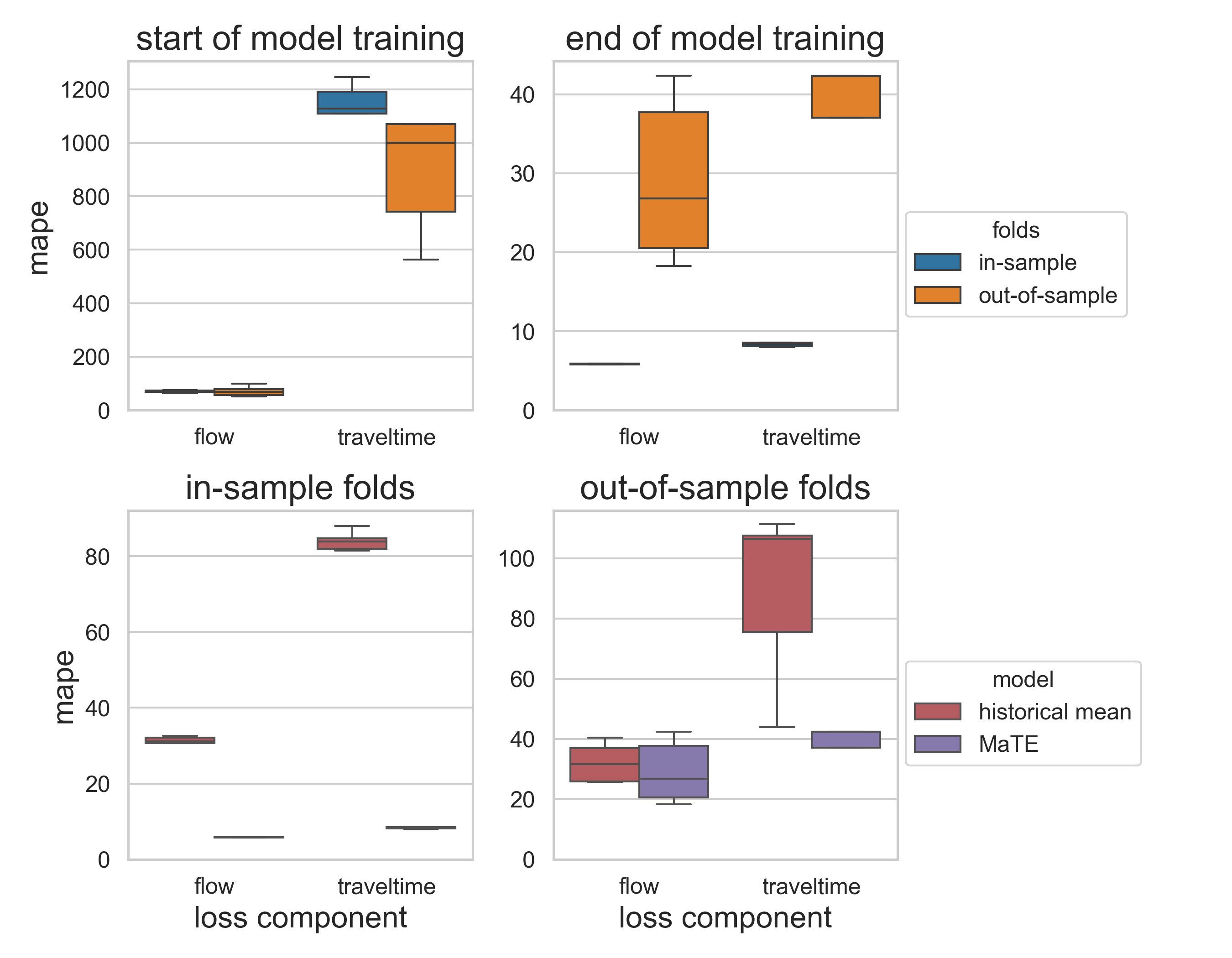}
    \caption{With generation step}
	\label{subfig:siouxfalls-kfold-mape-mate}
\end{subfigure}
\caption{Impact of incorporating generation step on out-of-sample performance}
\label{fig:siouxfalls-kfold-gode-versus-ode}
\end{figure}

%% file: sections/results.tex
\section{Large scale implementation}
\label{sec:large-scale-implementation}

Our methodology is applied to a large-scale transportation network located in the City of Fresno, California (Figure \ref{subfig:fresno-map}, Appendix \ref{appendix:sec:networks}). This network has 1,789 nodes and 2,413 links, and it primarily covers major roads and highways around the SR-41 corridor \citep{guarda_statistical_2024}. The models are trained using travel time and link flow hourly measurements collected during Tuesdays, Wednesdays, and Thursdays of October 2019 between 6:00 AM and 9:00 PM. The same measurements were collected in October 2020 for testing purposes. Each period of the model is defined as each hour of the day, and the periods are assumed to be common between different days of the week, which gives a total of 15 periods. A historic O-D matrix with 6,970 O-D pairs representative of a typical weekday at 4:00 PM is used to compute a reference value of the trip generation. The models are trained with the three shortest paths between every O-D pair. The data used for model training is described in Section \ref{ssec:fresno-eda}, Appendix \ref{appendix:sec:model-training}.

\subsection{Model specifications}
\label{ssec:fresno-model-specifications}

Table \ref{table:models-specifications-fresno} describes the number of parameters associated with different layers of \MTP. For comparative purposes, it also reports the parameters of the \TVODLULPE model for the same dataset and network instance. Although the number of parameters is large, every parameter of \MTP is interpretable. For example, each O-D estimation parameter corresponds to the O-D specific utility associated with a particular O-D pair and given hour of the day. Because the analysis with synthetic data proved that \MTP outperforms the \TVODLULPE (Section \ref{sec:numerical-experiments}), this section only reports estimation results for \MTP. The period-specific parameters of the models take 15 different values because the data is available for 15 different hours or periods of the day. The travelers' utility function is assumed linear-in-parameters and dependent on travel time, a link-specific effect, and five exogenous attributes weighted by time-specific parameters: the standard deviation of travel time, the number of bus stops, the number of street intersections, the number of yearly incidents, and the average monthly household income [1,000 USD per month] of the census block. The function that estimates the number of generated trips in a location in \MTP is assumed linear-in-parameters and dependent on a location-specific effect and on three feature-specific attributes that are weighted by period-specific parameters: the total population, the number of bus stops, and the level of income of the census block. 

\begin{table}[H]
\begin{adjustbox}{width=\textwidth}  
\renewcommand{\arraystretch}{1.05}
\centering
\begin{threeparttable}
\caption{Parameters of \MTP trained with data from Fresno, CA}
\label{table:models-specifications-fresno}
\begin{tabular}{ccccccc}
\hline
\multirow{2}{*}[-0.5cm]{Model} & \multicolumn{6}{c}{Parameters} \\ \cline{2-7} 
 & \begin{tabular}[c]{@{}c@{}}Link\\ flows\end{tabular} & \begin{tabular}[c]{@{}c@{}}Utility \\ function \end{tabular} & \begin{tabular}[c]{@{}c@{}} Generation \\ function \end{tabular} & \begin{tabular}[c]{@{}c@{}}O-D\\ estimation\end{tabular}  & \begin{tabular}[c]{@{}c@{}}Performance \\ function \end{tabular} & Total \\ \hline
\MTP & 2,413$T$ & 2,413 + 6$T$ & 1,789T + 3 & 6,970$T$ & $b + \|\mW\|_0 = 4 + 4,727$ & 7,147 + 11,178$T$ = 174,817 \\
\TVODLULPE & 2,413$T$ & 2,413 + 6$T$ & $-$ & 6,970$T$ & $2\times 2,413$ 
& 7,239 + 9,389$T$ = 148,074
\\ \hline
\end{tabular}
\begin{tablenotes}
      \footnotesize
        \item Note: $0 \leq \|\mW\|_0 \leq |A|^2$ corresponds to the number of cells in the flow interaction matrix that are different than zero. 
        \item $T$: number of distinct hourly periods in the model. 
        \item $b$: degree of the polynomial used to represent the link performance function. 
    \end{tablenotes}
\end{threeparttable}
\end{adjustbox}
\end{table}

\subsection{Estimation procedure}
\label{ssec:estimation-procedure-fresno} 

\MTP is trained with the Adam optimizer \citep{Kingma2015}, batches of size one and during 60 epochs. A learning rate of 0.05 is set to perform the gradient updates in all parameters except those associated with the fixed effects of the trip generation function. Increasing the learning rate to 10 for these parameters encourages larger changes in the parameter values over epochs. To represent a real-world setting where prior information can be leveraged for model training, we initialize the feature-specific parameters of the utility function of the \TVGODLULPE model with the estimates obtained for the \TVODLULPE model by \citet{guarda_estimating_2024}. The domain of the utility function parameters during the optimization is constrained such that the parameters have signs consistent with our prior expectations. The feature-specific parameters of the generation function are initialized with the estimates obtained from regressing the generated trips by location on the total population, the number of bus stops, and the level of income of the census block. The generated trips by location are obtained by aggregating the rows of the reference O-D matrix. 

The regression results show that all pre-trained parameters have signs consistent with our expectations. Only the parameter associated with the population is significant at a 1\% confidence level. The parameter weighting the number of bus stops by census block is negative, which is reasonable considering that the reference O-D matrix is for car trips, and these are negatively correlated with the density of bus stops. At training time, the feature-specific parameters are fixed to their pre-trained values. Therefore, the feature-specific parameters are not identifiable because the model learns location-specific effects that are period-specific. It is easy to check that the location-specific effects can reproduce any value taken by the weight between a feature and its parameter. To speed up convergence and increase training stability, the location-specific effects are initialized with the reference vector of generated trips by location. 

\subsection{In-sample performance}
\label{ssec:fresno-insample-performance}

This section studies the in-sample performance of \MTP. The training set in October 2019 compresses 31,624 and 468,303 measurements of traffic flow and travel times, equating to sensor coverages of 5.8\% and 86.3\%, respectively. The goodness of fit of the model is assessed with a variant of the MAPE known as the MeDian Absolute Percentage Error (MDAPE). By computing the median instead of the mean of the absolute percentage errors, the MDAPE becomes more robust to outliers that may be present in real-world data. 

\subsubsection{Convergence and goodness of fit}
\label{appendix:ssec:fresno-converegence}

Table \ref{table:fresno-gof-in-sample} reports the performance metrics in the final epoch. Model training takes about 174 minutes and the relative gap in the last epoch is 0.047. Figure \ref{fig:fresno-convergence-mate} shows the change in relative MSE and MDAPE over epochs in the training set. We observe that the error decreases smoothly over epochs for each component of the loss function. Figure \ref{fig:fresno-flow-traveltime-insample} shows scatter plots comparing the observed and estimated values of travel time and traffic flow. The bottom plots show $R^2$ equal to 0.77 and 0.96 for the traffic flow and travel time measurements, respectively, which suggest that the model can reproduce the observations of the training set. 

\begin{figure}[H]
\centering
\begin{subfigure}[t]{0.3\columnwidth}
    \centering
	\includegraphics[width=\textwidth]{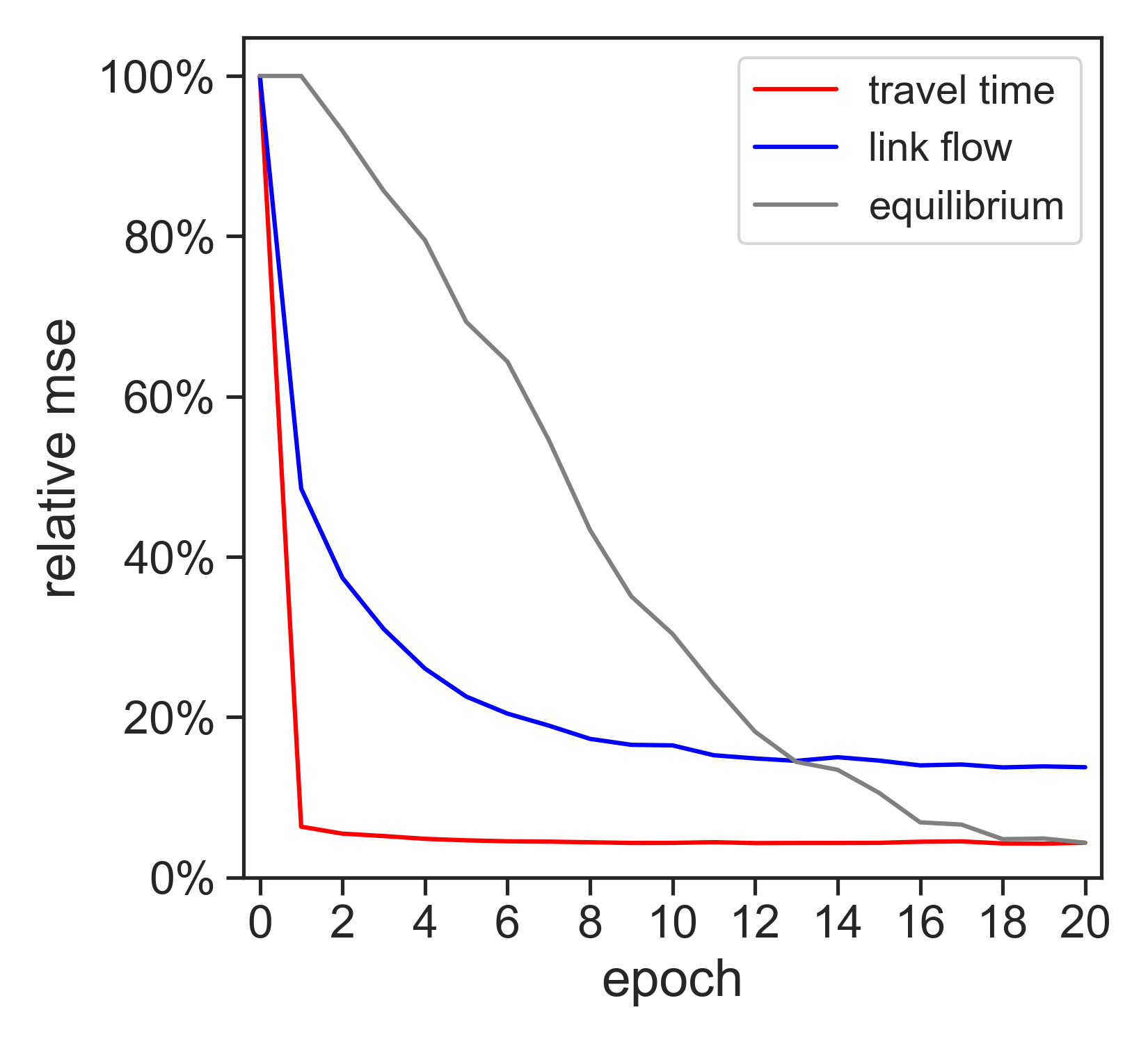}
    \caption{Relative MSE}
	\label{subfig:fresno-relative-mse-mate}
\end{subfigure}
\begin{subfigure}[t]{0.3\columnwidth}
	\centering
    \includegraphics[width=\textwidth]{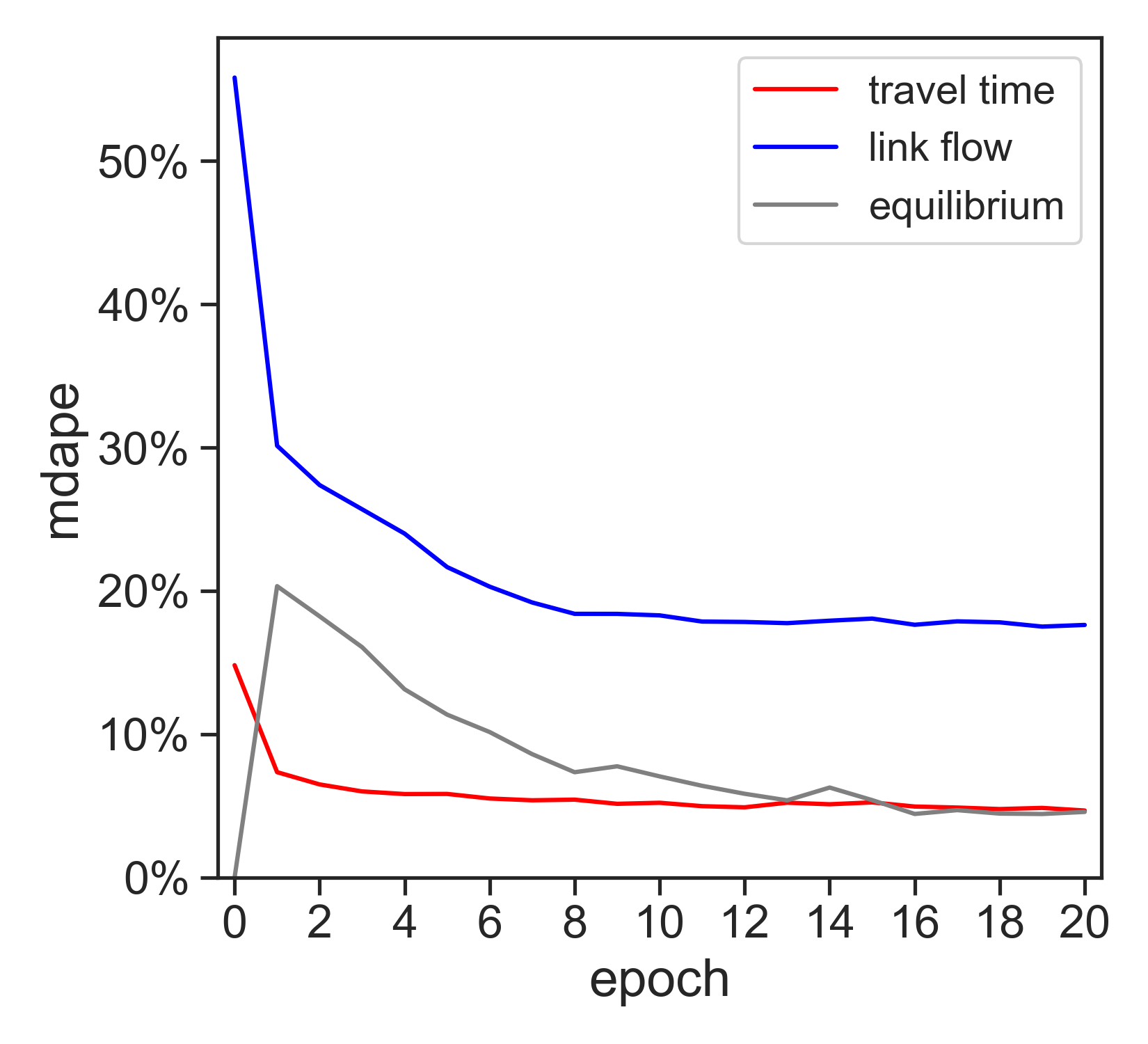}
    \caption{MDAPE}
	\label{subfig:fresno-mdape-mate}
\end{subfigure}
\caption{Convergence of \TVGODLULPE model in the Fresno network}
\label{fig:fresno-convergence-mate}

\end{figure}

\subsubsection{Parameter estimates}

\MTP is trained with machine learning algorithms and formulated according to transportation network theory. Thus, the model parameters are interpretable and useful to validate the properties of the model solutions. In particular, inconsistencies in the signs and relative magnitudes of the parameters of either the utility or the trip generation function can anticipate a bad estimation performance in the what-if-scenario analysis, even if the model can reproduce the training data. Figure \ref{fig:fresno-utility-generation-periods} shows the parameters of the trip generation and utility functions estimated by hour. As expected, the travel time, the standard deviation of travel time, the number of yearly incidents, and the number of street intersections have a negative effect on the utility of choosing a route in every hourly period. In contrast, the parameters associated with the number of bus stops and the income level have signs inconsistent with our prior expectation; thus, their values are projected to zero at training time. Interestingly, the parameters weighting the same features in the trip generation function have the expected sign (Section \ref{ssec:estimation-procedure-fresno}), which illustrates the importance of leveraging existing data in the appropriate layers of the model. 

\begin{figure}[H]
\centering
\begin{subfigure}[t]{0.3\columnwidth}
	\centering
	 \includegraphics[width=\textwidth]{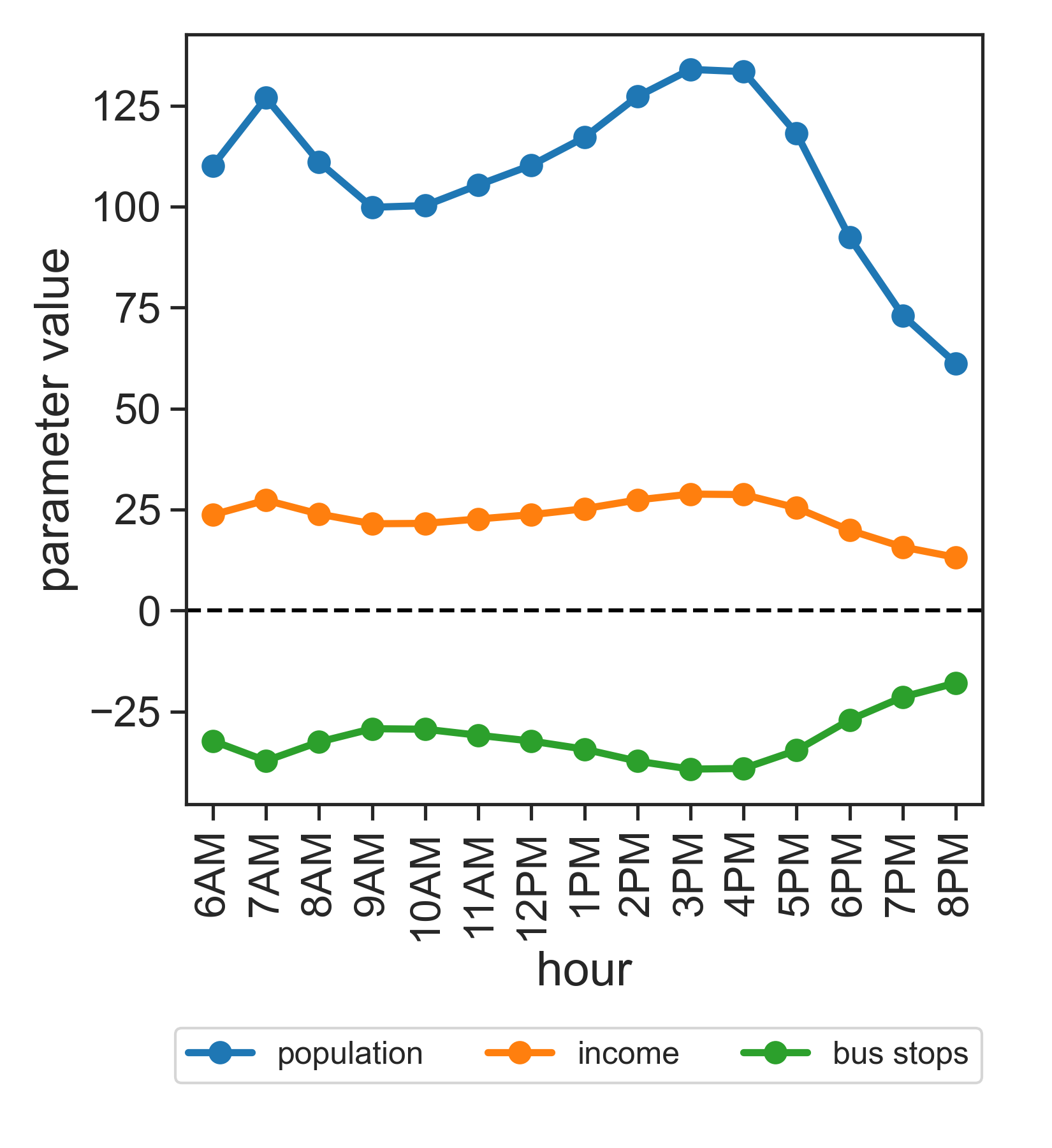}
	\caption{Generation}
	\label{subfig:fresno-utility-periods-mate}
\end{subfigure}
\begin{subfigure}[t]{0.3\columnwidth}
    \centering
 \includegraphics[width=\textwidth]{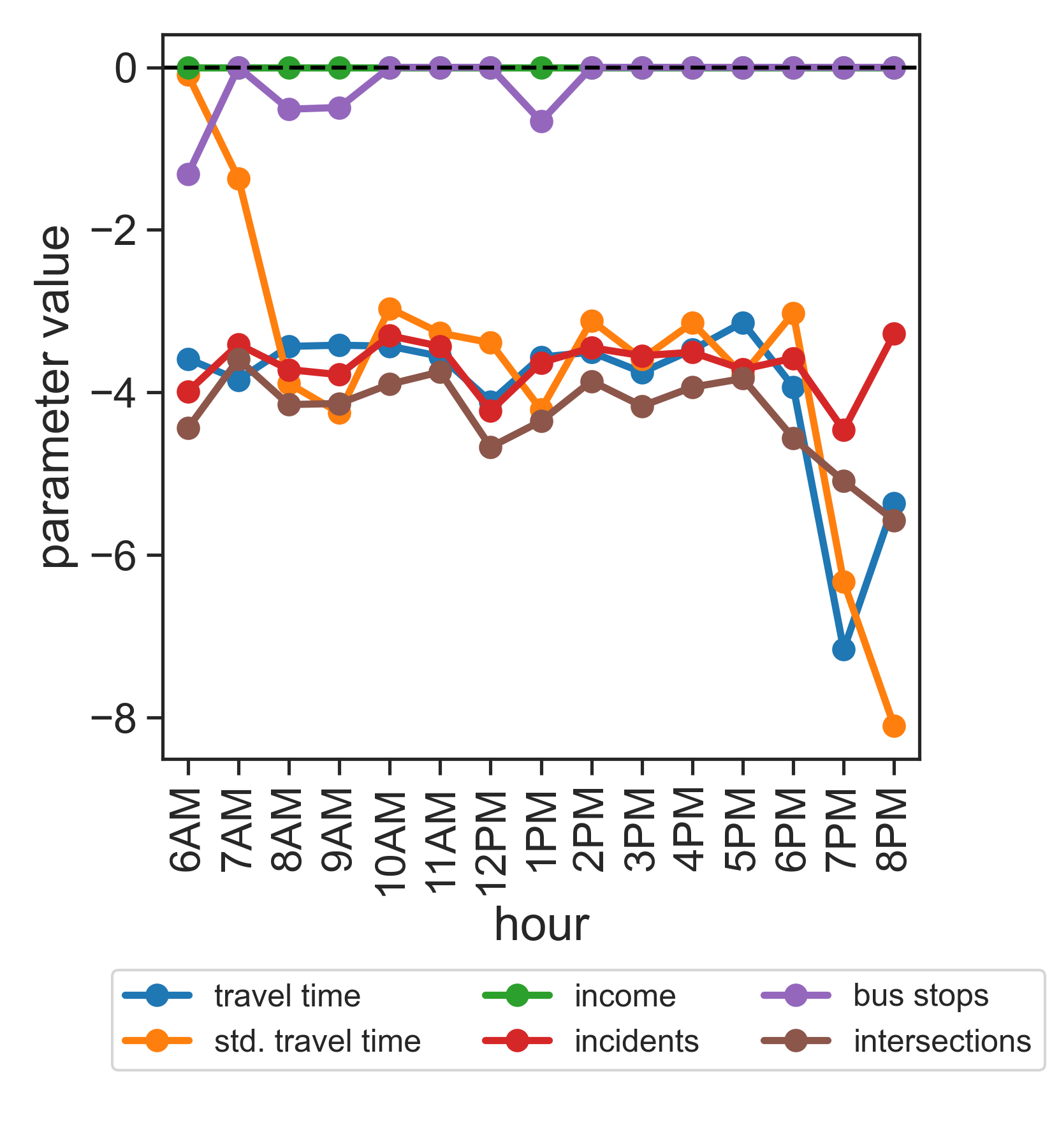}
	\caption{Utility}
	\label{subfig:fresno-generation-periods-mate}
\end{subfigure}

\caption{Parameters of utility and generation function by hour using data from Fresno, CA}
\label{fig:fresno-utility-generation-periods}
\end{figure}

Figure \ref{fig:fresno-estimates-by-hour-of-day} shows the change in the reliability ratios and the total number of trips by hour. The reliability ratios are calculated as the ratio of the utility function parameters associated with the standard deviation and the average of the travel time feature. The total trips are computed as the sum of the cells of the O-D matrices estimated by hour. As expected, the reliability ratio is positive every hour, except at 6:00 AM, which suggests that travelers consistently prefer to reduce travel time variability. The variation of the hourly trips resembles the pattern of hourly traffic flow per lane observed on the left side of Figure \ref{fig:fresno-hourly-eda}, Appendix \ref{ssec:fresno-eda}. For example, the estimated trips are higher in the afternoon peak hours, followed by a large drop in the number of trips in the evening hours. The total estimated number of trips is 1,178,180, which is reasonable for the city of Fresno, which has 175,000 households \citep{us_census_quickfacts_2022}. 

\begin{figure}[H]
\centering

\begin{subfigure}[t]{0.3\columnwidth}
	\centering
	\includegraphics[width=\textwidth]{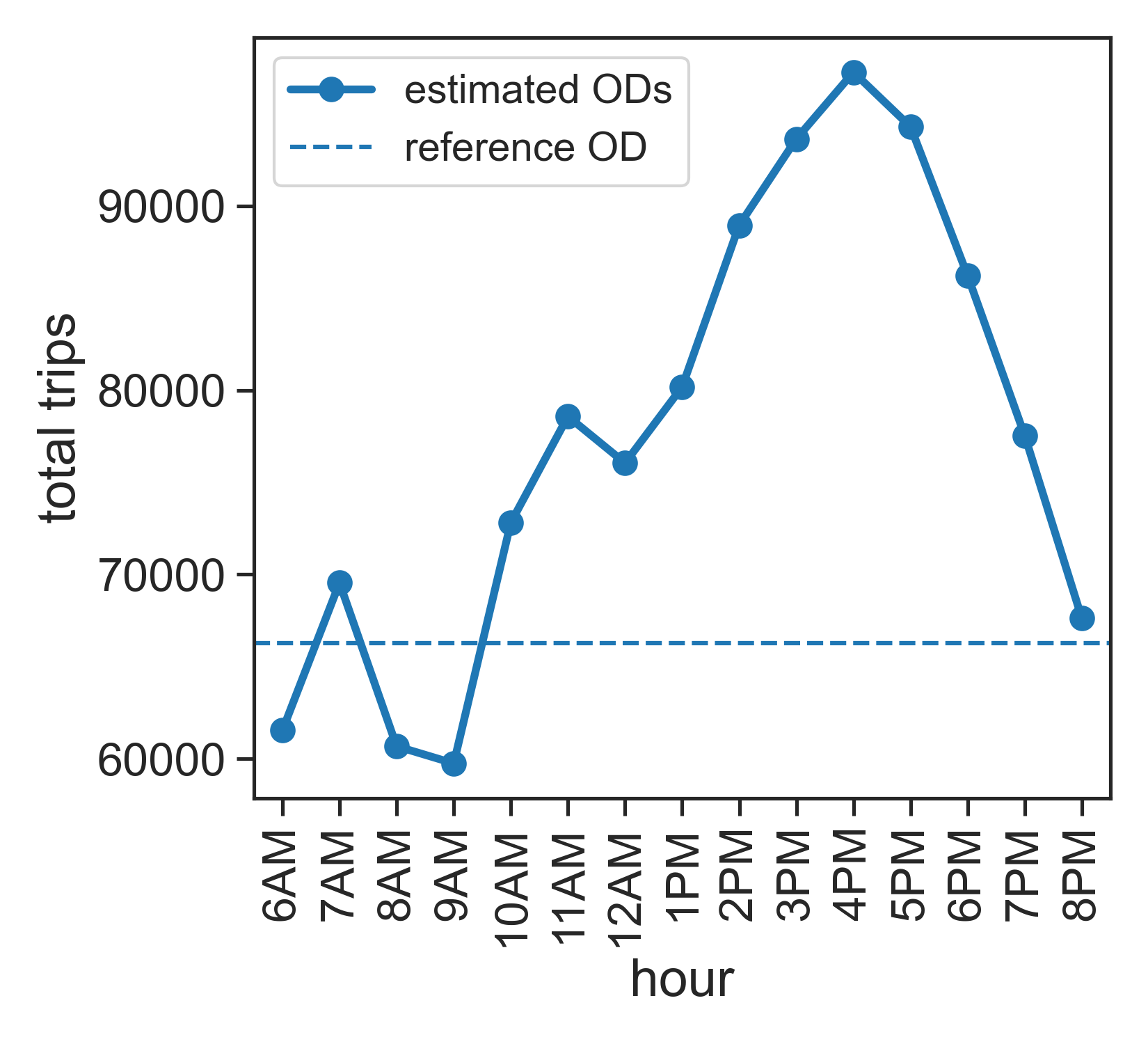}
	\caption{Total trips}
	\label{fig:fresno-total-trips-periods}
\end{subfigure}
\begin{subfigure}[t]{0.3\columnwidth}
    \centering
	\includegraphics[width=\textwidth]{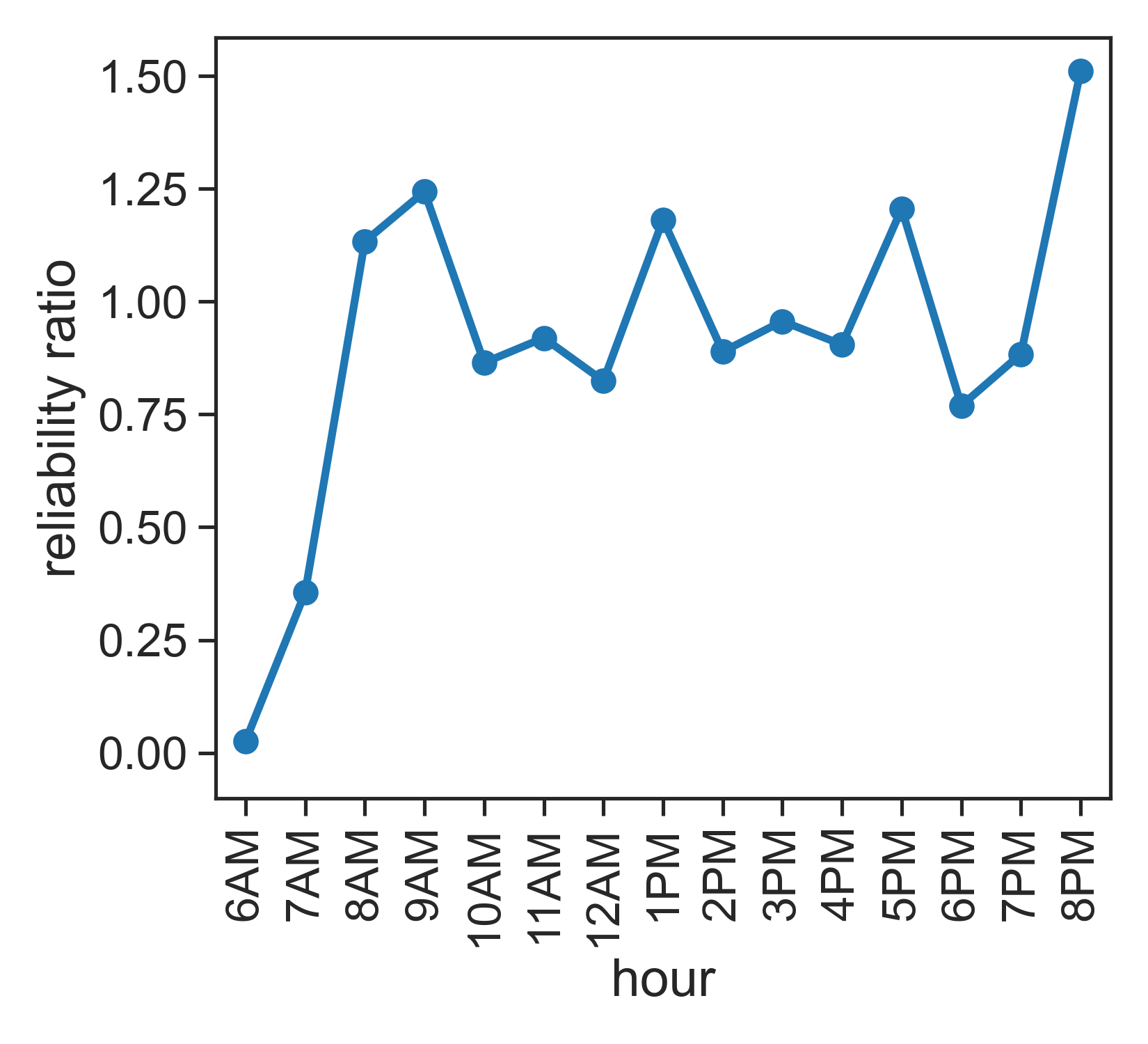}
	\caption{Reliability ratio}
	\label{fig:fresno-reliability-ratios-periods}
\end{subfigure}

\caption{Reliability ratios and total trips estimated hour of the day using data from the Fresno network}
\label{fig:fresno-estimates-by-hour-of-day}
\end{figure}

\subsubsection{Performance functions}

The right-hand side of Figure \ref{fig:fresno-performance-functions} shows the performance functions learned with \MTP. For comparative purposes, the left-hand side shows, for the same links, BPR functions with parameters $\alpha = 0.15, \beta=4$. The polynomial learned by \MTP is $0.1809x + 0.0093x^2 + 0.0095x^3$. As expected, we observe that all functions are monotonically increasing. The difference in curvature of the performance functions is due to the subsequent multiplication by the Kernel matrix (Section \ref{sssec:performance-function-constraint}). Figure \ref{fig:fresno-kernel-link-performance-functions-mate} shows the Kernel matrix for a subset of links of the Fresno network. The average values of the diagonal and non-diagonal parameters of $\mW$ are $3.16$ and $0.25$, respectively, in line with the hypothesized diagonal form of the kernel matrix. 

\begin{figure}[H]
	\centering
\begin{subfigure}[t]{0.3\columnwidth}
    \centering
    \includegraphics[width=1\textwidth, trim= {11.4cm 0.5cm 0.3cm 1.1cm}, clip]{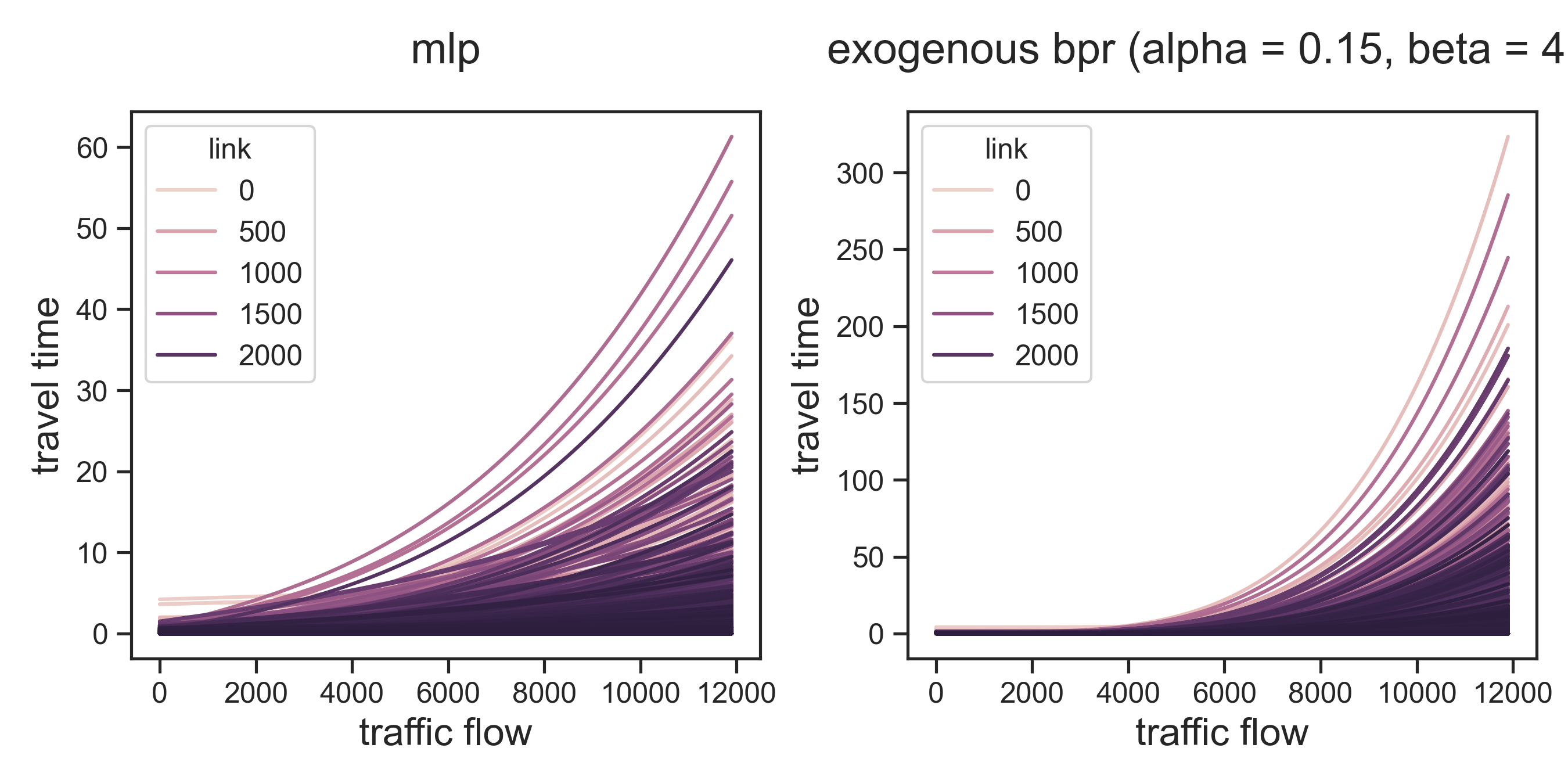}
 \caption{Exogenous BPR ($\alpha = 0.15, \beta=4$)}
	\label{subfig:fresno-performance-functions-mate}
    \end{subfigure}
\begin{subfigure}[t]{0.3\columnwidth}
    \centering
 \includegraphics[width=1\textwidth, trim= {0.3cm 0.5cm 11.4cm 1.1cm}, clip]{figures/results/fresno-comparison-all-link-performance-functions-mate.png}
    \caption{\TVGODLULPE}
	\label{subfig:fresno-some-performance-functions-mate}
\end{subfigure}
\caption{Performance functions learned by the model using data from Fresno, CA}
\label{fig:fresno-performance-functions}
\end{figure}

\begin{figure}[H]
	\centering
	\includegraphics[width=0.7\textwidth]{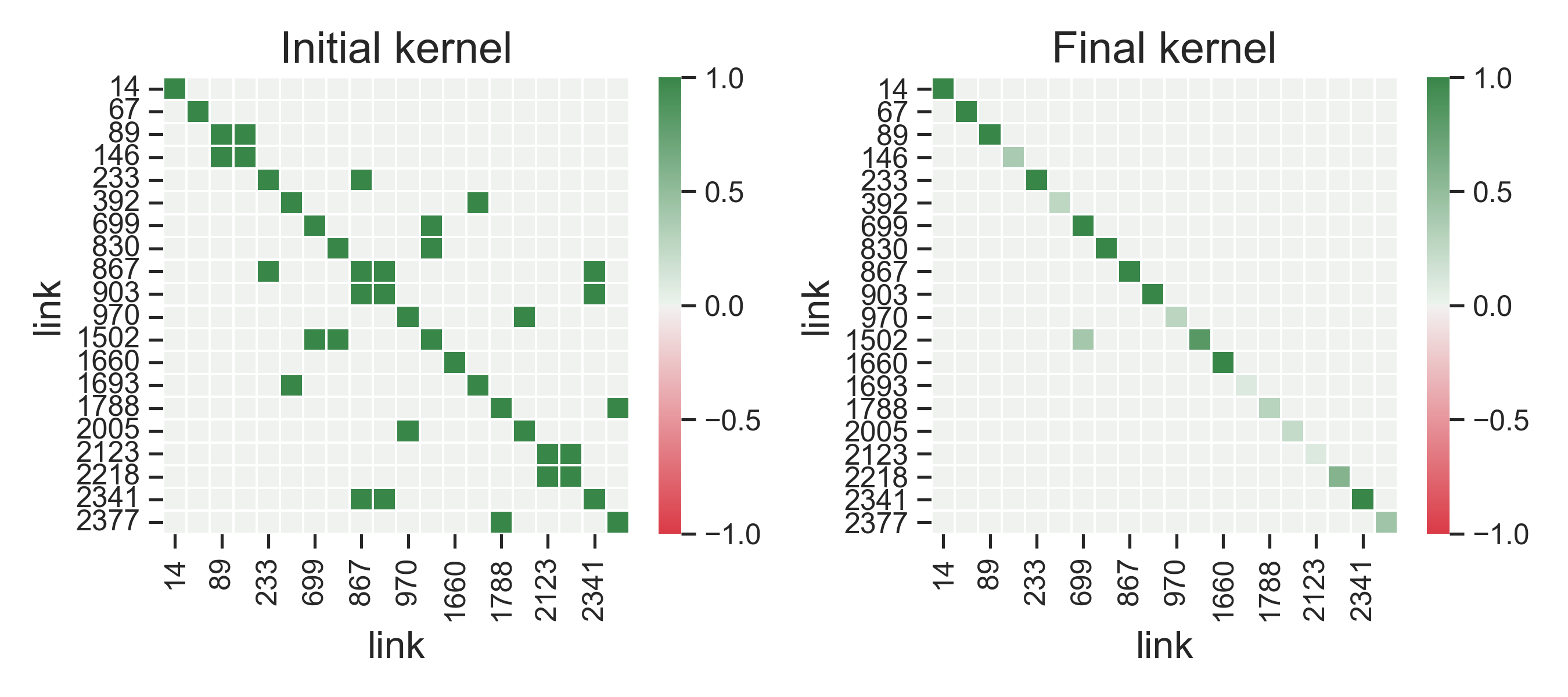}
	\caption{Traffic flow interactions learned by the \TVGODLULPE model using data from Fresno, CA}
	\label{fig:fresno-kernel-link-performance-functions-mate}
\end{figure}

\subsubsection{Traffic forecasting}
\label{ssec:fresno-forescasting}

The model in October 2020 is fine-tuned using \MTP pre-trained with data from October 2019. At training time, only the link flow parameters are defined as learnable and adjusted until the relative gap reported in October 2019 is achieved. Because most links with measurements of travel time and traffic flow in October 2020 also have measurements in October 2019, the traffic forecasting problem is another valid scenario to analyze in-sample performance. The main change in the input data in October 2020 with respect to October 2019 is in the values of the exogenous features of the utility function, e.g., the standard deviation of travel times or the number of yearly incidents. Other model inputs, such as the exogenous features of the generation model and the links' capacities, are assumed the same, but they could be intervened to study alternative what-if scenarios. 

Table \ref{table:fresno-gof-in-sample} reports the model's forecasting performance in October 2020. Figure \ref{fig:fresno-flow-traveltime-insample} shows scatterplots comparing the observed and estimated values of travel time and traffic flow by hour in October 2019 and October 2020. In both years, we observed that the positive correlation between travel times and traffic flow is preserved in the observations and the estimated values. The reduction in the $R^2$ reported in the bottom plots in October 2020 is in line with the increase in MSE and MDAPE with respect to October 2019. The increase in forecasting error between years may be due to changes in the O-D matrix or other factors related to the COVID-19 pandemic that could affect the distribution of link flows in the network. Despite that, the $R^2$ of traffic flow and travel times in 2020 are 0.63 and 0.93, respectively, which suggests a good estimation performance. 

\begin{table}[H]
\renewcommand{\arraystretch}{1.05}
\centering
\begin{threeparttable}
\caption{Goodness of fit of \MTP}
\label{table:fresno-gof-in-sample}
\begin{tabular}{ccccccc}
\hline
\multirow{2}{*}{\begin{tabular}[c]{@{}c@{}}  \\ Dataset\end{tabular}} & \multicolumn{5}{c}{Metric} \\ \cline{2-7} 
 & \begin{tabular}[c]{@{}c@{}}Relative \\ Gap \end{tabular} 
 & \begin{tabular}[c]{@{}c@{}}MSE \\ Equilibrium\end{tabular} 
 & \begin{tabular}[c]{@{}c@{}}MSE \\ Flow\end{tabular} 
 & \begin{tabular}[c]{@{}c@{}}MSE \\ Travel time\end{tabular} 
 & \begin{tabular}[c]{@{}c@{}}MDAPE \\ Flow \end{tabular} 
 & \begin{tabular}[c]{@{}c@{}}MDAPE \\ Travel time \end{tabular} 
 \\ \hline
October 2019
& $0.047$
& $1.5 \times 10^{4}$
& $9.5 \times 10^{5}$
& $2.9 \times 10^{-3}$
& $17.5$
& $4.7$ \\
October 2020
& $0.046$
& $3.8 \times 10^{4}$
& $1.3 \times 10^{6}$
& $4.8 \times 10^{-3}$
& $21.9$
& $6.8$ \\
\hline
\end{tabular}
\begin{tablenotes}
      \footnotesize
        \item MSE: Mean Squared Error. MDAPE: Median Absolute Percentage Error
    \end{tablenotes}
\end{threeparttable}
\end{table}

\begin{figure}[H]
\centering

\begin{subfigure}[t]{0.48\columnwidth}
    \centering
	\includegraphics[width=\textwidth]{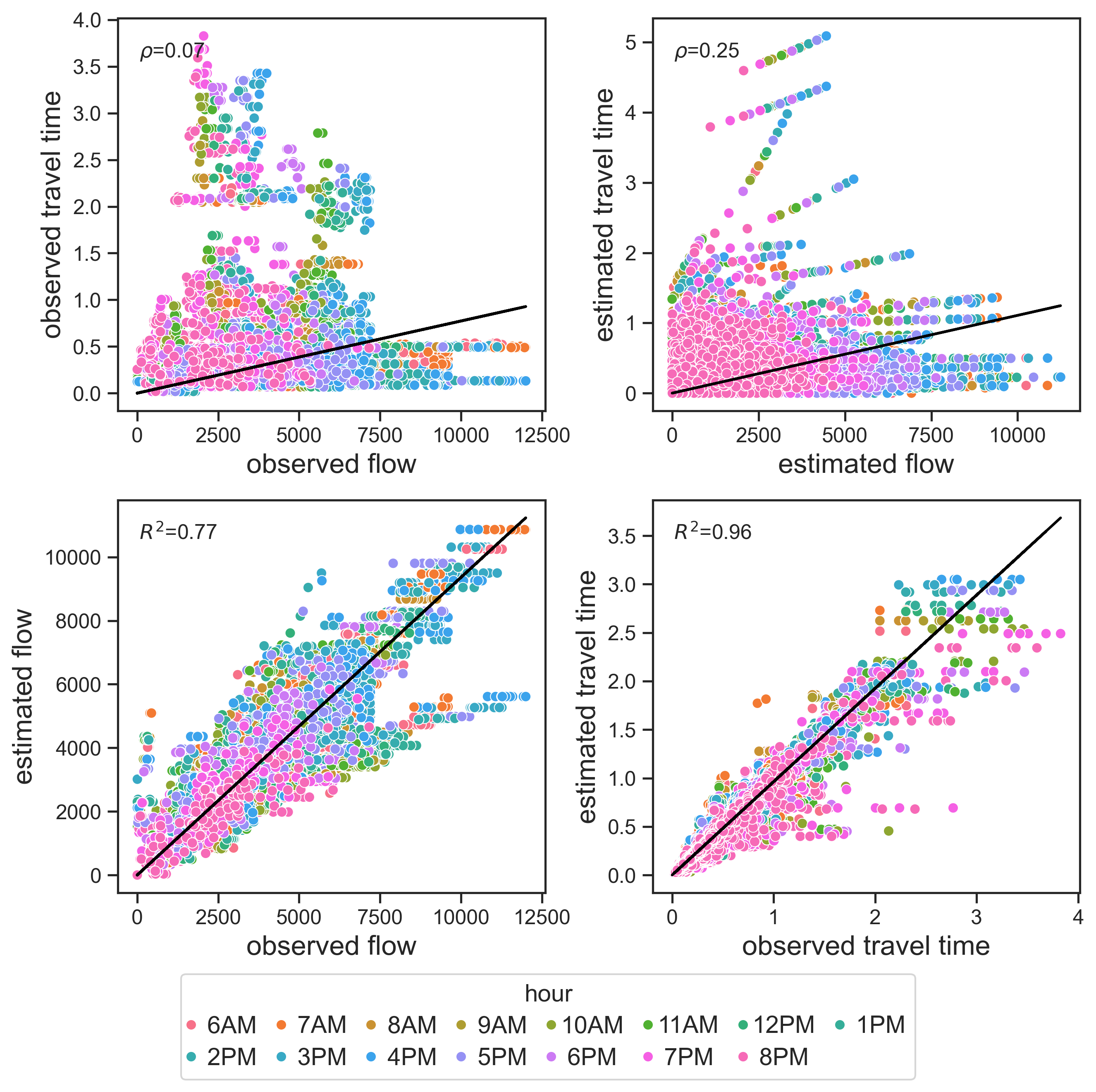}
 	\caption{October 2019}
	\label{subfig:fresno-flow-traveltime-insample-tvodlulpe}
\end{subfigure}
\begin{subfigure}[t]{0.48\columnwidth}
	\centering
	\includegraphics[width=\textwidth]{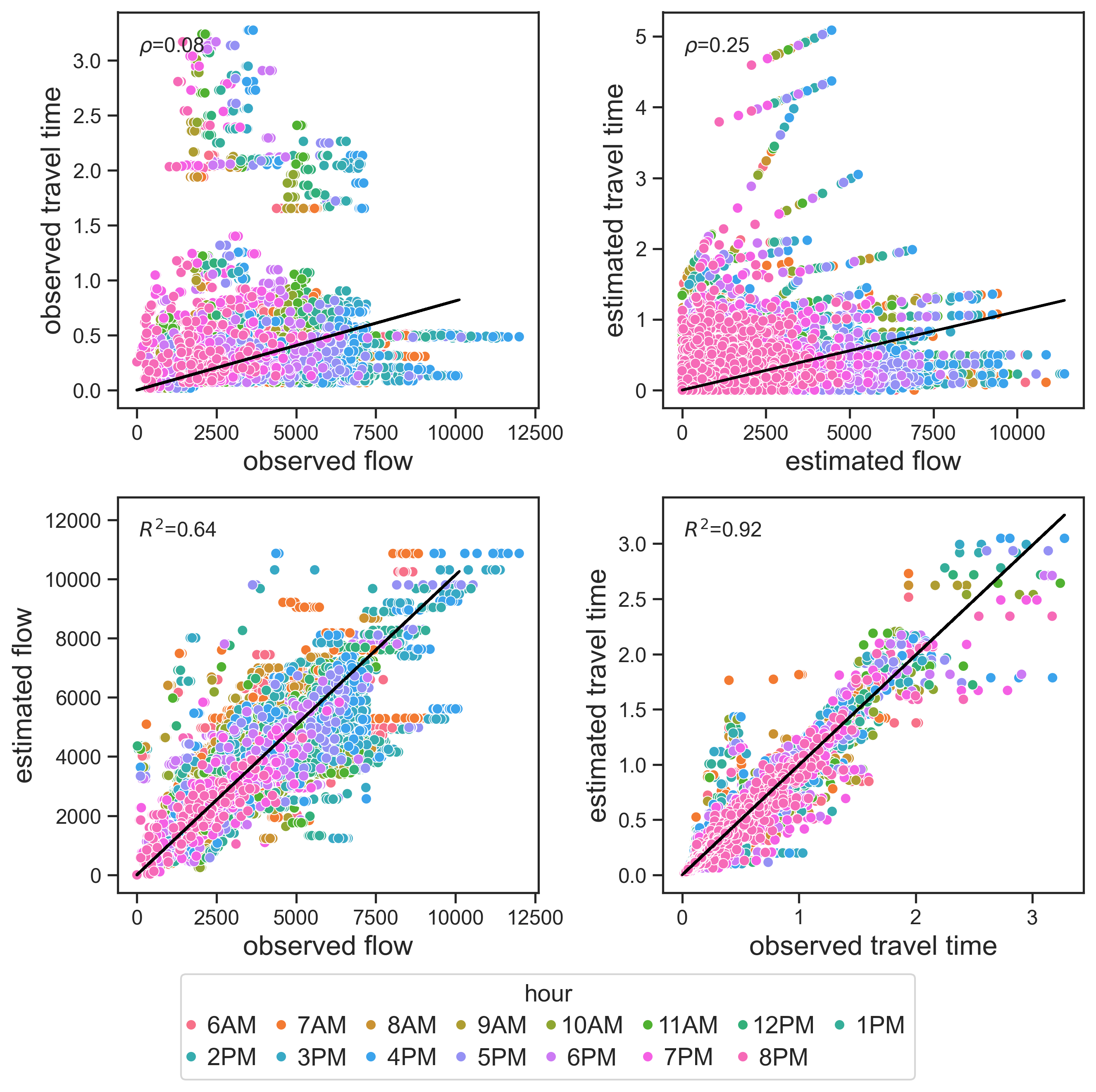}
 	\caption{October 2020}
	\label{subfig:fresno-flow-traveltime-insample-mate-2020}
\end{subfigure}

\caption{Observed versus estimated values of link flow and travel time in Fresno, CA}
\label{fig:fresno-flow-traveltime-insample}
\end{figure}

\subsection{Out-of-sample performance}
\label{ssec:fresno-outofsample-performance}

The out-of-sample performance of \MTP is analyzed using the data collected in October 2020 (Section \ref{ssec:fresno-eda}). This dataset compresses 27,472 and  419,510 measurements of traffic flow and travel times, equating sensor coverages of 5.1\% and 77.3\%, respectively. Out-of-sample performance is measured in links that do not report traffic flow or travel time observations in the training set. The out-of-sample performance of \MTP is compared against three data-driven benchmarks: historical mean, linear regression, and regression kriging. 

\subsubsection{K-fold cross-validation}

To perform the K-fold cross-validation, we follow the same procedure as the experiments conducted in Sioux Falls (Section \ref{siouxfalls-outofsample-performance}). Similar to the experiments, the estimation errors obtained in the training and validation folds are used to study the in-sample and out-of-sample performance of the model, respectively. First, we randomly split the links that report observations in the Fresno network into five folds. Because the coverages of travel time and traffic flow observations among links are significantly different, we create a different set of five folds for each data source, obtaining a total of five pairs of folds. Subsequently, the model is trained using four pairs of folds and validated with the remaining pair of folds. Finally, the process is repeated five times until all pairs of folds are used as validation sets. 

Figure \ref{fig:fresno-kfold-baselines-mdape} compares the out-of-sample performance of \MTP against our three data-driven benchmarks. Similar to the results obtained in the experiments (Figure \ref{siouxfalls-outofsample-performance}), we observe that \MTP outperforms the historical mean model in out-of-sample estimation of traffic flow by a small absolute percentage margin. The margin between \MTP and the kriging and linear regression models is larger, with a relative reduction of MDAPEs of approximately 15-20\%. It is important to consider that the estimations of \MTP are not only more accurate out-of-sample but also comply with fundamental constraints of network flow (Section \ref{ssec:constraints}). Regarding travel times, we observed that \MTP achieves MDAPE lower than 15\% and vastly outperforms the data-driven benchmarks, with relative reductions larger than 50\%. We believe that the larger sensor coverage of the travel time measurements could explain the higher accuracy of \MTP to estimate travel times relative to traffic flows. 

\begin{figure}[H]
\centering
\includegraphics[width=0.5\textwidth]{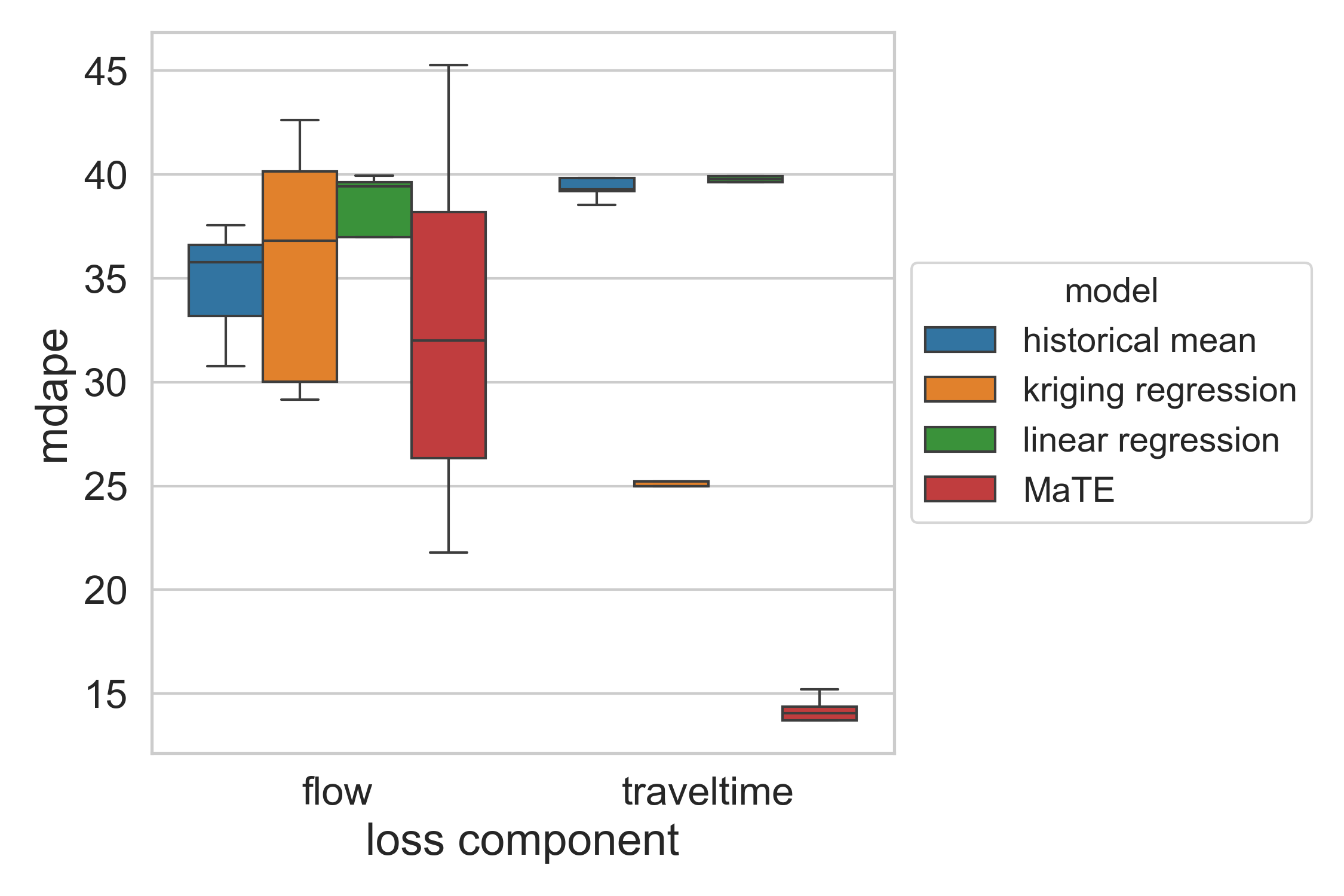}
\caption{MeDian Absolute Percentage Error (MDAPE) obtained with k-fold cross validation and using data from October 2020 in Fresno, CA}
\label{fig:fresno-kfold-baselines-mdape}
\end{figure}

\subsubsection{Network-wide estimation}

Figure \ref{fig:fresno-map-estimation} visualizes the network-wide estimations of \MTP against the regression kriging (\RK) model, the best-performing data-driven benchmark. Estimations are made in all links of the transportation network, regardless of whether the link has traffic flow or travel time measurements in the training set. The estimations of the benchmark models are made for the first Tuesday of October 2020 at 4:00 PM. The models are trained using data collected during the same hour and date to generate more accurate estimations. The estimations with \MTP are made with the model pre-trained with all data from October 2019. The first subplot on the left-hand side of Figure \ref{fig:fresno-map-estimation} shows a basemap obtained from Open Street Maps and the links of the transportation network which are highlighted in gray. To compare the estimations of traffic flow and travel time on a similar scale, we compute the ratio of the estimated and maximum traffic flows (Figure \ref{subfig:fresno-map-flow}) and the ratio of the estimated and maximum speeds (Figure \ref{subfig:fresno-map-speed}) for every link. The maximum flows and speeds are inputs of the model and they are used to compute the link performance functions (Section \ref{eq:neural-performance-function}). The travel times are transformed into speeds by using the length of each link. The maps of ground truth values for the traffic flow and travel speeds do not cover all links in the transportation network because observations are not available for every link. 

Regarding the traffic flow estimation, the \RK estimates that most links are highly congested, which aligns with the observed traffic flow pattern. As a spatial interpolation method, the out-of-sample estimations made by \RK are expected to be highly correlated with the distribution of values of the dependent variable in the training data. In contrast, \MTP generates more diverse traffic flow estimations, especially in out-of-sample links where traffic flow and travel time measurements are not observed in the training set and in links where the estimated level of traffic congestion is lower. This estimation pattern is reasonable, considering that the out-of-sample links are primarily local roads without PeMS traffic counters. Furthermore, the spatial distribution of the estimated speeds in the \RK and \MTP models are similar, except that \MTP seems to generalize better in out-of-sample links. For example, the set of links located farthest north on the map is assigned green and yellow colors in the \MTP and \RK models, respectively. However, the estimation of the \MTP seems more reasonable, considering that the link is part of a highway and the adjacent links are green. In the set of links located in the farthest east of the map and which forms an inverted T-shape, the estimation of \MTP is associated with a green color. In contrast, the \RK model finds that the out-of-sample links are yellow and that the contiguous links are green. The ability of \MTP to generate consistent estimations of travel times and traffic flow in adjacent links comes from the conservation constraints between path and link flows that are enforced by the computational graph during model training (Eq. \ref{eq:path-flows-link-flows-conservation-constraint}, Section \ref{sssec:flow-conservation-constraints}).

\begin{figure}[H]
\centering
\begin{subfigure}[t]{0.46\columnwidth}
	\centering
	 \includegraphics[width=\textwidth, trim= {0cm 0cm 4.8cm 1cm}, clip]{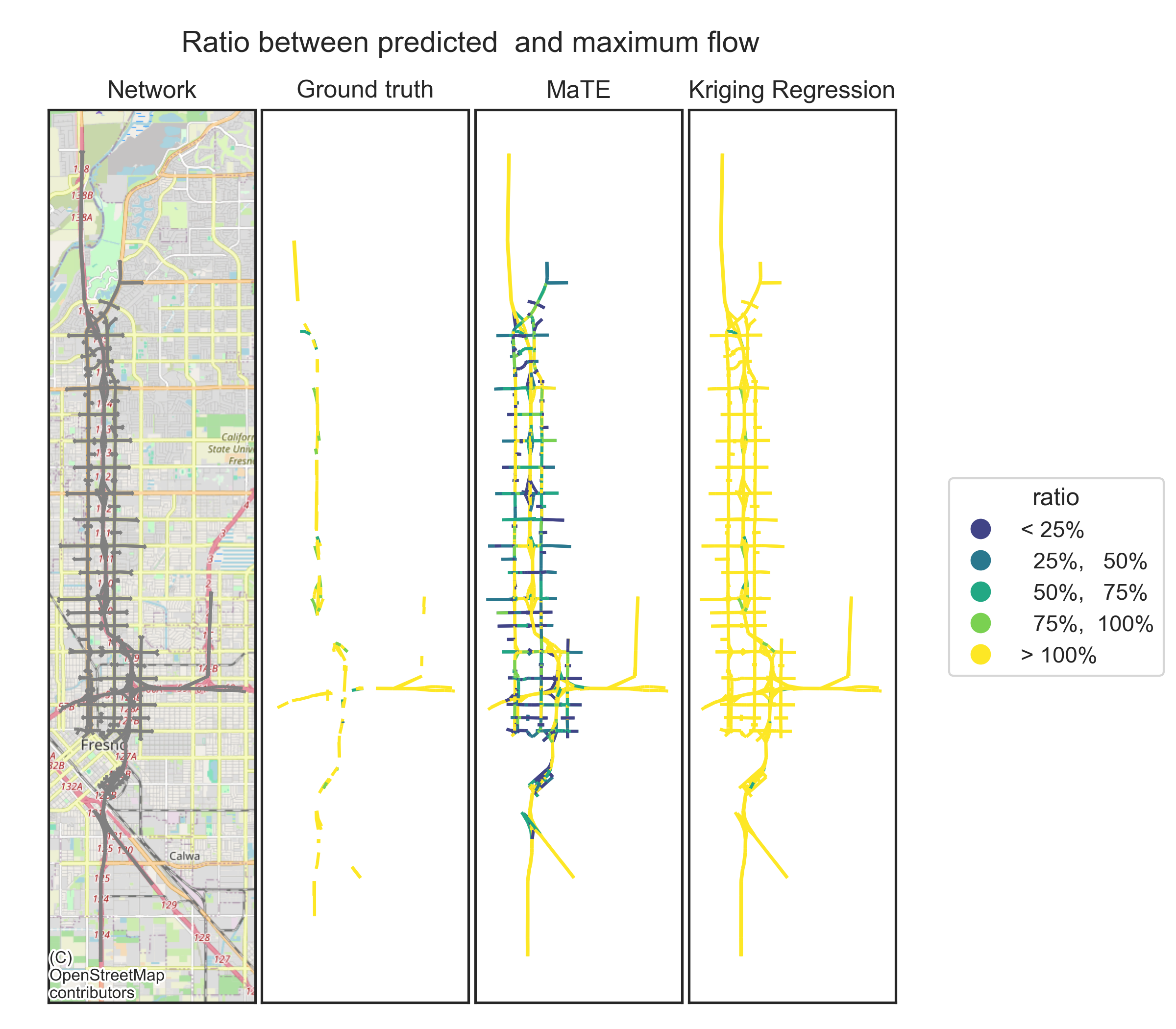}
	\caption{Ratio between estimated and maximum flows}
	\label{subfig:fresno-map-flow}
\end{subfigure}
\hspace{0.5cm}
\begin{subfigure}[t]{0.47\columnwidth}
    \centering
 \includegraphics[width=\textwidth, trim= {4.5cm 0cm 0cm 1cm}, clip]{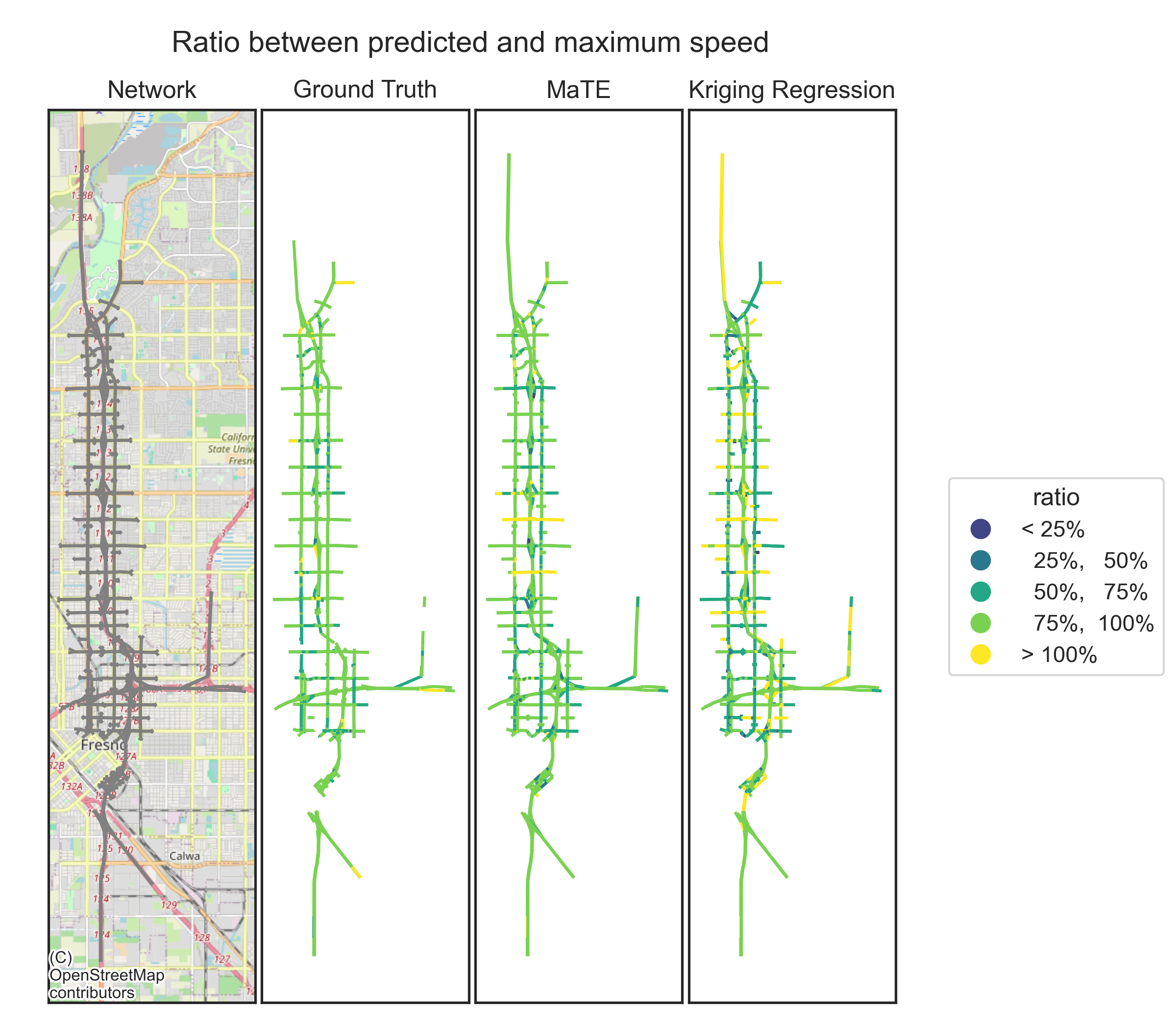}
	\caption{Ratio between estimated and maximum speeds}
	\label{subfig:fresno-map-speed}
\end{subfigure}
\caption{Network-wide estimation of speed and traffic flow during the first Tuesday of October 2020 at 4:00 PM}
\label{fig:fresno-map-estimation}
\end{figure}

\subsubsection{Impact of equilibrium condition in out-of-sample performance}

Figure \ref{fig:hyperparameter-grid} shows in the x-axis and y-axis the logarithm of $\ell_t$ and $\ell_x$, namely, the mean squared errors associated with the estimated travel time and traffic flow, respectively, obtained for different values of the hyperparameter $\lambda_e$ weighting the equilibrium component $\ell_e$ in the loss function. For each value of $\lambda_e$ in $[0,0.001,0.01,0.1, 1, 10]$, we conduct K-fold cross-validation using five folds and compute the average MSE of the estimated traffic flow and travel time for the in-sample (Figure \ref{subfig:hyperparameter-grid-training}) and out-of-sample links (Figure \ref{subfig:hyperparameter-grid-validation}). As expected, higher values of the logarithms of $\lambda_e$ in the (z-axis) induce lower equilibrium MSE and relative gaps in the in-sample folds. As expected, higher values of the logarithm of $\lambda_e$ in the z-axis cause a lower MSE for $\ell_e$ and lower relative gaps. Due to the high representational capacity of \MTP, we observe that the MSE associated with the estimated traffic flow and travel time are similar for all values of $\lambda_e$. Therefore, this evidence suggests that the in-sample performance analysis is insufficient for performing model selection. Figure \ref{subfig:hyperparameter-grid-validation} supports our intuition of the role of the equilibrium loss component as a regularizer in estimating travel time and traffic flow out-of-sample. Therefore, a higher equilibrium hyperparameter tends to improve the goodness of fit to reproduce traffic flow in out-of-sample links. When $\lambda_e = 1$, the trade-off between the traffic flow and travel time losses is Pareto-optimal. In contrast, the out-of-sample performance is suboptimal when the loss function does not account for the equilibrium condition by setting $\lambda_e = 0$. Overall, this evidence shows that regularization of the link flow parameters through incorporating the equilibrium loss component improves out-of-sample performance while having a negligible impact on in-sample performance. 

\begin{figure}[H]
\centering
\begin{subfigure}[t]{0.45\columnwidth}
	\centering
	 \includegraphics[width=\textwidth]{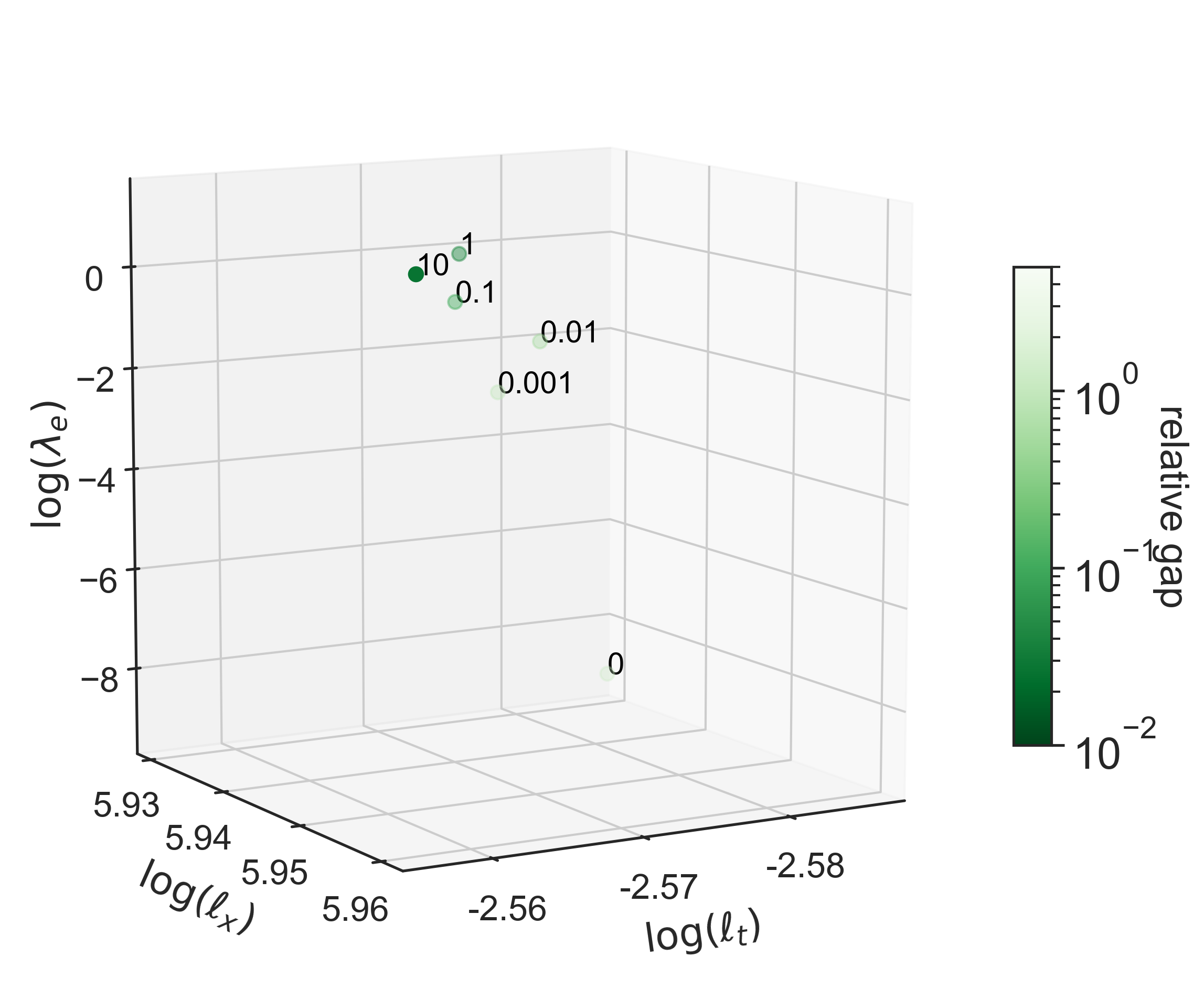}
	\caption{Training}
	\label{subfig:hyperparameter-grid-training}
\end{subfigure}
\begin{subfigure}[t]{0.45\columnwidth}
    \centering
 \includegraphics[width=\textwidth]{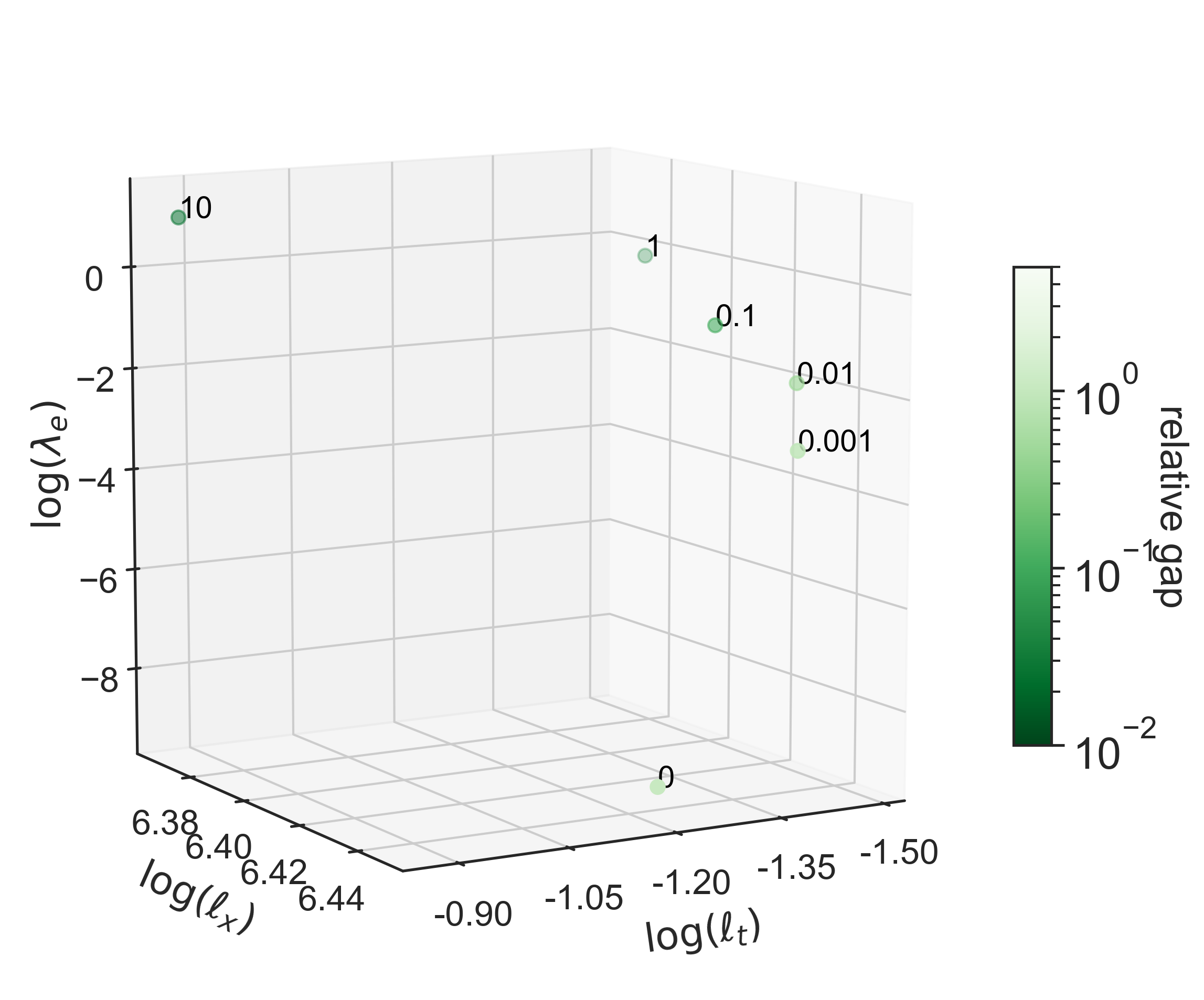}
	\caption{Validation}
	\label{subfig:hyperparameter-grid-validation}
\end{subfigure}
\caption{Average value of the loss components in the training and validation folds for different values of the equilibrium hyperparameter $\lambda_e$}
\label{fig:hyperparameter-grid}
\end{figure}

%% file: sections/conclusions.tex
Traffic flow and travel time estimation are crucial tasks for intelligent transportation systems. Purely data-driven approaches are suitable for making estimations under recurrent traffic conditions and do not require incorporating domain knowledge into the modeling framework. However, these models do not generalize well in traffic contexts that have not been observed in the training data, and they usually require large amounts of data to outperform simple benchmarks based on historical means or lagged traffic features. On the other hand, traditional transportation planning models are grounded on network flow theory, but they are not data-friendly and formulated to perform accurate traffic estimations. In this context, a challenge that remains open is to leverage the prediction capabilities of data-driven methods with the interpretability and theoretical consistency of model-based approaches for real-world transportation applications that demand accuracy and robust network-wide estimation of traffic flow and travel time.

\section{Conclusions}
\label{sec:conclusions}
With the motivation of developing an interpretable model grounded in network flow theory and enhancing generalization performance in settings where historical data is unavailable in a subset of links in the network, this paper introduces the Macroscopic Traffic Estimator (\MTP) model. \MTP provides a network-wide estimation of traffic flow and travel time under recurrent traffic conditions and includes layers that explicitly model trip generation and the travelers' destination and route choices. \MTP also leverages the power of neural networks to learn a more flexible representation of the performance functions that can capture traffic flow interactions and does not require pre-specifying a class of performance functions. Thanks to incorporating location-specific features in the trip generation function, \MTP can also forecast how socio-demographic or population changes could affect travel demand and traffic conditions. Compared with the model proposed by \citet{guarda_estimating_2024}, the solutions of \MTP are proven to be consistent with logit-based stochastic traffic assignment (\STALOGIT), even when the solutions do not strictly satisfy the equilibrium conditions. Thus, our learning algorithm seeks to maximize the accuracy of estimating traffic flow and travel times while leveraging the traffic equilibrium principle as a regularizer of the link flow solutions. 

To compare the performance of the \MTP with well-established data-driven benchmarks for network-wide estimation of traffic flow and travel time, we design a rigorous cross-validation strategy that measures both the in-sample and out-of-sample estimation error associated with travel time and traffic flow. Estimations made on links that do and do not report historical measurements of travel time and traffic flow in the training set are referred to as \textit{in-sample} and \textit{out-of-sample} estimations, respectively. Results obtained in synthetic data show that \MTP can accurately reproduce traffic flow and travel time in both settings. Notably, \MTP learns neural performance functions that mimic the ground truth BPR functions and do not overfit the training data in out-of-sample estimation. An ablation study of the effect of incorporating a trip generation function and a destination choice model in \MTP shows no detrimental impact on out-of-sample performance compared with traditional O-D estimation methods. The application of the model with real-world data collected from a large-scale transportation network shows that \MTP outperforms existing data-driven models in out-of-sample performance. Compared with the regression kriging model, the best-performing data-driven benchmark, \MTP shows relative reductions of 15\% and 50\% in the median absolute percentage error to estimate traffic flow and travel time, respectively. Furthermore, a sensitivity analysis confirms the role of the equilibrium loss component in preventing overfitting and acting as a regularizer of the link flow solution. In fact, higher values of the equilibrium hyperparameter tend to improve out-of-sample performance. Overall, these results make \MTP a trustful tool for transport planning applications aiming at forecasting the impacts of various interventions in the transportation network that can be captured through the model, such as congestion pricing or road capacity improvements. 

%% file: sections/further-research.tex
\section{Limitations and further research}
Further research should evaluate the performance of our model in other real-world transportation networks and against additional model-based benchmarks. The comparison should consider at least computational speed, the feasibility of the solutions, and prediction accuracy in settings of in-sample and out-of-sample predictions. An example of a model-based benchmark is the path-flow estimation model \citep{bell_stochastic_1997}, which has already been used for fitting travel time and traffic count data. The specification of the neural performance functions could also be extended in further studies. In reality, the relationship between flow and speed demonstrates a bi-valued nature, indicating that low flow values can represent either a state of free flow or congestion. Considering the availability of flow and speed data, it becomes feasible to develop a relationship that is more aligned with physical reality and offers a more accurate depiction. Thanks to the flexibility of our computational graph, we expect these extensions to be feasible to implement and validate in real-world transportation networks. Some limitations of our model could also be addressed in future work. Using additional data sources, such as GPS trajectories, could generate a better prior on the path sets and improve the identification of the utility function coefficients. In addition, the current model architecture includes a destination choice layer that should be enriched with destination-specific attributes. This extension could be valuable to analyze how interventions in the network could affect trip distribution over O-D pairs and to identify new drivers of destination choices. Furthermore, our model currently focuses on prediction under recurrent traffic conditions. A natural step is to adapt the model to perform predictions under non-recurrent traffic conditions where the network equilibrium conditions are not necessarily satisfied. To address this challenge, enriching the route choice layer with a model that can depict travel behavior under incidents and disruptions in the transportation network may be necessary. Cross-validation schemes to tune the hyperparameter weighting of the loss components should contribute to properly balancing the traffic equilibrium principle with the network routing behavior observed under disruptions. Finally, \MTP should incorporate additional layers to capture mode choices and accommodate multiple classes of travelers. 

%% file: sections/acknowledgments.tex
\section{Model implementation and data}
\label{sec:model-implementation-data}

The codebase that implements our methodology is available at \url{https://github.com/pabloguarda/mate}. The folder \texttt{notebooks} contains Jupyter notebooks that reproduce all the experiments and results presented in this paper.

\section{Acknowledgments}

This research is supported by a National Science Foundation grant CMMI-1751448. Pablo Guarda acknowledges that 80\% of the work was completed at Carnegie Mellon University. He also thanks Fujitsu Research of America for providing access to multiple Linux servers that helped to speed up model experimentation.  



%% file: sections/references.tex
\bibliographystyle{templates/elsearticle/elsarticle-harv}\biboptions{authoryear}

\renewcommand{\refname}{}

\section*{References} 
\vspace{-1cm}
\bibliography{references.bib}

%% file: sections/appendix.tex
\section{Notation}
\label{appendix:sec:notation}

\begin{table}[h]
	\caption{Network model and inputs} 
	\begin{tabularx}{\textwidth}{ll@{}}
		\toprule
		Notations & Definitions \\
		\midrule
		$\sA, \sV, \sW, \sH$ & The set of all links, nodes, O-D pairs and paths \\
		$\overline{\mD} \in \sR^{|\sA|\times |\sH|}$ & The path-link incidence matrix \\
        $\overline{\mM} \in \sR^{|\sW|\times |\sH|}$ & The path-demand incidence matrix \\
        $\overline{\mL} \in \sR^{|\sV| \times |\sW|}$ & The demand-generation incidence matrix \\
		$\overline{\vx}^{\max} \in \sR^{|\sA|}_{\geq 0}$ & The vector of links' capacities \\
		$\overline{\vt}^{\min} \in \sR^{|\sA|}_{+}$ & The vector of links' free flow travel times \\
        $\overline{\vx}, \overline{\vt} \in\sR^{|\sS| \times |\sA|}$ & The matrices of observed link flow and travel time \\
        $\sK_{\mZ}$ & The set of exogenous features in the route choice utility function \\
        $\sK_{\mO}$ & The set of exogenous features in the trip generation function \\
        $\overline{\mZ}  \in \sR^{|\sS| \times |\sA| \times |\sK_{\mZ}|}$ & The tensor of exogenous link features in the route choice utility function \\
        $\overline{\mO} \in \sR^{|\sS| \times |\sV| \times |\sK_O|}$ & The tensor of exogenous features in the trip generation function\\
		\bottomrule
	\end{tabularx}
	\label{table:notation1}
\end{table}

\begin{table}[H]
	\caption{Trip generation, destination choice and route choice parameters and variables} 
	\begin{tabularx}{\textwidth}{ll@{}}
		\toprule
		Notations & Definitions \\
		\midrule
		$\vt, \tilde{\vt}  \in \sR_{>0}^{|\sA|}$ & The vectors of link travel times and auxiliary travel times \\
        $\vg \in \sR^{|\sS| \times |\sV|}$ & The matrix of generated flow \\
        $\vq \in \sR^{|\sS| \times |\sW|}$ & The matrix of O-D flow \\
  	$\vf  \in \sR_{\geq 0}^{|\sS| \times |\sH|}$ & The matrix of path flows\\
        $\vx  \in \sR_{\geq 0}^{|\sS| \times |\sA|}$ & The matrix of link flows \\
  	$\hat{\vx} \in \sR^{|\sS| \times |\sA|}$ & The matrix of link flow parameters \\
		$\hat{\vtheta} \in \sR^{|\sS| \times |\sK_{\mZ}+1|}$ & The matrix of feature-specific parameters in the utility function\\
  	$\hat{\theta}_t \in \sR^{|\sS|}_{\leq 0}$ & The vector of travel time parameters in the utility function \\
        $\hat{\vgamma} \in \sR^{|\sA|} $ & The vector of link-specific parameters in the route choice utility function \\
        $\hat{\vtheta} \in \sR^{|\sS| \times (1+|\sK_Z|)}$ & The matrix of feature-specific parameters in route choice utility function\\
		$\vv \in \sR^{|\sS| \times |\sH|}$ & The matrix of path utilities \\
  	$\vpf  \in \sR_{]0,1[}^{|\sS| \times |\sH|}$ & The matrix of path choice probabilities \\
        $\vphi \in \sR^{|\sS| \times |\sW|}$ & The matrix of destination choice probabilities \\
        $\hat{\vkappa} \in \sR^{|\sS| \times |\sK_O|}$ & The matrix of feature-specific parameters of the trip generation function \\
        $\hat{\vdelta}  \in \sR^{|\sS| \times |\sV|}$ & The matrix of location-specific parameters of the trip generation function \\
        $\hat{\vomega}  \in \sR^{|\sS| \times |\sW|}$ & The matrix of O-D specific parameters in destination choice utility function\\
		\bottomrule
	\end{tabularx}
	\label{table:notation2}
\end{table}

\pagebreak
\section{Illustrative example}
\label{appendix:sec:illustrative-example}

Consider a network with five nodes and six links (Figure \ref{fig:ilustrative-example}). Assume that only pair 1-4 ($q_1$) and 1-5 ($q_2$) carry flow and that the travelers' consideration set includes the following paths:  1-2-4 ($f_1$), 1-3-4 ($f_2$), 1-2-5 ($f_3$), 1-3-5 ($f_4$). Each link is associated with a performance function that maps the flow in that link into travel time. The travelers' utility function depends on the travel time $t$ and an exogenous feature $z$ whose values are obtained from Census data and vary among links but not over time. 

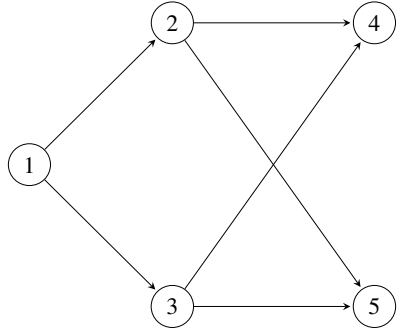
\begin{figure}[h]
	\centering
 \resizebox{0.32\textwidth}{!}{
\begin{tikzpicture}[>=stealth,shorten >=1pt,node distance=3cm,on grid,initial/.style={}]
  \node[circle, draw] (1) {1};
  \node[circle, draw] (2) [above right=of 1] {2};
  \node[circle, draw] (3) [below right=of 1] {3};
  \node[circle, draw] (4) [right=of 2] {4};
  \node[circle, draw] (5) [right=of 3] {5};

  \draw[->] (1) to (2);
  \draw[->] (2) to (4);
  \draw[->] (2) to (5);
  \draw[->] (1) to (3);
  \draw[->] (3) to (4);
  \draw[->] (3) to (5);
\end{tikzpicture}
}
\caption{Network in illustrative example}
\label{fig:ilustrative-example}
\end{figure}
Suppose sets of measurements of travel times $\overline{\vt}$ and link flow $\overline{\vx}$ are available in some links in the network. Assume that the travel times and link flows are obtained under recurrent traffic conditions, i.e., a state where travelers' choices are consistent with logit-based Stochastic User Equilibrium (\SUELOGIT). Suppose that the goal is jointly learning the parameters of the utility function, the performance functions, and the O-D flows using $\overline{\vt}$ and $\overline{\vx}$.

Figure \ref{fig:computational-graph} describes the structure of the computational graph that would solve the learning problem in this small network. Each step in Algorithm \ref{alg:tvgodlulpe}, including the forward and backward passes, can be applied in the computational graph to learn the model parameters and obtain a solution. The elements highlighted in blue summarize the main improvements to the computational graph proposed by \citet{guarda_estimating_2024}. 

\begin{figure}[H]
	\centering
	\includegraphics[width=0.9\textwidth, trim= {3.5cm 2cm 3cm 2.4cm},clip]{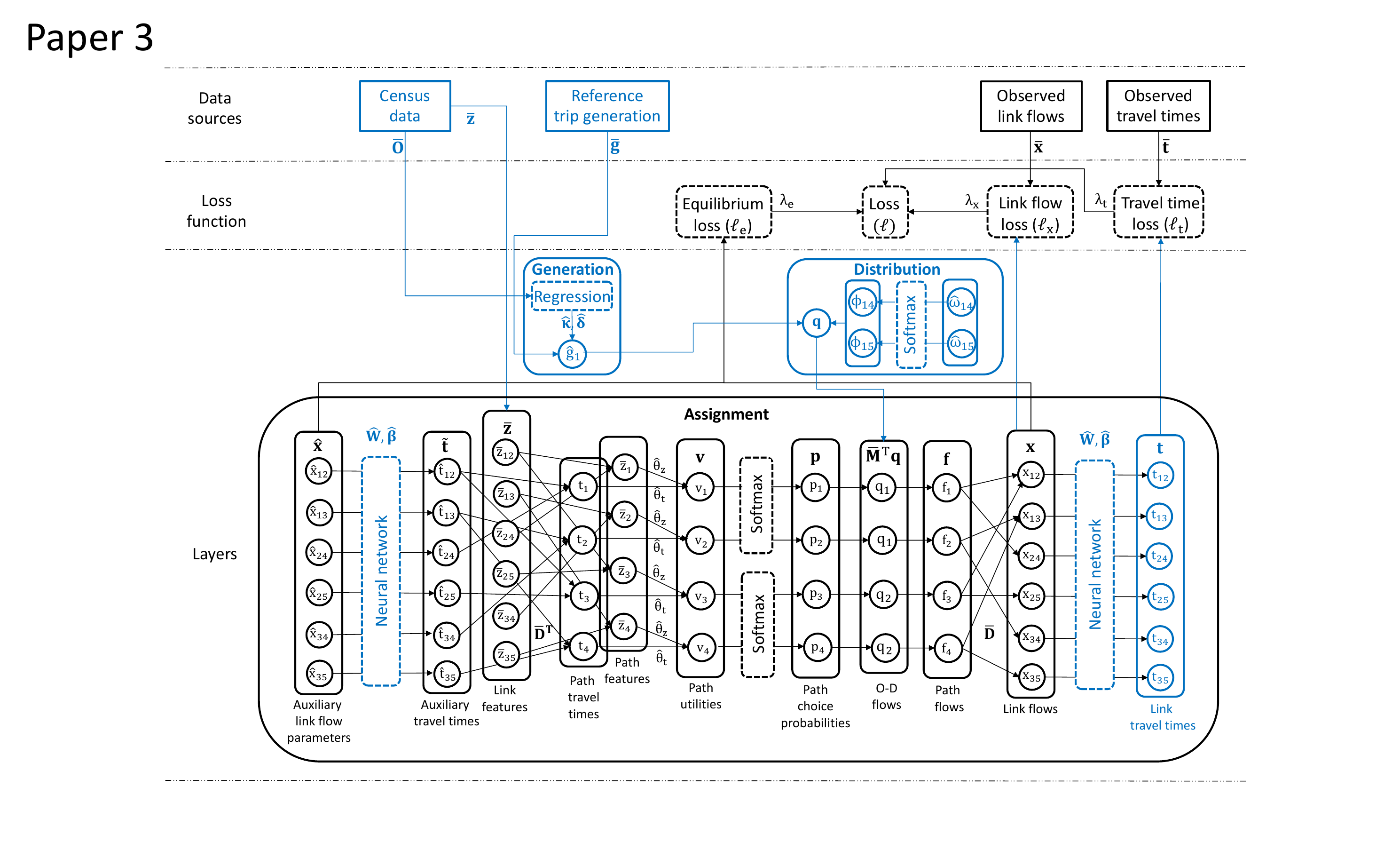}
	\caption{An illustration of the computational graph of \MTP using the toy network}
	\label{fig:computational-graph}
\end{figure}

\pagebreak

\section{Networks}
\label{appendix:sec:networks}

\begin{figure}[H]
	\centering
	
	\begin{subfigure}[t]{0.45\columnwidth}
		\centering
		\includegraphics[width=0.9\columnwidth, trim= {7cm 1.5cm 10cm 0cm},clip]{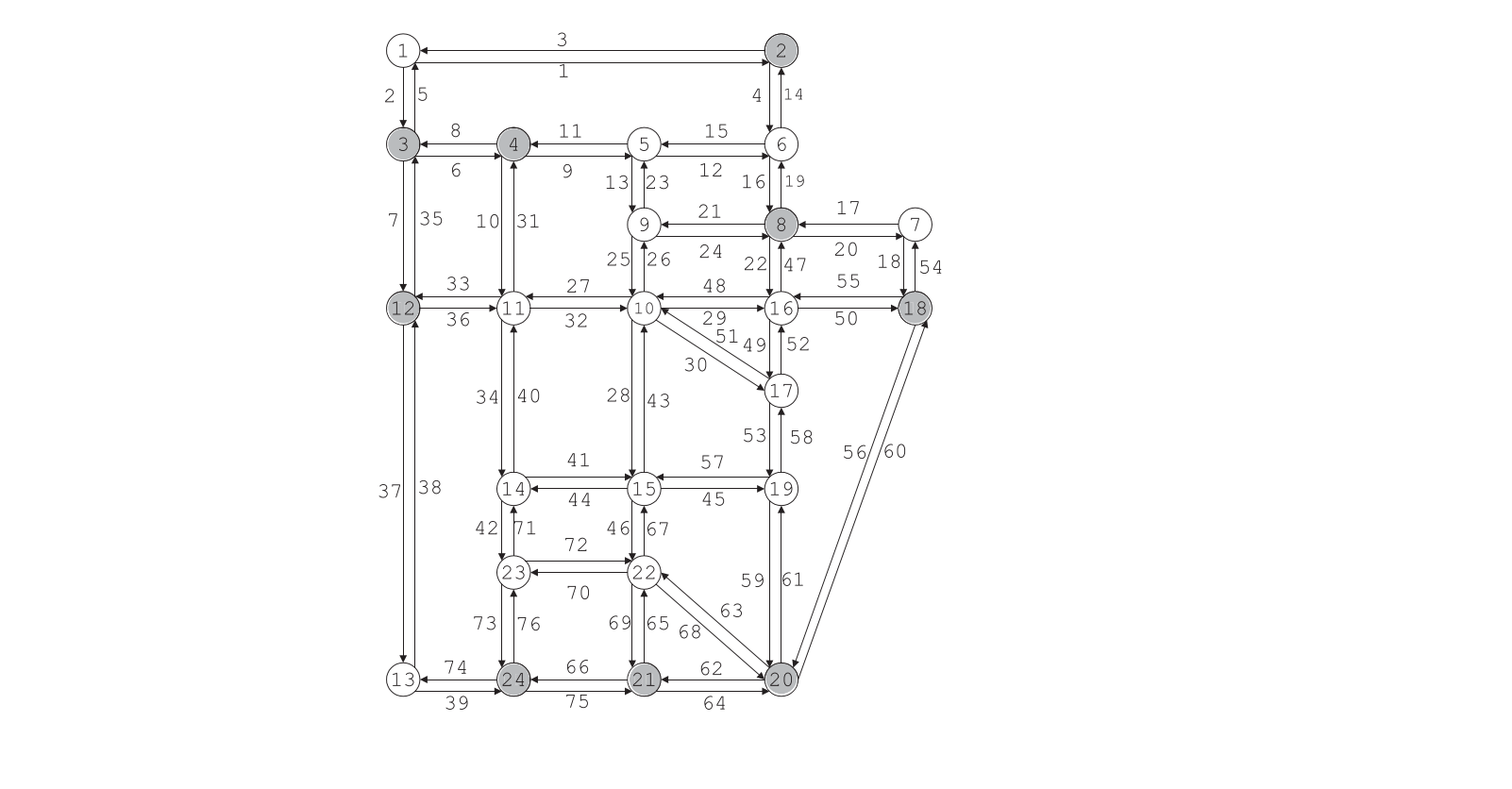}
        \hspace{-2cm}
		\caption{Sioux Falls, SD}
		\label{subfig:sioux-falls-network}
	\end{subfigure}
	\begin{subfigure}[t]{0.45\columnwidth}
        \centering
        \includegraphics[width=0.8\textwidth]{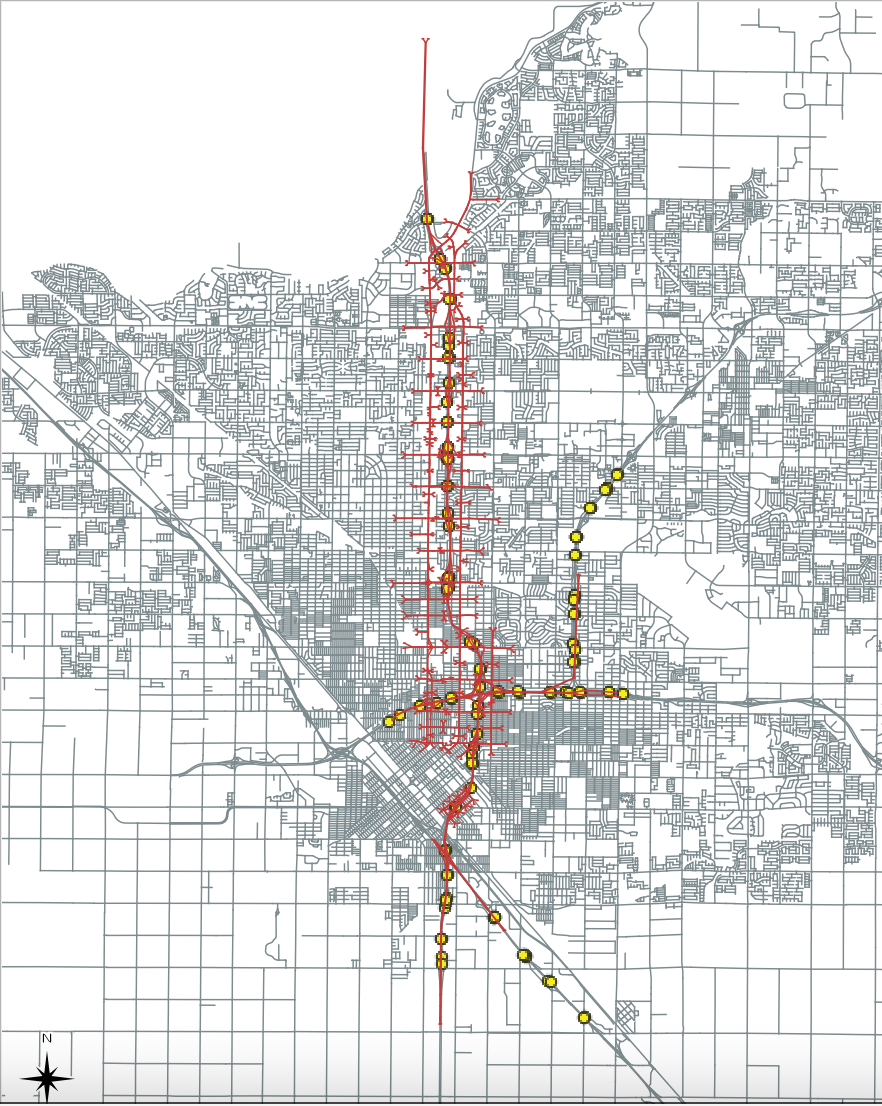}
        \caption{SR-R1 corridor in Fresno, CA}
        \label{subfig:fresno-map}
	\end{subfigure}
	
	\caption{Topologies of transportation networks}
	\label{fig:sioux-falls}
	
\end{figure}


\section{Model training}
\label{appendix:sec:model-training}

\subsection{Sioux Falls network}
\label{appendix:ssec:siouxfalls-model-training}

\subsubsection{Comparison of observed versus estimated values}

\begin{figure}[H]
	\centering
 \begin{subfigure}[t]{0.3\columnwidth}
 \centering
	\includegraphics[width=\textwidth]{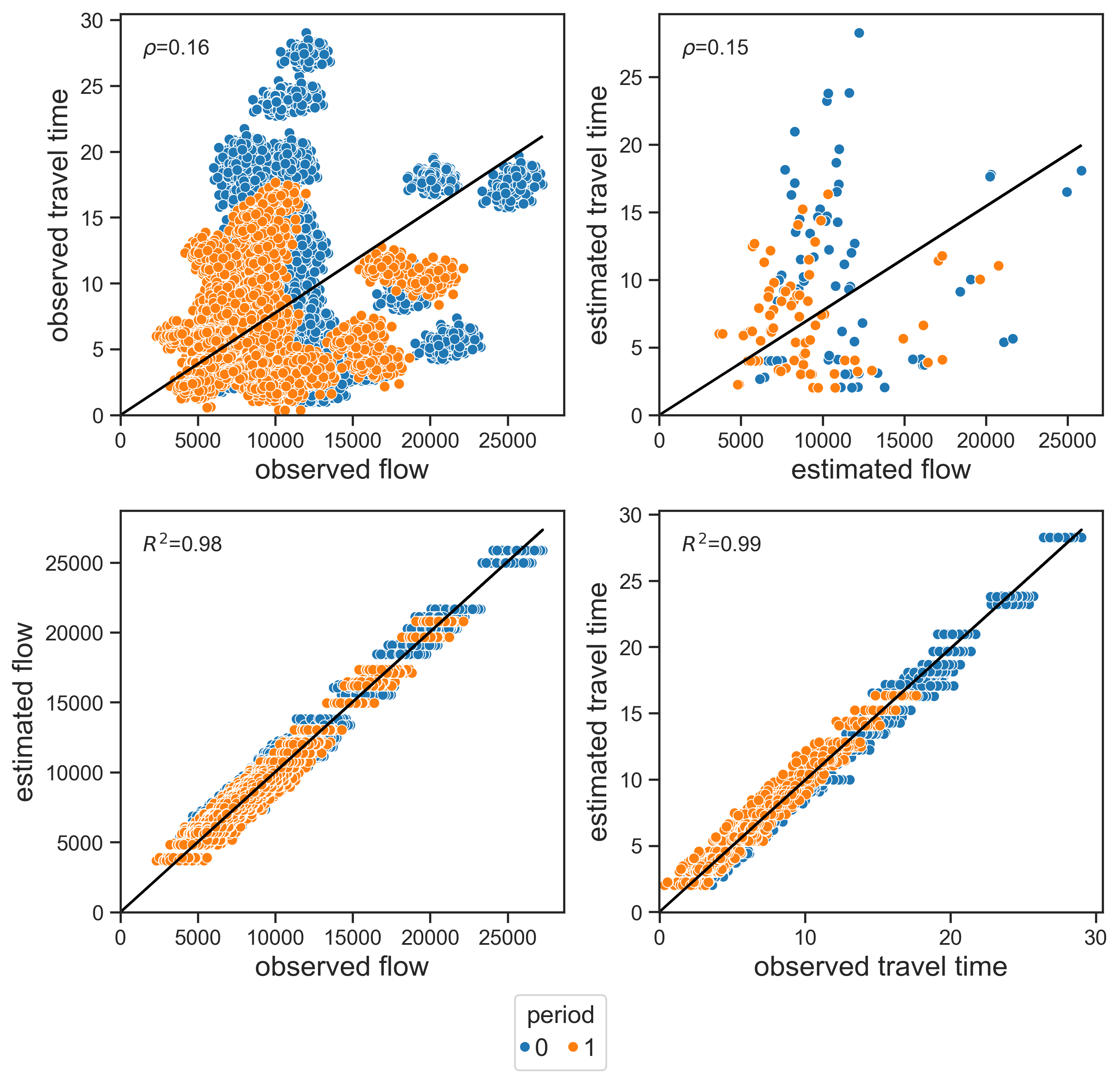}
        \caption{\SUELOGIT}
	\label{fig:siouxfalls-flow-traveltime-suelogit}
 \end{subfigure}
 \begin{subfigure}[t]{0.3\columnwidth}
 \centering
	\includegraphics[width=\textwidth]{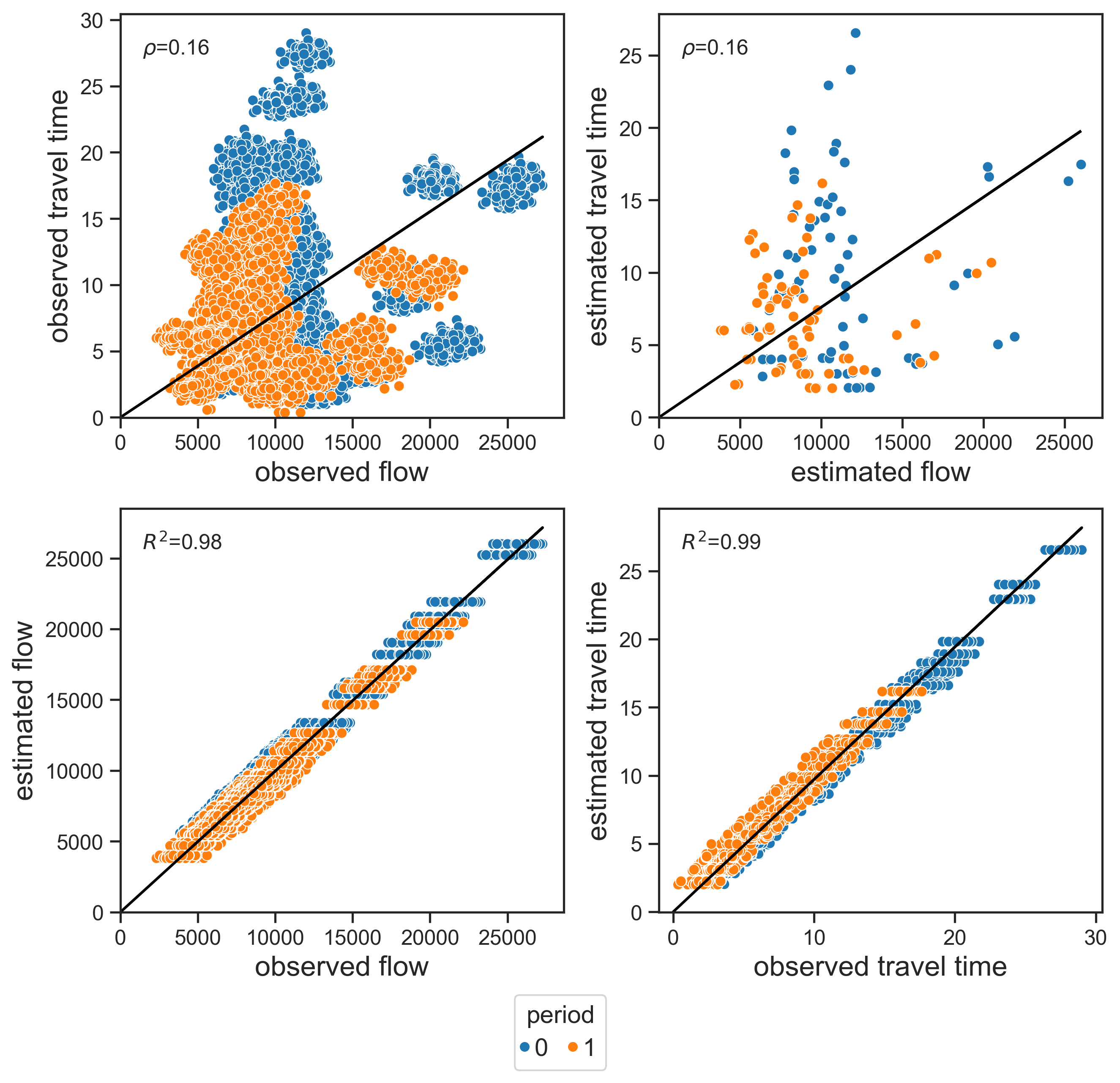}
	\caption{\TVODLULPE}
	\label{fig:siouxfalls-flow-traveltime-tvodlulpe}
  \end{subfigure}
  \begin{subfigure}[t]{0.3\columnwidth}
  \centering
	\includegraphics[width=\textwidth]{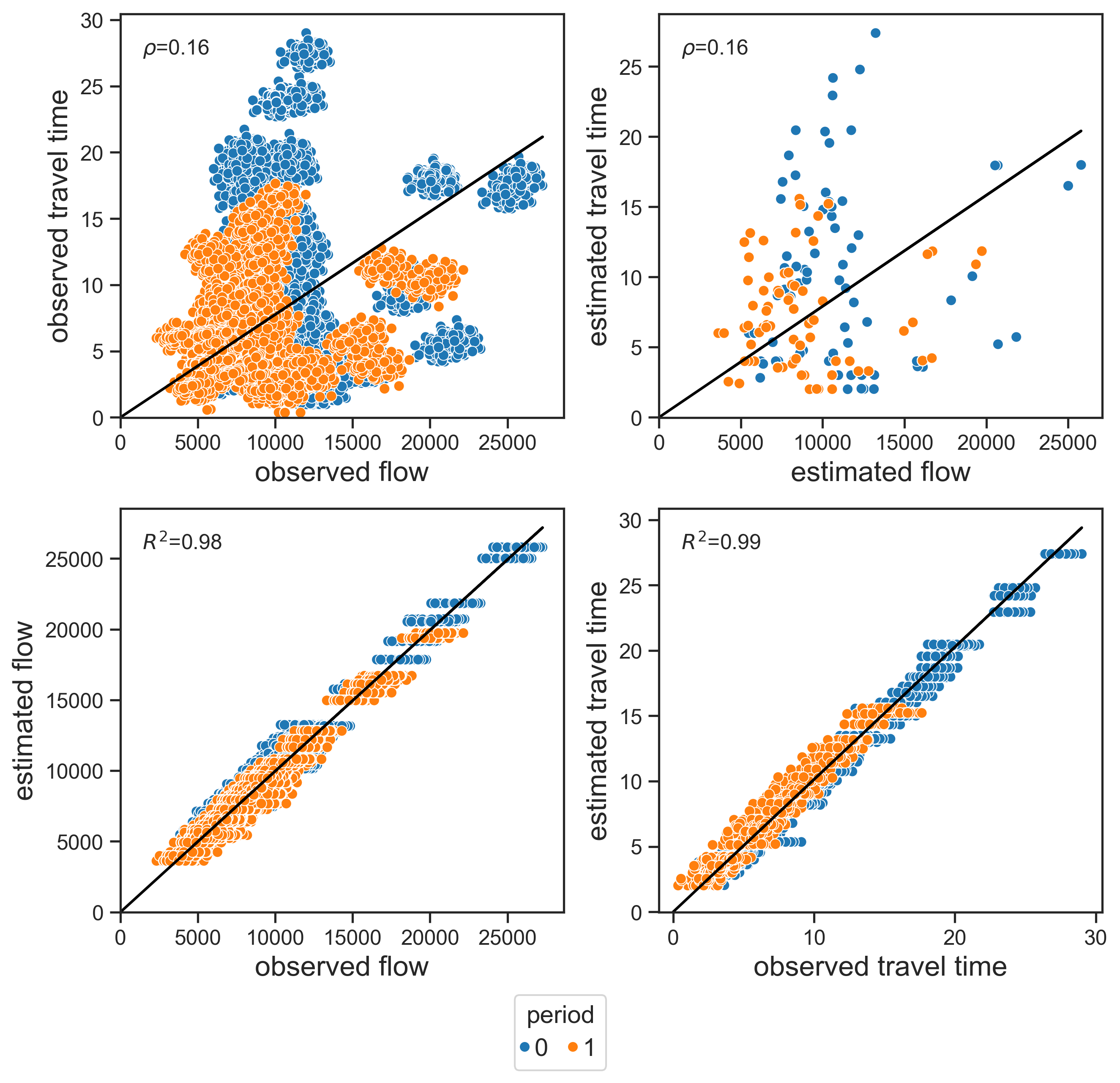}
 	\caption{\MTP}
	\label{fig:siouxfalls-flow-traveltime-mate}
    \end{subfigure}
\caption{Comparison of observed and estimated values of link flow and travel time using synthetic data from Sioux Falls network}
\label{fig:siouxfalls-flow-traveltime}
\end{figure}

\begin{figure}[H]
\centering
\begin{subfigure}[t]{0.3\columnwidth}
    \centering
	\includegraphics[width=\textwidth]{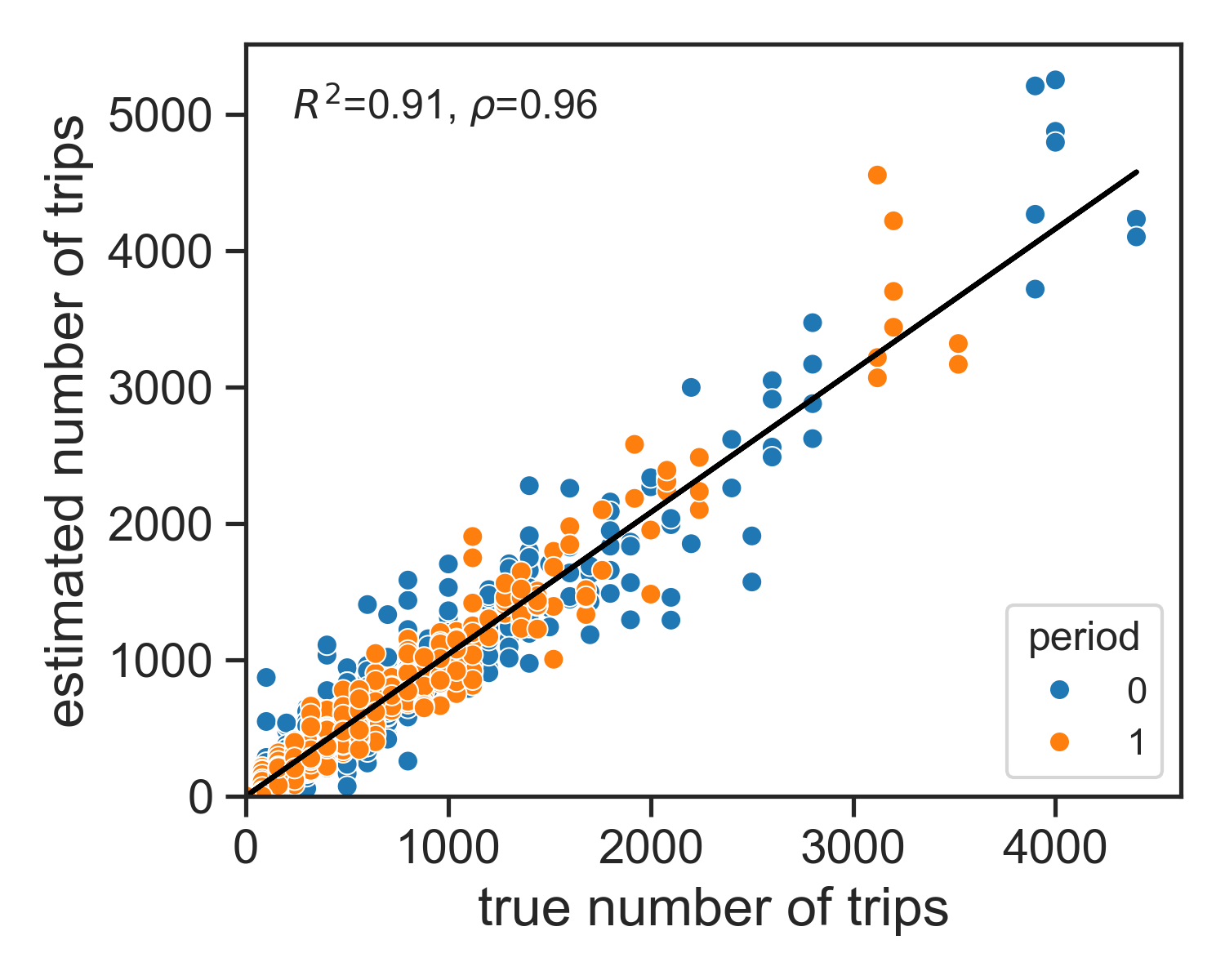}
    \caption{\TVODLULPE}
	\label{subfig:siouxfalls-scatter-ode-tvodlulpe}
\end{subfigure}
\begin{subfigure}[t]{0.3\columnwidth}
	\centering
    \includegraphics[width=\textwidth]{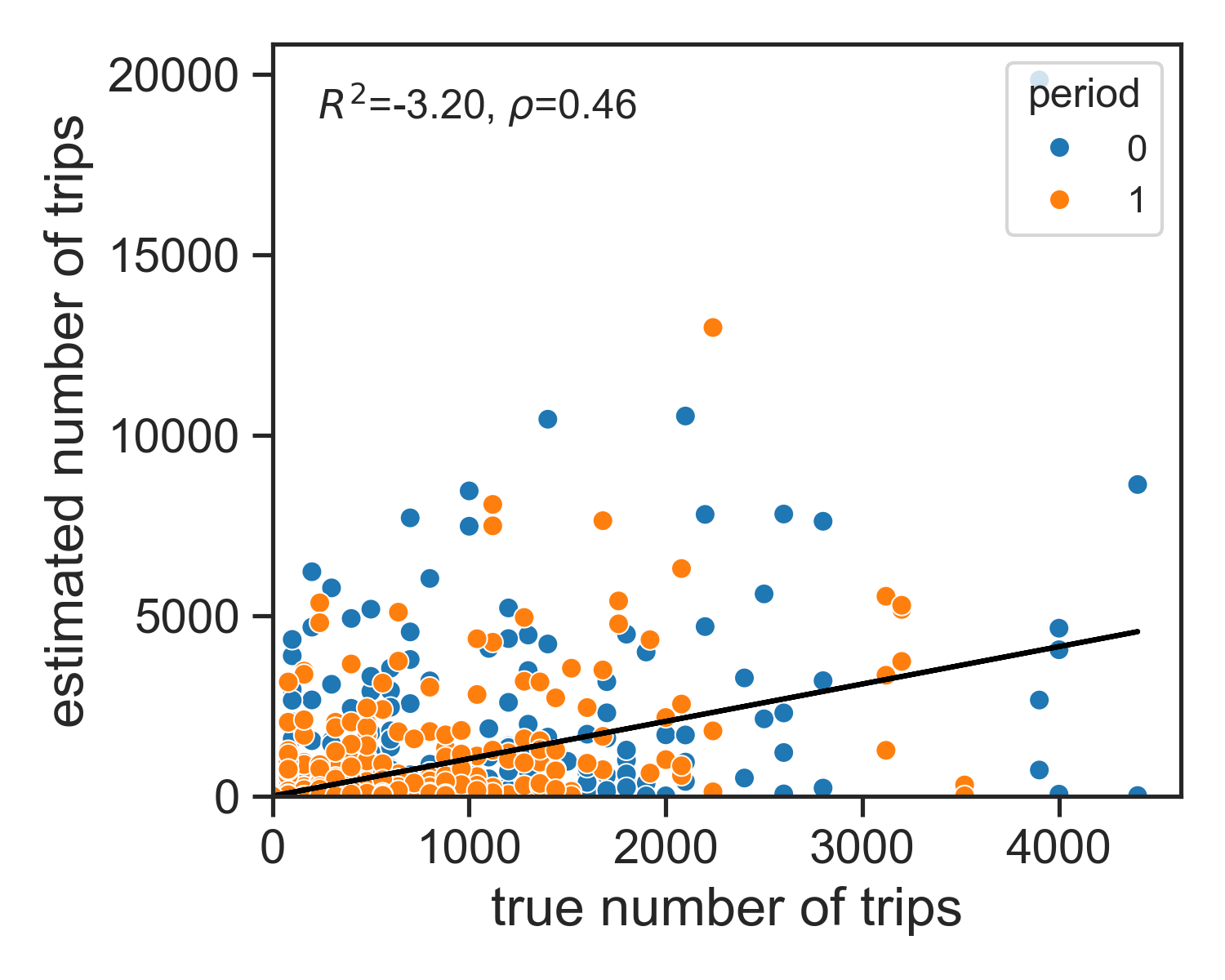}
    \caption{\TVGODLULPE}
	\label{subfig:siouxfalls-scatter-ode-mate}
\end{subfigure}
\caption{Comparison of ground truth and estimated number of trips per O-D pair using synthetic data from the Sioux Falls network}
\label{subfig:siouxfalls-scatter-ode}
\end{figure}

\subsection{Fresno network}
\label{appendix:ssec:fresno-model-training}

\subsubsection{Data description}
\label{ssec:fresno-eda}

Traffic counts from \citet{PeMS} and speed data from \citet{INRIX2021} are used to obtain hourly measurements of traffic flow and travel time across links. The yellow circles shown in Figure \ref{subfig:fresno-map}, Appendix \ref{appendix:sec:networks} represent the location of the traffic counters. The proportions of links with available observations of traffic flow and travel times during the time horizon are 87\% and 5.8\%. Although there are day-to-day variations in the coverage associated with each data source, these differences are small and do not restrict the application of our methodology. 

Figure \ref{fig:fresno-daily-eda} shows the variation of speed [miles/hour], link flow [cars/hour], and the exogenous features of the travelers’ utility between Tuesdays and Thursdays of October 2019 and October 2020.  Features in the utility function include the standard deviation of travel time, the number of bus stops, the number of street intersections, the number of yearly incidents, and the average monthly household income of the census block. The third plot shows the standard deviation of speed instead of travel time for visualization purposes. Exogenous features such as the median household income and the number of intersections and bus stops are constant because their values vary among links but not between days. The standard deviation of speed varies day to day because it is computed using historical travel time information that varies according to the hour of the day and day of the week. In line with our expectations, in 2020, the average speed increased by 9.5\%, and the traffic flow by lane decreased by 9.2\% compared to 2019. We also observe that the increase in the average speed in 2020 is associated with a decrease in the standard deviation of speeds. The significant drop in link flow and the increase in speeds between 2019 and 2020 is mainly attributed to the impact of COVID-19 on travel patterns. 

\begin{figure}[H]
	\centering
	\includegraphics[width=\textwidth]{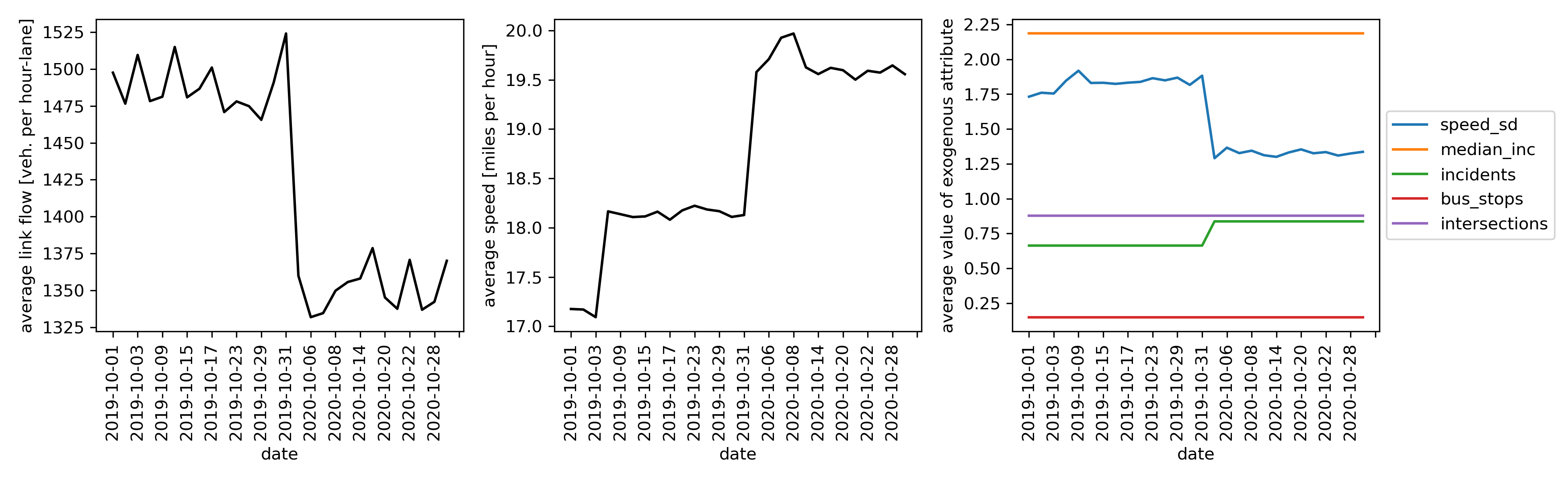}
	\caption{Variation of traffic flow per lane, speed, and the exogenous attributes of the utility function during Tuesdays, Wednesdays, and Thursdays of October 2019 and October 2020 between 6:00 AM and 8:00 PM in Fresno, CA.} 
	\label{fig:fresno-daily-eda}
\end{figure}

Figure \ref{fig:fresno-hourly-eda} shows the hourly change in the average speed and link flow in October 2019 and October 2020. As expected, the traffic flow during the afternoon peak hours is higher than the morning peak hours, especially during 2020. The speed profile is similar to those reported for other US cities \citep{replica_annual_2023}, except that the two largest drops in average speed observed throughout the day start at different hours. The standard deviation of speeds tends to be higher in the hours when there are steady drops in average speed, namely, when links become congested. 

\begin{figure}[H]
	\centering
	\includegraphics[width=\textwidth]{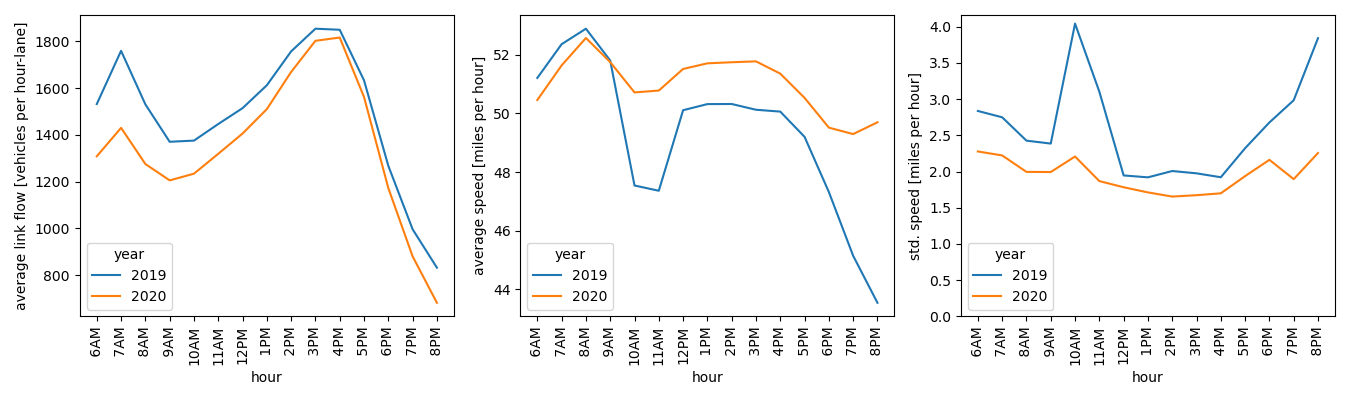}
	\caption{Hourly change of average link flow, average speed and the standard deviation of speed using data from Tuesdays, Wednesdays and Thursdays in October 2019 and October 2020}
	\label{fig:fresno-hourly-eda}
\end{figure}

\subsubsection{Parameter estimates}

\begin{table}[H]
\begin{adjustbox}{width=0.9\textwidth}  
\renewcommand{\arraystretch}{1.05}
\centering
\begin{threeparttable}
\caption{Point estimates obtained with \MTP for the peak hours of the morning (6:00 AM to 9:00 AM) and afternoon (4:00 PM to 6:00 PM) periods using data collected between 6:00 AM and 9:00 PM during Tuesdays, Wednesdays and Thursdays of October 2019 in Fresno, CA}
\label{table:estimates-mate-fresno}
\begin{tabular}{lccccccc}
\hline
\multirow{2}{*}{} & \multicolumn{6}{c}{Hour} \\ \cline{2-8} 
& All$^{a}$ & 6am & 7am & 8am & 4pm & 5pm & 6pm \\ \hline
Utility function \\ 
\hline
Endogenous travel time & &-3.5865 & -3.8456 & -3.4329 & -3.4746 & -3.1385 & -3.9305 \\
Standard deviation of travel time & & 0.0925 & -1.3664 & -3.8896 & -3.1430 & -3.7841 & -3.0242 \\
Neighborhood income & & 0 & 0 & 0 & 0 & 0 & 0 \\
Incidents &  & -3.9881 & -3.4084 & -3.7187 & -3.5054 & -3.7146 & -3.5782 \\
Bus stops &  & -1.3149 & 0 & -0.5146 & 0 & 0 & 0 \\ 
Streets Intersections  & & -4.4391 &-3.5906 &-4.1488 &-3.9346 &-3.8281 &-4.5596 \\
Reliability ratio & & 0.0258 & 0.3553 & 1.1330 & 0.9046 & 1.2057 & 0.7694 \\ 
Average fixed effects & 0.0414 \\
Std of fixed effects & 1.5411 \\
\hline
Generation \\ 
\hline
Population && 110.1214 & 126.9262 &111.0016 & 133.4070 &118.0597 & 92.3437\\
Neighborhood income && 23.6941&  27.3098 & 23.8834 & 28.7043 & 25.4021 & 19.8690\\
Number of bus stops &&  -32.2786 &-37.2044 &-32.5366 &-39.1040 &-34.6055 &-27.0676 \\
\hline
O-D estimation \\
\hline
Average trips per O-D pair && 8.8299 & 9.9782 & 8.7059  & 13.9547 & 13.5291 & 12.3693 \\
Std of trips per O-D pair & & 78.5221 & 81.3923 & 70.1779 & 65.5904 & 70.1467 & 78.0475 \\ 
Total trips & & 61544.5 & 69548.2 & 60680.1 & 97264.5 & 94297.6 & 86214.0\\ 
\hline 
Link performance function$^{b}$ \\
\hline 
Average kernel's diagonal &3.0849 \\
Average kernel's non-diagonal & 0.2396\\
First polynomial parameter  ($x$) & 0.1982\\
Second polynomial parameter ($x^2$) & 0.014\\
Third polynomial parameter ($x^3$)& 0.0124\\
\hline 
Metrics \\
\hline 
Total parameters & 174,817 \\ 
Total samples$^{c}$ & 225 \\
Traffic flow observations (coverage)& 31,624 (5.8\%) \\ 
Travel time observations (coverage) & 468,303 (86.3\%) \\  
\hline 
\end{tabular}
\begin{tablenotes}
      \footnotesize
      \item Note: the hourly estimates between 9:00 AM and 4:00 PM and between 7:00 PM and 9:00 PM are available, but they are not provided in this table. 
      \item $^{a}$For the group of time-varying parameters, this column represents the average value of the estimates over hours. For the group of metrics, this column represents the metric's value including all hourly periods.
      \item $^{b}$The parameters of the link performance function do not vary between periods
      \item $^{c}$Each sample represents a snapshot of the transportation network at a given hour and day, and it compresses a set of link flow and travel time observations in the subset of links where observations are available.
    \end{tablenotes}
\end{threeparttable}
\end{adjustbox}
\end{table}